\documentclass[11pt]{article}

\usepackage[utf8]{inputenc}
\usepackage{makecell}
\usepackage{comment,url,algorithm,algorithmic,graphicx,subcaption,relsize}
\usepackage{amssymb,amsfonts,amsmath,amsthm,amscd,dsfont,mathrsfs,mathtools,nicefrac}
\usepackage{float,psfrag,epsfig,color,xcolor,url}
\usepackage{epstopdf,bbm,mathtools,enumitem}
\usepackage{xparse}
\usepackage{morefloats}
\usepackage{placeins}
\usepackage{xspace}
\usepackage{multirow}
\usepackage{ltablex}
\usepackage{booktabs}
\usepackage{soul}
\usepackage{hyperref}

\def\balign#1\ealign{\begin{align}#1\end{align}}
\def\baligns#1\ealigns{\begin{align*}#1\end{align*}}
\def\balignat#1\ealign{\begin{alignat}#1\end{alignat}}
\def\balignats#1\ealigns{\begin{alignat*}#1\end{alignat*}}
\def\bitemize#1\eitemize{\begin{itemize}#1\end{itemize}}
\def\benumerate#1\eenumerate{\begin{enumerate}#1\end{enumerate}}

\newenvironment{talign*}
 {\let\displaystyle\textstyle\csname align*\endcsname}
 {\endalign}
\newenvironment{talign}
 {\let\displaystyle\textstyle\csname align\endcsname}
 {\endalign}

\def\balignst#1\ealignst{\begin{talign*}#1\end{talign*}}
\def\balignt#1\ealignt{\begin{talign}#1\end{talign}}

\let\originalleft\left
\let\originalright\right
\renewcommand{\left}{\mathopen{}\mathclose\bgroup\originalleft}
\renewcommand{\right}{\aftergroup\egroup\originalright}

\def\Holder{H\"older\xspace}

\def\tinycitep*#1{{\tiny\citep*{#1}}}
\def\tinycitealt*#1{{\tiny\citealt*{#1}}}
\def\tinycite*#1{{\tiny\cite*{#1}}}
\def\smallcitep*#1{{\scriptsize\citep*{#1}}}
\def\smallcitealt*#1{{\scriptsize\citealt*{#1}}}
\def\smallcite*#1{{\scriptsize\cite*{#1}}}

\def\mbb#1{\mathbb{#1}}

\def\reals{\mathbb{R}} %
\def\R{\mathbb{R}}

\def\naturals{\mathbb{N}} %

\def\<{\left\langle} %
\def\>{\right\rangle}

\def\defeq{\triangleq} %

\newcommand{\norm}[1]{\lVert#1\rVert}
\newcommand{\inner}[1]{{\langle #1 \rangle}} %

\def\Esub#1{\E_{#1}}

\def\P{\mbb{P}} %

\DeclareMathOperator{\Tr}{Tr} %

\newcommand{\Gsn}{\mathcal{N}}

\newcommand{\grad}{\nabla}
\newcommand{\Hess}{\nabla^2} %

\def\ent#1{\textnormal{H}\left({#1}\right)}
\def\KL#1#2{\textnormal{H}\big({#1} | {#2}\big)}
\def\I#1#2{\textnormal{I}\big({#1} | {#2}\big)}
\def\M{\textnormal{M}}

\ifdefined\nonewproofenvironments\else
\ifdefined\ispres\else
\newtheorem{theorem}{Theorem}
\newtheorem{lemma}[theorem]{Lemma}
\newtheorem{corollary}[theorem]{Corollary}

\renewenvironment{proof}{\noindent\textbf{Proof.}\hspace*{.3em}}{\qed\\}
\newenvironment{proof-sketch}{\noindent\textbf{Proof Sketch}
  \hspace*{1em}}{\qed\bigskip\\}
\newenvironment{proof-idea}{\noindent\textbf{Proof Idea}
  \hspace*{1em}}{\qed\bigskip\\}
\newenvironment{proof-of-lemma}[1][{}]{\noindent\textbf{Proof of Lemma {#1}.}
  \hspace*{0em}}{\qed\\}
\newenvironment{proof-of-corollary}[1][{}]{\noindent\textbf{Proof of Corollary {#1}.}
  \hspace*{0em}}{\qed\\}
\newenvironment{proof-of-theorem}[1][{}]{\noindent\textbf{Proof of Theorem {#1}.}
  \hspace*{0em}}{\qed\\}
\newenvironment{proof-of-proposition}[1][{}]{\noindent\textbf{Proof of Proposition {#1}.}
  \hspace*{0em}}{\qed\\}
\newenvironment{proof-attempt}{\noindent\textbf{Proof Attempt}
  \hspace*{1em}}{\qed\bigskip\\}

\newenvironment{remark}{\noindent\textbf{Remark.}
  \hspace*{0em}}{\smallskip}%

\fi

\newtheorem{proposition}[theorem]{Proposition}

\newtheorem{assumption}{Assumption}
\fi
\makeatletter
\@addtoreset{equation}{section}
\makeatother

\hypersetup{
  colorlinks,
  linkcolor={red!50!black},
  citecolor={blue!50!black},
  urlcolor={blue!80!black}
}

\newcommand{\eq}[1]{\begin{align}#1\end{align}}
\newcommand{\eqn}[1]{\begin{align*}#1\end{align*}}

\newcommand{\ind}[1]{{\mathbbm{1}}_{\{ #1 \}} }
\newcommand{\abs}[1]{\left| #1 \right|}

\DeclarePairedDelimiter{\ceil}{\lceil}{\rceil}

\def\loja{\L ojasiewicz}
\def\lmc{LMC\xspace}
\DeclareMathOperator{\Span}{span}
\def\EE#1{\mathbb{E}\left[#1\right]}
\def\Esub#1#2{\mathbb{E}_{#2}\left[{#1}\right]}
\def\eps{\epsilon}
\def\!#1{\ensuremath{\mathbf{#1}}}
\def\*#1{\ensuremath{\mathcal{#1}}}
\newcommand{\wasser}[2]{\*{W}_2(#1, #2)}
\newcommand{\wasserp}[3]{\*{W}_2^{#3}(#1, #2)}
\newcommand{\wasserpp}[3]{\*{W}_#3(#1, #2)}

\newcommand{\dotprod}[2]{\left< #1, #2\right>}
\newcommand{\tv}[2]{\mathrm{TV}\left(#1, #2\right)}
\def\rhotild{\widetilde{\rho}}
\def\xtild{\widetilde{x}}

\def\tlam{\tilde{\lambda}}
\def\lam{\lambda}
\def\id{{I}_d}
\newcommand{\tf}{\tilde{f}}
\newcommand{\pert}{\xi}
\newcommand{\decon}{m}
\newcommand{\decof}{\mu}
\newcommand{\decopow}{\theta}

\newcommand{\tilh}{\tilde{h}}

\newcommand{\td}{\tilde{d}}

\newcommand{\tO}{\widetilde{\*{O}}}
\newcommand{\teps}{\tilde{\varepsilon}}
\newcommand{\target}{\nu_*}
\newcommand{\tmu}{\tilde{\mu}}
\newcommand{\growth}{M}
\newcommand{\smooth}{L}

\begin{document}
\mathtoolsset{showonlyrefs}

\title{On the Convergence of Langevin Monte Carlo:\\
The Interplay between Tail Growth and Smoothness
}
 \author{Murat A. Erdogdu\thanks{
  Department of Computer Science and Department of Statistical Sciences at
  the University of Toronto, and Vector Institute
 }
 \and
  Rasa Hosseinzadeh\thanks{
  Department of Computer Science at
  the University of Toronto, and Vector Institute
  }
}

\maketitle

\begin{abstract}
  \item We study sampling from a target distribution ${\target = e^{-f}}$ using
  the unadjusted Langevin Monte Carlo (LMC) algorithm.
  For any potential function $f$ whose
  tails behave like ${\|x\|^\alpha}$ for ${\alpha \in [1,2]}$,
  and has $\beta$-H\"older continuous gradient, we prove that
  ${\tO \raisebox{.5ex}{\big(}d^{\frac{1}{\beta}+\frac{1+\beta}{\beta}(\frac{2}{\alpha} - \ind{\alpha \neq 1})}
  \eps^{-\frac{1}{\beta}}\raisebox{.5ex}{\big)}}$
  steps are sufficient to reach the $\eps$-neighborhood of a $d$-dimensional target distribution $\target$
  in KL-divergence.
  This convergence rate, in terms of $\eps$ dependency, is not directly influenced by
  the tail growth rate $\alpha$ of the potential function as long as its growth is at least linear,
  and it only relies on the order of smoothness $\beta$.
  One notable consequence of this result is that for potentials with Lipschitz gradient, i.e. $\beta=1$,
  the above rate recovers the best known rate ${\tO (d\eps^{-1})}$ which was
  established for strongly convex potentials
  in terms of $\eps$ dependency, but we show that the same rate is achievable for
  a wider class of potentials that are degenerately convex at infinity.
  The growth rate $\alpha$ starts to have an effect on the established rate
  in high dimensions where $d$ is large; furthermore, it
  recovers the best-known dimension dependency
  when the tail growth of the potential is quadratic, i.e. ${\alpha = 2}$,
  in the current setup.
  
  We establish the convergence rate of \lmc by first proving a moment dependent modified
  log-Sobolev inequality with explicit constants for a class of target distributions that have
  convex degenerate potentials at infinity. Then, we prove linearly diverging estimates for
  any order moments of the Markov chain defined by the \lmc algorithm,
  and show that this is sufficient to obtain the above convergence rate.
  Our framework also allows for
  finite perturbations, and any order of smoothness ${\beta\in(0,1]}$;
  consequently, our results are applicable to a wide class of non-convex potentials that are
  weakly smooth and
  exhibit at least linear tail growth.

\end{abstract}

\section{Introduction}
Sampling from a target distribution using Markov chain Monte Carlo (MCMC)
is a fundamental problem in statistics, and
it often amounts to discretizing a continuous-time diffusion process
with invariant measure as the target. 
When the target distribution corresponds to the Gibbs measure ${\target = e^{-f}}$
where ${f:\reals^d\to\reals}$ is the potential function
satisfying ${\int e^{-f(x)}dx = 1}$, a popular candidate diffusion is the overdamped Langevin diffusion,
which is the solution of the following stochastic differential equation (SDE),
\begin{equation} \label{eq:overdamped}
    dZ_t = -\grad f(Z_t) dt + \sqrt{2} dB_t,%
\end{equation}
where $B_t$ is a $d$-dimensional Brownian motion.
It is straightforward to show that the Langevin diffusion \eqref{eq:overdamped} admits the target Gibbs measure $\target$ as its invariant distribution~\cite{ma2015complete}.
In general, simulating a continuous-time diffusion such as \eqref{eq:overdamped} is impractical;
thus, a numerical integration scheme is needed to approximate it.
Due to its simplicity, efficiency, and well-understood theoretical properties,
algorithms based on Langevin diffusion have found numerous applications in sampling and optimization literature \cite{welling2011bayesian,dalalyan2017furthur,raginsky2017non, xu2018global,li2019stochastic}.
In this work, we focus on the unadjusted Langevin Monte Carlo algorithm (\lmc) which is the Euler discretization of the overdamped Langevin diffusion \eqref{eq:overdamped} and relies on the following update rule
\begin{equation}\label{eq:ULA}
    x_{k+1} = x_k - \eta \grad f(x_k) + \sqrt{2 \eta} W_k,%
\end{equation}
where $W_k$ is an isotropic Gaussian vector independent from $W_m$ and $x_m$ for $m < k$,
and $\eta$ is the step size.
\lmc defines a Markov chain which has an invariant measure
that is different than the target $\target$,
and this difference is often termed as the bias
which is due to the numerical integration.
The bias of a discretized diffusion such as \lmc
can be generally controlled with a smaller step size,
where in the limit case ${\eta \downarrow 0}$, the iteration \eqref{eq:ULA} scaled with ${1/\eta}$ reduces to the SDE \eqref{eq:overdamped}.

Convergence rate of \lmc has been the focus of recent research.
Rates are established under structural assumptions on the
potential function, and they quantify the number of iterations required to
reach the $\eps$-neighborhood of a $d$-dimensional target distribution $\target = e^{-f}$
under a particular distance measure -- our focus is on KL-divergence.
Earlier attempts established convergence rates under the global curvature assumptions on the potential function.
For example, for strongly convex and smooth potentials,
the convergence rate of ${\tO(d\eps^{-1})}$ has been shown~\cite{dalalyan2017theoretical},
whereby smooth function is a function with Lipschitz continuous gradient.
We note that higher-order smoothness on the potential function
may improve the convergence rate \cite{mou2019improved};
however, we consider only the first-order smoothness in the current paper.
For convex and smooth potentials with growth rate $\alpha$,
a convergence rate of ${\tO(d^{1+4/\alpha}\eps^{-3})}$ 
is known to hold for LMC~\cite{cheng2018convergence}.
More recently, however, it has been observed that tail growth structure is the determinant factor in sampling~\cite{cheng2018sharp,eberle2016reflection,erdogdu2018global,eberle2019couplings},
rather than the global curvature structure such as (strong) convexity, where
in this context, a strongly convex potential is understood to exhibit quadratic growth.
Growth-based structural conditions has the additional benefit of allowing for finite perturbations,
which in turn allows for sampling from non-convex potentials with a wide range of modern applications
in statistics.
A condition on the target distribution $\target$ that fits in this framework is
the \emph{log-Sobolev inequality} (LSI) of Bakry and Emery~\cite{bakry1985LSI}, which can be written as
\eq{\label{eq:lsi}
 \forall \rho,\ \  \KL{\rho}{\target} \leq \lambda \I{\rho}{\target}, %
}
where $\KL{\rho}{\target}$ denotes the KL-divergence (relative entropy) and
$\I{\rho}{\target}$ denotes the relative Fisher information between $\rho$ and $\target$,
and ${\lambda >0}$ is the log-Sobolev constant.
The LSI condition \eqref{eq:lsi} can be verified for potentials with certain growth structure.
Indeed, it is known to hold for strongly convex potentials~\cite{bakry1985LSI},
and it allows for finite perturbations due to Holley-Stroock perturbation lemma~\cite{holley1987LSI};
thus, potentials that have quadratic growth can be shown to satisfy LSI (this will be made precise later).
Denoting the distribution of Langevin diffusion \eqref{eq:overdamped} at time $t$ with $\rho_t$,
it is known that ${\tfrac{d }{d t} \KL{\rho_t}{\target} = -\I{\rho_t} {\target}}$
which, combined with the LSI condition \eqref{eq:lsi} entails a differential inequality of the form
${\frac{d }{d t} \KL{\rho_t}{\target}  \leq - \frac{1}{\lambda} \KL{\rho_t}{\target}}$,
which in turn yields an exponential contraction in KL-divergence, i.e.,
${\KL{\rho_t}{\target} \leq e^{-t/\lambda}\KL{\rho_0}{\target}}$ for the diffusion process.

An important implication of a condition like LSI~\eqref{eq:lsi} in the context of sampling with \lmc is
that fast mixing properties of the continuous-time diffusion induced by LSI
are inherited by the discrete algorithm \lmc.
That is, LSI coupled with the smoothness condition on the potential is sufficient
to obtain the fast convergence rate
${\tO(d\eps^{-1})}$~\cite{vempala2019rapid}, which is the best known rate for \lmc
in this framework.
The significance of this result is in that, it relaxes the strong convexity assumption
which is a global curvature condition on $f$
to the LSI condition \eqref{eq:lsi}, which can be regarded as a tail growth condition on $f$,
allowing for perturbations and consequently sampling from non-convex potentials.

The fundamental idea leading to the current paper is that the fast convergence of \lmc does not require
an exponentially contracting Langevin diffusion, which is essentially obtained under strong tail growth conditions on the potential.
A representative convergence analysis of the \lmc algorithm under some distance measure $\textnormal{D}$ (our main focus is KL-divergence)
starts with establishing a single step bound in the following sense,
\eq{\label{eq:single-step}
  \forall k \in\naturals,\ \ \textnormal{D}(\rho_{k+1} | \target) \leq r(\eta)\, \textnormal{D}(\rho_k | \target) + C \eta^{\theta},
}
where ${r:[0,\infty) \to [0,1]}$ is a monotone decreasing function
which is typically inherited from the fast decaying diffusion counterpart.
The discretization error ${C\eta^\theta}$ can be made small with smaller step size $\eta$,
and the exponent $\theta$ is intrinsic to the numerical scheme as well as the order of smoothness
of the potential. Elementary algebra reveals that,
one can iterate the inequality \eqref{eq:single-step} and achieve convergence as long as ${r(\eta)<1}$.
Recent literature focused on exponential decays ${r_{\exp}(t) = e^{-\alpha t}}$
which
are usually established under conditions like LSI \eqref{eq:lsi}
or strong convexity 
that correspond to potentials exhibiting quadratic growth
(see, for example \cite{dalalyan2017theoretical, vempala2019rapid}).
Nevertheless,
the inequality \eqref{eq:single-step} by no means benefits from the exponential decay, as $r(t)$ is only evaluated at short time horizons ${t=\eta}$.
Indeed, any decreasing rate function $r(t)$ satisfying ${r(0)=1}$ and ${r'(0) < 0}$ would achieve
the same rate of convergence as exponential decay ${r_{\text{exp}}(t)}$.
For example, consider the algebraic rate
${r_\text{alg}(t) = 1/(1+\alpha t)}$ which is much slower than the exponential rate,
but it provides the same level of decay in small time horizons, i.e. evaluated at the step size $\eta$,
one has
\eq{\label{eq:exp-eq-alg}
  r_\text{alg}(\eta) \approx r_{\exp}(\eta) \approx 1-\alpha \eta \ \text{ when $\eta$ is small.}
}
However, in contrast to exponential decay,
algebraic rates can be obtained under much weaker tail growth conditions on
the potential function $f$.

Modified versions of the LSI condition \eqref{eq:lsi} or weak Poincar\'e inequalities
are commonly employed in the analysis of diffusion processes~\cite{bakry2013analysis},
and can be used to explain different convergence behavior.
For example in the seminal work by \cite{toscani2000trend},
a modified log-Sobolev inequality is used to establish
a convergence rate of ${\mathcal{O}(t^{-\kappa})}$ for all ${\kappa>0}$ for the Langevin diffusion
\eqref{eq:overdamped} ($\mathcal{O}(t^{-\infty})$ in their notation). Our results build on a similar construction.
For a class of potentials that are convex degenerate at infinity,
with tails growing like ${\|x\|^\alpha}$ for ${\alpha\in[1,2]}$,
we establish the following modified log-Sobolev inequality (mLSI)
\eq{\label{eq:weak-lsi}
  \forall \rho,\ \   & \KL{\rho}{\target} \leq \lambda \I{\rho}{\target}^{1-\delta} \M_s(\rho+\target)^{\delta}
  \ \ \text{ with }\ \  \delta \in [0,1/2),
}
where ${\M_s(\rho) = \int  (1+\| x \|^2 )^{s/2} \rho(x) dx}$ is the $s$-th moment of any function $\rho$.
This inequality entails a decay with the desired properties \eqref{eq:exp-eq-alg}
under mild conditions on the potential.
By further assuming that the gradient of the potential is $\beta$-H\"older continuous
and carefully tuning the moment order ${s = \mathcal{O}(\log(d/\eps))}$ in mLSI~\eqref{eq:weak-lsi},
we can prove that, even with linearly diverging moment estimates for the \lmc iterates,
the algorithm is guaranteed to reach
the $\eps$-neighborhood of
a $d$-dimensional target $\target$ in KL-divergence
after taking the advertised number of steps
${N=\tO \raisebox{.5ex}{\big(}
  d^{\frac{1}{\beta}+\frac{1+\beta}{\beta}(\frac{2}{\alpha} - \ind{\alpha \neq 1})}
  \eps^{-\frac{1}{\beta}}\raisebox{.5ex}{\big)}}$.
In moderate dimensions ${d \ll \eps^{-1}}$,
this convergence rate does not depend on the tail growth rate $\alpha$,
and it is controlled solely by the order of smoothness $\beta$,
whereas in high dimensions ${d =\mathcal{O}(\eps^{-1})}$, the rate is determined by an interplay
between the growth rate and the order of smoothness.
The above rate also recovers the best known rate which was established for smooth potentials (${\beta=1}$) under the LSI condition~\eqref{eq:lsi}
where the tail growth is quadratic $(\alpha= 2)$.

Our contributions can be summarized as follows.
\vspace{-.1in}
\begin{itemize}[noitemsep]
\item For a potential function $f$
  whose tails behave like ${\|x\|^\alpha}$,
  and has $\beta$-\Holder continuous gradient, i.e.,
  \eq{\label{eq:our-framework}
    f(x) \sim \|x\|^\alpha\ \text{ for } \ \alpha \in [1,2],\ \  \text{ and }\ \
    \|\grad f(x)-\grad f(y)\| \leq \smooth \|x-y\|^\beta \ \ \  \forall x,y,
  }
  we prove that \lmc achieves the convergence rate ${\tO \raisebox{.5ex}{\big(}
    d^{\frac{1}{\beta}+\frac{1+\beta}{\beta}(\frac{2}{\alpha} - \ind{\alpha \neq 1})}
    \eps^{-\frac{1}{\beta}}\raisebox{.5ex}{\big)}}$ in KL-divergence.
  In moderate dimensions when ${d \ll \eps^{-1}}$,
  the tail growth rate $\alpha$ does not impact the convergence rate,
  whereas in high dimensions where ${d =\mathcal{O}( \eps^{-1})}$, tail growth
  enters the convergence rate through dimension dependency.

\item
  As a key step in deriving the above convergence rate,
  we establish a modified log-Sobolev inequality (mLSI)~\eqref{eq:weak-lsi}
  with explicit constant $\lambda$, and a target dependent moment function ${\M_s(\rho+\target)}$
  for any order $s\geq2$.
  Both of these are crucial in deriving
  a convergence rate with correct dependence
  on the dimension $d$ as well as the accuracy $\eps$.
  The final convergence result is obtained by employing the mLSI condition~\eqref{eq:weak-lsi} for
  the optimal moment order ${s= \mathcal{O}(\log(d/\eps))}$.
\item
  In lieu of \eqref{eq:our-framework}, we are mainly interested in potentials exhibiting weak dissipativity, i.e.,
  \vspace{-.1in}
  \eq{\label{eq:weak-dissip}
    \inner{x,\grad f(x)} \geq a\norm{x}^{\alpha} - b\ \ \text{ with }\ \  \alpha\in[1,2),\ \ a,b>0.
  }
  In order to use the condition mLSI~\eqref{eq:weak-lsi},
  we establish linearly diverging moment estimates for the \lmc iterates under \eqref{eq:weak-dissip}.
  Somewhat surprisingly, this is sufficient to establish the convergence of \lmc in KL-divergence.

\item
  Our convergence results are valid under finite perturbations of the potential;
  consequently, they cover sampling from non-convex potentials with at least linear growth.
  Furthermore, our results also cover the case ${\beta<1}$
  for which the potential function is not smooth; more specifically,
  it does not have a Lipschitz gradient.
  To the best of our knowledge,
  this is the first convergence result for the \lmc algorithm for weakly smooth potentials
  that exhibit subquadratic growth,
  which does not rely on regularization or Gaussian smoothing.

\item Finally, using Csisz\'ar-Kullback-Pinsker inequalities,
  the above convergence rates obtained under KL-divergence can be translated to estimates
  in total variation and $L_\alpha$-Wasserstein metrics with respective rates
  ${\tO \raisebox{.5ex}{\big(} d^{\frac{1}{\beta}+\frac{1+\beta}{\beta}(\frac{2}{\alpha} - \ind{\alpha \neq 1})}
    \eps^{-\frac{2}{\beta}}\raisebox{.5ex}{\big)}}$ and
  ${\tO \raisebox{.5ex}{\big(} d^{\frac{3}{\beta}+\frac{1+\beta}{\beta}(\frac{2}{\alpha} - \ind{\alpha \neq 1})}
    \eps^{-\frac{2\alpha}{\beta}} \raisebox{.5ex}{\big)}}$. 

\end{itemize}
Rest of the paper is organized as follows.
Section~\ref{sec:notation} reviews our notation, and Section~\ref{sec:rel_work}
surveys the related work with a detailed comparison on the existing convergence rates.
In Section~\ref{sec:main-res}, we establish the main technical results on the convergence of \lmc
for potentials with certain growth and smoothness properties.
Section~\ref{sec:implications} discusses further implications of
the tools developed in Section~\ref{sec:main-res}.
We give concrete examples in Section~\ref{sec:examples},
by applying these tools to non-convex sampling
problems that are also weakly smooth.
Proofs of the main theorems and corollaries are provided in Sections~\ref{sec:proof-mlsi}, \ref{sec:proof-moment},
\ref{sec:proof-main}, \ref{sec:convex-proof} and \ref{sec:post-proof} in order of appearance of their statement in the main text.
Finally, we conclude in Section~\ref{sec:conclusion} with brief remarks on future work.

\subsection{Notation}\label{sec:notation}
For a real number ${x\in\reals}$, we denote its absolute value with $\abs{x}$.
We denote the $p$-norm of a vector ${x\in\reals^d}$ with $\norm{x}_p$
and whenever ${p=2}$, we omit the subscript and simply write ${\|x\| \defeq \|x\|_2}$ to ease the notation.
For a matrix ${A\in\reals^{d\times k}}$, $A_{ij}$ denotes its entry in the $i$-th row and $j$-th column,
and whenever ${d=k}$,  its trace is denoted by ${\Tr(A) = \Sigma_{i=1}^d A_{ii}}$.
We use $\id$ to denote the identity matrix in $d$-dimensions.

For a function ${f:\reals^d\to\reals}$, we define its infinity norm as ${\|f\|_\infty = \sup_{x\in\reals^d}|f(x)|}$.
${\M_s(f)}$ is used to denote the modified $s$-th moment of the function $f$
(which is not necessarily a distribution),
defined as ${\M_s(f)=\int f(x) (1+\|x\|^2)^{s/2}dx}$.
The gradient and the Hessian of $f$ are denoted by ${\grad f(x)}$ and ${\Hess f(x)}$, respectively,
where the derivatives are with respect to $x$. For a statement $A$,
the indicator function is denoted with
$\ind{A}$ and defined as
\eqn{
  \ind{A} = \left\{
    \begin{array}{ll}
    1 & \text{if $A$ is true,}\\
    0 & \text{otherwise.}\\
    \end{array}
    \right.
}

We use $\EE{ x}$ to denote the expected value of a random variable $x$,
where expectations are over all the randomness inside the brackets.
For probability densities $p$,$q$ on $\reals^d$, we use $\KL{p}{q}$ and $\I{p}{q}$ to denote
their KL-divergence (or relative entropy) and relative Fisher information, respectively, which are defined as
\begin{align*}
  &\KL{p}{q} = \int p(x) \log{\frac{p(x)}{q(x)}} dx,
    \ \ \ \text{ and }\ \ \ 
    \I{p}{q} = \int p(x) \Big\|\grad \log{\frac{p(x)}{q(x)}}\Big\|^2 dx.
\end{align*}
Similarly, we denote the entropy of $p$ with
${\ent{p} = -\int p(x) \log{p(x)}dx}$.
Denoting the Borel $\sigma$-field of $\reals^d$ with ${\mathcal{B}(\R^d)}$,
$L_\alpha$-Wasserstein for ${\alpha >0}$ and total variation metrics
are defined as
\begin{align*}
  & \wasserpp{p}{q}{\alpha} = \inf_{\nu} \left(\int \norm{x-y}^\alpha d \nu(p,q)\right)^{{1}/{\alpha}},
  &
  & \tv{p}{q} = \sup_{A \in \mathcal{B}(\R^d)} \abs {\int_A p(x)dx - \int_Aq(x)dx},
\end{align*}
where in the first formula, infimum runs over the set of probability measures
on ${\reals^d \times \reals^d}$ that has marginals with corresponding densities $p$ and $q$.

Finally, $\*O$ and $\tO$ notations are frequently used to describe the dependence of a function $f$ on another function $g$,
and defined in the following sense
\eqn{
  f(x) = \*O(g(x)) \implies
  \limsup_{x\to \infty} \frac{ f(x)}{ g(x)} <\infty,\ \text{ and }\
  f(x) = \tO(g(x)) \implies \limsup_{x\to \infty} \frac{ f(x)}{ g(x)\log(g(x))^k} <\infty,
}
for some $k\geq 0$, where $\tO$ simply ignores the logarithmic factors.
We use $f(x)\lesssim g(x)$ instead of $f(x) \leq \*O(g(x))$ to improve readability.

\section{Related Work}\label{sec:rel_work}
The \lmc algorithm has been extensively studied in the context of sampling from
a log-concave target distribution.
Earlier results focused on characterizing its bias
which is also referred to as the integration error~\cite{milstein1994numerical,milstein2013stochastic},
and the convergence guarantees were mostly asymptotic~\cite{gelfand1991recursive,meyn2012markov}.
Non-asymptotic analysis of \lmc has drawn a lot of interest recently~\cite{dalalyan2012sparse,dalalyan2017furthur,dalalyan2017theoretical,
  durmus2019analysis,cheng2018convergence,cheng2018sharp,vempala2019rapid, dalalyan2019user,brosse2019tamed}
where the focus was on potentials exhibiting strong tail growth properties.
These papers were mostly influenced by the pioneering works
by Dalalyan~\cite{dalalyan2017theoretical}, and Durmus and Moulines~\cite{durmus2016sampling,durmus2017nonasymptotic}
where it was shown that for strongly convex and smooth potentials,
\lmc reaches $\eps$ accuracy in terms of total variation (TV) distance
after $\tO(d\eps^{-2})$ steps.
Similarly, ${\tO\left( d\eps^{-2}\right)}$ steps are sufficient to reach
$\eps$ accuracy under the $L_2$-Wasserstein distance~\cite{durmus2019high},
which can be further improved to ${\tO\left( d \eps^{-1}\right)}$ under an additional second-order
smoothness assumption on the potential function.

In this paper, we establish guarantees under KL-divergence (relative entropy) which can be
easily translated to TV and Wasserstein metrics using Csisz\'ar-Kullback-Pinsker (CKP)~\cite{bolley2005weighted}
and/or Talagrand inequalities~\cite{talagrand1996transportation,otto2000generalization}.
For strongly convex and smooth potentials,
it is known that ${\tO\left(d\eps^{-1}\right)}$ steps of \lmc yield an $\eps$ accurate
sample in KL-divergence~\cite{cheng2018convergence,durmus2019analysis}.
This is still the best known rate in this setup,
and recovers the best known rates in TV~\cite{durmus2017nonasymptotic,dalalyan2017theoretical}
as well as in $L_2$-Wasserstein metrics~\cite{durmus2019high}.
However, for convex and smooth potentials that grow like $\|x\|^\alpha$,
the rate drops to ${\tO\big(d^{1+\frac{4}{\alpha}}\eps^{-3}\big)}$ due to lack of strong convexity~\cite{cheng2018convergence}.
Among various contributions of \cite{durmus2019analysis},
LMC was also analyzed for convex potentials, but
their result does not yield a convergence guarantee for the last \lmc iterate.

Existing results that establish the fast convergence of \lmc require strong curvature conditions on the potential function; therefore,
their applicability is limited. Recently, it has been observed that global curvature assumptions
can be relaxed to the tails of the potential~\cite{eberle2016reflection,eberle2019couplings}.
For example, \cite{cheng2018sharp} extended these results to sampling from smooth potentials that are strongly convex outside of a compact set,
obtaining the same dimension and $\eps$ dependency in the strongly convex case at the expense
of an exponential dependence in the radius of the compact set.
Similarly, \cite{vempala2019rapid} established convergence guarantees for
target distributions that satisfy a log-Sobolev inequality. This corresponds to potentials with quadratic tails~\cite{bakry1985LSI,bobkov1999exponential}
up to finite perturbations~\cite{holley1987LSI};
thus, this result is able to deal with non-convex potentials that are not limited to a compact set,
while achieving the same convergence rate of ${\tO\left(d\eps^{-1}\right)}$ in KL-divergence.

Convergence of the LMC algorithm is very little understood when the potential is weakly smooth.
Contrary to previous work, our focus is on the convergence of 
vanilla LMC~\eqref{eq:ULA} without requiring any modifications on the algorithm
such as methods based on proximal mapping~\cite{atchade2015moreau, luu2017sampling, durmus2018efficient, mou2019efficient, durmus2019analysis},
Gaussian smoothing~\cite{chatterji2019langevin, doan2020weakly},
or mirror mapping~\cite{hsieh2018mirrored}.
We also do not assume a composite structure on the potential,
in which case the potential is given by $f(x) = U(x)+\psi(x)$ where
$\psi(x)$ is a strongly convex and smooth function, and $U(x)$ is a convex function with
$\beta$-\Holder continuous gradient.
This assumption enforces a quadratic tail growth on the potential,
in which case,
\cite{chatterji2019langevin} established the convergence rate of
${\tO \left(d^{2+1/\beta} \eps^{-2/\beta} \right)}$ in
total variation distance.
Furthermore, we focus on the last \lmc iterate
which characterizes the practical performance of this algorithm, in contrast to
\cite{durmus2019analysis} which provided guarantees
for the average of the distributions of the LMC iterates.

Our analysis draws heavily on the theory of diffusion
processes~\cite{bakry2013analysis,toscani2000trend} -- more specifically,
logarithmic Sobolev inequalities. These inequalities were first established for
the Gaussian density~\cite{gross1975logarithmic}, and later generalized to
Gibbs measure with a strongly convex potential by Bakry and \'Emery~\cite{bakry1985LSI}.
Combined with the Holley and Stroock's perturbation lemma~\cite{holley1987LSI},
this theory covers a wide range of potentials that can be represented as
a finite perturbation of a strongly convex function.
It is well-known that the overdamped Langevin diffusion~\eqref{eq:overdamped}
follows the gradient flux or the steepest descent of KL-divergence with respect to the $L_2$-Wasserstein metric~\cite{jordan1998variational}.
Building on this, sampling with a diffusion can be seen
as an optimization algorithm in the space of probability distributions~\cite{wibisono2018sampling,vempala2019rapid,ma2019there};
similarly, LSI can be interpreted as a gradient domination condition in this space,
which is commonly referred to as the PL-inequality~\cite{polyak1963gradient} in the optimization theory.
LSI and PL-inequality both yield exponential convergence in their corresponding space~\cite{polyak1963gradient,karimi2016linear,toscani1999entropy,carlen1991entropy}.
Further promoting this analogy,
PL-inequality is a special case of \loja \ inequality~\cite{lojasiewicz1963propriete},
and their counterparts are considered recently in \cite{blanchet2018family} in the space of functionals.
Thus, the modified LSI introduced in \cite{toscani2000trend}, can be viewed
as a modified version of the \loja \ inequality in the space of probability distributions.
For a survey about the convergence properties of diffusion processes with the Fokker-Planck equation governing their evolution (including overdamped Langevin dynamics~\eqref{eq:overdamped}) and several inequalities from functional analysis,
we refer the reader to \cite{markowich99onthe}.
Finally, the analogy between optimization and sampling provided invaluable insights,
in many cases improving our understanding, and ultimately the performance of various algorithms~\cite{zhang2017hitting,brosse2017sampling, brosse2018promises,
  chatterji2018theory, bhatia2019bayesian,hsieh2018mirrored,ma2019sampling}.

It is worth mentioning that the
rates we discussed in this section can be further improved
by making higher order smoothness assumptions
on the potential function~\cite{mou2019improved}, or by considering
higher order numerical integrators~\cite{li2019stochastic,shen2019randomized},
or by certain adjustments~\cite{durmus2017fast, ge2018simulated, dwivedi2019log}.
The overdamped Langevin diffusion \eqref{eq:overdamped} considered in this work
is first order,
and its higher order versions such as underdamped
\cite{cheng2018underdamped,ma2019there},
or third-order schemes
\cite{ma2015complete, mou2019high} may also provide additional improvements.

\setlength{\tabcolsep}{1.3pt}
\begin{table}[t]\small
  \begin{center}
    \begin{tabular}{c c c c c c}
      \Xhline{4\arrayrulewidth}
      \textsc{Work} & \textsc{Convergence Rate} & \textsc{Smoothness} & \textsc{Curvature} & \textsc{Perturbation} & \textsc{Distance} \\ 
      \hline
      \cite{cheng2018convergence,durmus2019analysis} & $\tO \raisebox{.5ex}{\big(} d\eps^{-1} \raisebox{.5ex}{\big)}$
                                                & \makecell{Lipschitz\\gradient} & \makecell{Strongly\\Convex} & None & KL\\
      \hline
      \cite{vempala2019rapid} & $\tO\left(d\eps^{-1}\right)$ & \makecell{Lipschitz\\gradient} & \makecell{Strongly\\Convex}
                                                & \makecell{Bounded\\difference} & KL \\ 
      \hline
      \cite{cheng2018convergence} & $\tO \raisebox{.5ex}{\big(} d^{1+\frac{4}{\alpha}}\eps^{-3} \raisebox{.5ex}{\big)}$
                                                & \makecell{Lipschitz\\gradient} & \makecell{Convex\\ Growth rate $\alpha$} & None & KL\\
      \hline
      \textbf{This work} &${\tO \raisebox{.5ex}{\big(}d^{\frac{1}{\beta} + \frac{1+\beta}{\beta}\left(\frac{2}{\alpha}-\ind{\alpha \neq 1}\right)} \eps^{-\frac{1}{\beta}}\raisebox{.5ex}{\big)}}$ %
                                                & \makecell{$\beta$-\Holder\\gradient} & \makecell{Tail growth\\
       $\sim\|x\|^\alpha$} & \makecell{Bounded\\difference} & KL\\ 
      \Xhline{3\arrayrulewidth}
      \cite{dalalyan2017theoretical,durmus2017nonasymptotic} & $\tO\left(d \eps^{-2}\right)$ & \makecell{Lipschitz\\gradient} & \makecell{Strongly\\convex} & \makecell{None} & TV\\ 
      \hline
      \cite{dalalyan2017theoretical} & $\tO\left(d^3 \eps^{-4}\right)$ & \makecell{Lipschitz\\gradient} & \makecell{Convex} & \makecell{None} & TV\\ 
      \hline
      \cite{chatterji2019langevin} & $\tO\big(d^{2+\frac{1}{\beta}} \eps^{-\frac{2}{\beta}}\big)$ & \makecell{Lipschitz+$\beta$-\Holder\\gradient} & \makecell{Strongly\\Convex} & \makecell{None} & TV\\
      \hline
      \textbf{This work} & ${\tO \raisebox{.5ex}{\big(} d^{\frac{1}{\beta}+\frac{1+\beta}{\beta}(\frac{2}{\alpha} - \ind{\alpha \neq 1})}
                           \eps^{-\frac{2}{\beta}}\raisebox{.5ex}{\big)}}$
                                                & \makecell{$\beta$-\Holder\\gradient} & \makecell{Tail growth\\
      $\sim\|x\|^\alpha$} & \makecell{Bounded\\difference} & TV\\
      \hline
    \end{tabular}
    \caption{
      List of convergence rates in KL-divergence and TV distance
      for the \lmc~\eqref{eq:ULA} algorithm in various papers and their accompanying assumptions.
      Comparison is made with results relying only on first order smoothness.
      For additional information, refer to Section~\ref{sec:comp}.}
    \label{tab:comp}
  \end{center}
  \vspace{-1.em}
\end{table}
\subsection{Comparison }\label{sec:comp}
In Table~\ref{tab:comp}, we compare the assumptions and results of this paper
to those of existing works that only make the first order smoothness assumption.
Among these,
\cite{dalalyan2017theoretical,durmus2017nonasymptotic,cheng2018convergence,chatterji2019langevin,vempala2019rapid,durmus2019analysis} are in the quadratic growth regime, and achieve the best rates known to authors.
Our results recover the convergence rate of \cite{vempala2019rapid} for smooth potentials ($\beta=1$)
satisfying the LSI condition ($\alpha=2$).
\cite{cheng2018convergence,dalalyan2017theoretical} establish guarantees for convex and smooth potentials;
however, these rates drop significantly under lack of strong convexity, and cannot tolerate perturbations
on the potential.
In contrast to these results, our analysis provides a continuous interpolation
in both the growth rate $\alpha \in (1,2]$,
and order of smoothness $\beta\in(0,1]$.
In case of linear growth when $\alpha=1$, there is no convexity in the tails which is why
the convergence loses an additional factor in dimension dependency.
The results of \cite{chatterji2019langevin} on the vanilla \lmc require the potential to
have a composite structure, namely,
$f(x) = U(x)+\psi(x)$ where
$\psi(x)$ is a strongly convex and smooth function, and $U(x)$ is a convex function with
$\beta$-\Holder continuous gradient.
It is worth emphasizing that the actual rate obtained in \cite{cheng2018convergence}
is $\tO \big( d\eps^{-3}\times\wasserp{\rho_0}{\target}{4}\big)$,
and depends polynomially on the $L_2$-Wasserstein distance between the initial distribution and the target,
whereas other works depend logarithmically on this difference in terms of KL-divergence.
For a potential growing with rate $\alpha$,
one may show $\wasserp{\rho_0}{\target}{2} \lesssim d^{2/\alpha}$ justifying the reported rate in Table~Table~\ref{tab:comp}.
For details of the initializations
when $\alpha=2$, we refer to \cite{cheng2018underdamped}.

\section{Main Results}\label{sec:main-res}
Convergence rates of diffusion-based algorithms have been the subject of growing attention
recently with many applications related to sampling with MCMC and non-convex optimization.
Algorithms based on Langevin diffusion have been particularly of interest
where the fast convergence of the algorithm has been frequently linked to
the quadratic growth of the potential~\cite{erdogdu2018global,cheng2018sharp,vempala2019rapid}.
Conversely and somewhat surprisingly, we prove that the rate of convergence,
in terms of its dependence on the accuracy $\eps$, is not directly influenced by the 
growth behavior of the potential function $f$
as long as the growth is at least linear. Therefore in moderate dimensions,
the smoothness properties of the potential entirely determines the rate.
Furthermore, our results show that the tail growth rate of the potential
impacts the performance in high dimensions, as
it enters the convergence rate only through dimension dependency,
in which case the convergence is determined by the interplay between the tail growth rate and
the weak smoothness degree.

We develop our explicit bounds on the convergence rate of the \lmc algorithm in three key steps.
First, in Theorem~\ref{thm:MLSI}, we prove a modified log-Sobolev inequality (mLSI) for
a class of asymptotically convex degenerate potentials
described in Assumption~\ref{as:degen_conv},
which can accommodate for sub-quadratic tail growth.
The condition mLSI relies on the moments of the Markov chain defined by the iterates of \lmc;
thus, in Proposition~\ref{prop:disc_mom_bound},
we prove that any order moments of the \lmc iterates grow
at most linearly in the number of iterations, an estimate that is diverging in the limit.
Finally in Theorem~\ref{thm:main}, we invoke these two results for an arbitrary moment order
and establish a general convergence result, which in turn yields the main result of this paper
after tuning the moment order in Corollary~\ref{cor:main}.
We focus on the following class of potentials functions.
\begin{assumption}[Degenerate convexity at $\infty$]\label{as:degen_conv}
 \  The potential function $f(x)$ is degenerately convex at infinity in the sense that
  there exist a function ${\tf : \reals^d \to \reals}$ such that for a constant ${\pert \geq 0}$
\begin{equation*}
\big\|f-\tf\big\|_\infty \leq \pert,
\end{equation*}
where $\tf$ satisfies, 
\eq{\label{eq:perturb-hessian}
  \Hess \tf(x) \succeq \displaystyle \frac{\decof}{\big(1+\frac{1}{4}\|x\|^2\big)^{{\decopow}/{2}}} \id,
}
for some ${\decof>0}$ and ${\decopow \geq 0}$.
\end{assumption}
The above condition allows for finite perturbations,
and consequently permits sampling from non-convex potentials;
thus, the determinant factor is the tail growth properties of the potential function.
The boundary case ${\decopow=0}$ corresponds to quadratic tail growth,
and whenever $\decopow >0$,
due to the decaying nature of the lower bound on the Hessian,
the potential function exhibits no convexity at infinity.
For example, consider the following potential function
$f(x) = \|x\|^\alpha$ for $\alpha \in [1,2]$.
The case $\alpha=2$ corresponds to quadratic growth with $\decopow=0$,
and it is easy to see that for a superlinear tail $\alpha\in(1,2]$, one has $\decopow=2-\alpha$.
However, when the tail is exactly linear with $\alpha=1$,
the assumption can be shown to hold for any ${\decopow>2}$.

It is known that the LSI condition~\eqref{eq:lsi} is not satisfied when ${\alpha<2}$,
for example for the potential ${f(x)=|x|^\alpha+c}$
(see e.g.~\cite{bobkov1999exponential}); therefore,
for the above class of potentials,
we state the following log-Sobolev-type inequality.
\begin{theorem}[mLSI]\label{thm:MLSI} 
  If the potential ${f=-\log \target}$ satisfies Assumption~\ref{as:degen_conv},
  then the following inequality holds for all ${s \geq 2}$,
  \begin{equation}\label{eq:MLSI}
   \forall \rho,\ \  \KL{\rho}{\target} \leq \lambda \I{\rho}{\target}^{1-\delta} \M_s(\rho + \target)^{\delta},
 \end{equation}
 where ${\M_s(\rho) = \int  (1+\| x \|^2)^{s/2} \rho(x) dx}$ is the $s$-th moment of any function $\rho$, and
 $\delta$ and $\lambda$ are constants that depend on $s$, and defined as
  \begin{align*}
    & \delta\defeq \frac{\decopow}{s-2+2\decopow} \in [0, 1/2),\\
    &\lambda \defeq 4e^{2\pert}\decof^{-\frac{s-2}{s-2+2\decopow}}.
  \end{align*}
\end{theorem}
The constants $\lambda$ and $\delta$ are explicit,
and the above inequality reduces exactly to the LSI condition~\eqref{eq:lsi}
up to the absolute constant 4
when ${\decopow = 0}$ and ${\pert = 0}$, in which case
the potential function $f$ is strongly convex.
Notice that the moment term ${\M_s(\rho + \target)}$ depends on both $\rho$ and $\target$,
which is crucial in deriving a convergence rate with correct dimension and accuracy dependence.
Modified LSI-type inequalities such as \eqref{eq:MLSI} as well as weak Poincar\'e inequalities appear
in the analysis of diffusion operators~\cite{bakry2013analysis}.
The mLSI condition~\eqref{eq:MLSI} is similar in nature to the modified LSI of~\cite{toscani2000trend};
yet, the latter was established for the purpose of proving the rate $\mathcal{O}(t^{-\infty})$
for the diffusion process~\eqref{eq:overdamped},
and will yield a convergence rate that is worse than what will be established below in Corollary~\ref{cor:main}.
It also cannot recover the existing rates (e.g. \cite{vempala2019rapid}) in the limit case $\alpha \to 2$.
Our proof builds on the construction made in \cite{toscani2000trend}
and uses the results of \cite{bakry1985LSI,holley1987LSI},
which we defer to Section~\ref{sec:proof-mlsi}.

The gradient of the potential function is employed as the drift of Langevin diffusion~\eqref{eq:overdamped},
and it also governs its discretization, the \lmc algorithm \eqref{eq:ULA}.
The growth behavior of this term is regulated in
the following assumption which should
be seen as a relaxation to the standard 2-dissipativity condition,
${\inner{\grad f(x), x} \geq a \norm{x}^2 -b}$ for some $a,b>0$~\cite{mattingly2002ergodicity,meyn2012markov}.
\begin{assumption}[$\alpha$-dissipativity \& $\zeta$-growth of gradient]\label{as:mild_dis}
For ${\alpha\in [1,2]}$ and ${a,b>0}$, we have
\begin{equation}\label{eq:mild-dis}
  \inner{\grad f(x), x} \geq a \norm{x}^{\alpha} - b\ \ \text{ for all }\ \  x\in\reals^d.
\end{equation}
Moreover, for a positive constant ${\zeta \leq \alpha/2}$, the gradient satisfies the following growth condition,
\begin{equation}\label{eq:grad-growth}
  \norm{\grad f (x)} \leq \growth(1+\norm{x}^{\zeta})\ \ \text{ for all }\ \ x\in\reals^d.
\end{equation}
\end{assumption}
Note that when the tail growth is superlinear $\alpha \in (1,2]$, the parameter $\decopow$ in Assumption~\ref{as:degen_conv}
satisfies $\decopow = 2-\alpha$ where $\alpha$ is as in Assumption~\ref{as:mild_dis}.
The key difference between the cases ${\alpha=2}$ and ${\alpha<2}$ is that the former implies that
the \lmc iterates have uniformly bounded moments of all orders~\cite{erdogdu2018global},
whereas in the latter case,
obtaining such a uniform bound is an open problem.
This poses significant challenges in the proof.
That is, we establish that the moments of \lmc can diverge at most linearly,
and even though it is not immediately clear that \lmc even converges in this setup,
we are able to show that this estimate is sufficient to establish
a non-asymptotic convergence rate for the algorithm.
It is also worth noting that under an additional condition on the gradient perturbation, i.e.
${\|\grad f - \grad \tf \|_\infty\leq \pert}$, it can be shown that
\eqref{eq:perturb-hessian} implies \eqref{eq:mild-dis} in
Assumption~\ref{as:mild_dis} (cf. Lemma~\ref{lemma:as_not_needed});
however, the above setting is more general and covers a wider range of potentials,
justifying the current presentation.

In a representative analysis of \lmc, one considers a sequence of diffusion processes
${\{ \xtild_{k,t}\}_{k\in\naturals,t\geq 0}}$ where
each iteration $x_{k+1}$ of the \lmc algorithm \eqref{eq:ULA} can be written as $\xtild_{k,\eta}$ where
\eq{\label{eq:inter-lmc}
  d\xtild_{k, t} = -\grad f(x_k) d t + \sqrt{2} d B_t\ \ \text{ with }\ \ \xtild_{k,0}=x_k,
}
for an appropriate Brownian motion $B_t$.
Denoting the distribution of $\xtild_{k,t}$ with $\rhotild_{k,t}$,
it can be shown that the time derivative of the KL-divergence between $\rhotild_{k,t}$ and the target,
${{d \KL{\rhotild_{k,t}}{\target}}/{d t}}$,
reduces to the negative relative Fisher information ${- \I{\rhotild_{k,t}}{\target}}$
up to an additive error term
that depends on the difference between the \lmc iterate $x_k$ and the its interpolating
diffusion $\xtild_{k,t}$ (see for example \cite[Proof of Lemma~3]{vempala2019rapid}),
which yields the inequality
\eq{\label{eq:weak-LSI-lmc}
  \forall k \in \naturals, \forall t\geq 0,\ \  \frac{d}{d t} \KL{\rhotild_{k, t}}{\target} \leq 
  -\frac{3}{4}\I{\rhotild_{k, t}}{\target}
  + \EE{\norm{\grad f(\xtild_{k,t}) - \grad f(x_k)}^2}.
}
Combining this with mLSI~\eqref{eq:MLSI} for ${\rho = \rhotild_{k,t}}$,
one obtains the following differential inequality for the interpolating diffusion process~\eqref{eq:inter-lmc},
\eq{\label{eq:main-differential-inequality}
  \frac{d}{d t} \KL{\rhotild_{k, t}}{\target} \leq
  -\frac{3}{4\lambda}\KL{\rhotild_{k, t}}{\target}^{\frac{1}{1-\delta}}
  \M_s(\rhotild_{k, t} + \target)^{-\frac{\delta}{1-\delta}}
  + \EE{\norm{\grad f(\xtild_{k,t}) - \grad f(x_k)}^2}.
} 
The convergence rate of \lmc can be derived by analyzing the differential inequality \eqref{eq:main-differential-inequality},
which requires appropriate estimates
on the Markov chain moments ${\M_s(\rhotild_{k, \eta} + \target)}$ defined by the iterates of \lmc 
as well as the additive error $\EE{\norm{\grad f(\xtild_{k,t}) - \grad f(x_k)}^2}$.
The following proposition establishes the former with a linearly growing upper bound in the number of iterations.
We defer to proof to Section~\ref{sec:proof-moment}.
\begin{proposition}\label{prop:disc_mom_bound}
  If the potential ${f=-\log \target}$ satisfies Assumption~\ref{as:mild_dis},
  then denoting the distribution of the $k$-th iterate of \lmc with $\rho_k$,
  for a step size satisfying ${\eta \leq \frac{1}{2}\big(1 \wedge \tfrac{a}{2\growth^2}\big)}$,
  we have
\eq{\label{eq:linear-bound}
    \M_{s}(\rho_k + \target) \leq \M_{s}(\rho_0 + \target) + C_{s} k \eta,\ \ \text{ for even integer}\ \ s\geq 2,
}
  where
  \begin{equation}\label{eq:C_s_closed}
    C_{s} \defeq
    \left(\frac{3a+2b+3}{1 \wedge a} \right)^{\frac{s-2}{\alpha}+1}
    s^{s} d^{\frac{s-2}{\alpha}+1}.
  \end{equation}
\end{proposition}
Although the bound \eqref{eq:linear-bound} grows linearly with the number of iterations and
diverges in the limit $k\to\infty$, this estimate is sufficient to establish a global
convergence guarantee for the \lmc algorithm. 
The leading coefficient in the bound $C_{s}$ \eqref{eq:C_s_closed} is of order $\*{O}(d^{\frac{s-2}{\alpha}+1})$
which is the same order as in the continuous-time case (cf. Lemma~\ref{lemma:moment_bound}),
and will help obtain an accurate dimension dependency in the final convergence rate.

In the following, we make an assumption on the order of smoothness of the potential
function $f$ in order to obtain an estimate for the additive error term in the differential
inequality \eqref{eq:main-differential-inequality}.
In this context, order of smoothness refers to the \Holder exponent of the gradient of the potential,
which is defined explicitly below.
\begin{assumption}[Order of smoothness]\label{as:holder}
  The potential function $f$ is differentiable with $\beta$-\Holder continuous gradient with constant $\smooth$, i.e.
\begin{equation}\label{eq:holder-c}
  \norm{\grad f(x) - \grad f(y)} \leq \smooth \norm{x-y}^{\beta}\ \ \text{ for all }\ \ x,y\in\reals^d,
\end{equation}
where the order of smoothness $\beta$ satisfies ${\zeta \leq \beta \leq 1}$ for the constant $\zeta$ in \eqref{eq:grad-growth}.
\end{assumption}
Potentials with order of smoothness $\beta=1$ are termed as \emph{smooth} and those with $\beta<1$ are often
referred to as
\emph{weakly smooth} in the literature \cite{chatterji2019langevin,nesterov2015universal},
a term that is borrowed from optimization theory.
Our results cover potentials satisfying \eqref{eq:holder-c} for any ${\beta \in (0,1]}$.
In the case of weakly smooth potentials, existing results on vanilla \lmc can only
explain the convergence in the setting $\alpha=2$, 
but even for this case,
the conclusions that will be made in Corollary~\ref{cor:main} are new and improves the best known rates.

We reiterate that the main use of the above assumption is to control the additive error term due to discretization,
the second term on the right hand side of \eqref{eq:main-differential-inequality},
with a bound that behaves like $\eta^{\beta}$ (cf. Lemma~\ref{lemma:kl_deriv_bound}).
This term determines the accuracy $\eps$ dependence of the convergence rate.

$\beta$-\Holder continuity already imposes a growth condition on
the gradient \eqref{eq:grad-growth} with ${\zeta=\beta}$.
However, we state these separately as the order of smoothness $\beta$ and the growth rate $\zeta$ need not be the same;
a smaller growth rate on the gradient improves certain estimates in the main result,
which in turn allows us to cover a wider class of potentials.
For example, the function ${f(x) = \left(1+ x^2 \right)^{1/2}}$ is smooth with Lipschitz
gradient, but its gradient is also bounded implying ${\zeta=0}$.
One cannot simply use $\zeta = 1$ since the condition $\zeta \leq \alpha/2$ in Assumption~\ref{as:mild_dis}
implies that $\alpha\geq 2$ which is clearly not true.
Hence, keeping the gradient growth and the order of smoothness separate allows us to
cover a wider range of potentials. The relationship among these parameters can be summarized as given below,
\eq{
  2\zeta  \leq \alpha \leq \zeta + 1 \leq \beta +1.
}
If one requires quadratic growth on the potential, i.e. $\alpha=2$,
this immediately implies that the smoothness order is at least 1, i.e $\beta \geq1$,
which limits the applicability of the results to only smooth potentials with Lipschitz gradient.

Before we present the main technical result of this paper, we note that when ${\alpha>1}$,
all these assumptions are satisfied for potentials of the form ${f(x)= \|x\|^\alpha +c}$,
and their bounded perturbations, e.g.
\eq{
  f(x)= \|x\|^\alpha +\sin(x)+c.
} 
This potential is non-convex and it does not have a Lipschitz gradient, and serves as a canonical example
that demonstrates the wide applicability of the following result.

\begin{theorem}\label{thm:main}
  Suppose the potential ${f=-\log \target}$ satisfies Assumptions~\ref{as:degen_conv}, \ref{as:mild_dis}, \ref{as:holder},
  and denote the distribution of the $k$-th iterate of \lmc with $\rho_k$.
  Then,
  for a sufficiently small $\eps$ satisfying $\eps \leq \psi$ where
  $\psi$ is defined in \eqref{eq:tol-up},
  and for some $\Delta_0>0$ upper bounding the error at initialization,
  i.e. ${\KL{\rho_0}{\target} \leq \Delta_0}$,
  if the step size satisfies
  \begin{equation}\label{eq:main_step}
    \begin{split}
      \eta &= (\sigma c_\gamma)^{-\frac{1}{1+\beta}} d^{-\frac{\alpha+\decopow}{\alpha\beta}-\frac{\gamma}{\beta+1}}
      \left(1+ (1-\alpha/2)\log(d) \right)^{-\frac{1}{\beta}}
      \log{\left(\frac{\Delta_0}{\eps}\right)}^{-\frac{\gamma}{1+\beta}}
      \left(\frac{2}{\eps}\right)^{-\frac{1}{\beta} - \frac{\gamma}{1+\beta}},
    \end{split}
  \end{equation}
  then the \lmc iterates reach $\epsilon$-accuracy of the target, i.e.  ${\KL{\rho_N}{\target} \leq \eps}$,
  after $N$ steps for
  \begin{equation}\label{eq:step_count}
    \begin{split}
      N = c_\gamma d^{\frac{\alpha+\decopow+\beta\decopow}{\alpha\beta} + \gamma}
      \left(1+ (1-\alpha/2)\log(d) \right)^{\frac{1}{\beta}}
      \log{\left(\frac{2\Delta_0}{\eps}\right)}^{1+ \gamma}
      \left(\frac{2}{\eps}\right)^{\frac{1}{\beta} + \gamma},
    \end{split}
  \end{equation}
  where $\gamma$ is given by
  \eq{\label{eq:gamma}
    \gamma \defeq \gamma(s)= \frac{(1+\beta)\decopow}{\beta (s-2)}\ \ \text{ for any even integer}\ \ s \geq 4,
  }
  and $\sigma$ and $c_\gamma$ are constants given as
  \eq{
     \sigma= & 4\smooth^2 \left(
      1 + 2 a^\beta\left[
        1 + \tfrac{2 \alpha}{a}\big(\log\left({16\pi}/{a}\right) +
        \growth \left({2}+{2b}/{a}\right)^2 + b +\abs{f(0)}\big) 
      \right]
    \right),\\
    c_\gamma = & \sigma^{\frac{1}{\beta}}
    (16\lambda)^{1+\frac{1}{\beta}+2\gamma}
    \Bigg(
      \frac{\M_{s}(\rho_0 + \target)}{16d^{\frac{s-2}{\alpha}+1}} 
      \vee
      \frac{s^{s}}{16}
      \left(\frac{3a+2b+3}{1 \wedge a}\right)^{\frac{s-2}{\alpha}+1}
    \Bigg)^{\gamma}.
  }
\end{theorem}
The above theorem, proved in Section~\ref{sec:proof-main}, implies that for smooth potentials that satisfy the LSI condition~\eqref{eq:lsi}, i.e.
${\alpha=2}$ and ${\beta=1}$, we have ${\gamma=\decopow=0}$;
thus,
LMC achieves the convergence rate of
${\tO\left(d\eps^{-1}\right)}$, recovering the rate established by \cite{vempala2019rapid}.
In the general case, Theorem~\ref{thm:main} implies
the convergence rate
${\tO\raisebox{.5ex}{\big(}\gamma^{-1}d^{\frac{\alpha+\decopow+\beta\decopow}{\alpha\beta} + \gamma}
  \eps^{-\frac{1}{\beta} - \gamma}\raisebox{.5ex}{\big)}}$
where ${\gamma>0}$ is given in \eqref{eq:gamma} and can be arbitrarily small.
However, note that one cannot simply let $\gamma \to 0$ by taking the limit ${s \to \infty}$.
For any other potential function with subquadratic tail growth $\alpha <2$,
the parameter $\gamma$ requires tuning.

The bound on accuracy \eqref{eq:tol-up} is $\*O(1)$, depending on the fixed problem parameters
and the bound on the initial KL-divergence $\Delta_0$. When initialized with
a Gaussian, $\Delta_0$ can also be characterized with
the fixed problem parameters (cf. Lemma~\ref{lemma:init_point}).
More importantly, the upper bound on $\eps$, as stated in \eqref{eq:tol-up},
does not depend on the moment order $s$, which
enables us to choose ${s=\*O\left( \log\left({d}{\eps^{-1}}\right)\right)}$ and accordingly
$\gamma = \*O\big( 1/\log\left({d}{\eps^{-1}}\right)\big)$,
which in turn yields the optimal convergence rate.
This is formalized in the next corollary which is the main result of this paper.
\begin{corollary}\label{cor:main}
  Suppose the potential ${f=-\log \target}$ satisfies Assumptions~\ref{as:degen_conv}, \ref{as:mild_dis}, \ref{as:holder},
  and denote by $\rho_k$, the distribution of the $k$-th iterate of \lmc
  initialized with ${x_0 \sim \Gsn(x,\id)}$ for any $x \in \reals^d$ and $\Delta_0$ 
  upper bounding the error at initialization (cf. Lemma~\ref{lemma:init_point}).
  Then, for a sufficiently small $\eps$ satisfying $\eps \leq \psi$ where
  $\psi$ is defined in \eqref{eq:tol-up},
  if the step size satisfies \eqref{eq:main_step} for
  ${s=2+2\ceil{\log(\frac{6d}{\eps})}}$,
  the iterates of \lmc reaches $\epsilon$-accuracy of the target,
  i.e.  ${\KL{\rho_N}{\target} \leq \eps}$, after $N$ steps satisfying
  \eqn{
    N \leq 
    c
    d^{\frac{\alpha+\decopow+\beta\decopow}{\alpha\beta}}
      \left(1+ (1-\alpha/2)\log(d) \right)^{\frac{1}{\beta}}
      \log{\left(\frac{2\Delta_0}{\eps}\right)}^{1+ \frac{(1+\beta)\decopow}{2\beta}}
      \left(2 + 2\left\lceil\log\left(\frac{6d}{\eps}\right)\right\rceil\right)^\frac{2(1+\beta)\decopow}{\beta}
      \left(\frac{2}{\eps}\right)^{\frac{1}{\beta}},
    }
    where $c$ is a constant independent of $d$ and $\epsilon$, and given as
  $$
    c =
    e^{\frac{(1+\alpha)(1+\beta)\decopow}{\alpha\beta}}
    \sigma^\frac{1}{\beta} \left(\frac{64e^{2\pert}}{1\wedge \decof}\right)^{1+\frac{1+\decopow+\beta\decopow}{\beta}}
    \left(\frac{3a + 2b + 3}{1 \wedge a}\right)^\frac{2(1+\beta)\decopow}{\alpha\beta}.
  $$
\end{corollary}

The above corollary that is proved in Section~\ref{sec:post-proof}, implies that the \lmc algorithm achieves $\eps$ accuracy
of the target in KL-divergence in 
${\tO\raisebox{.5ex}{\big(}d^{\frac{\alpha+(1+\beta)\decopow}{\alpha\beta}}
  \eps^{-\frac{1}{\beta}}\raisebox{.5ex}{\big)}}$ steps.
Whenever the tail growth of the potential is superlinear and 
behaves like $\norm{x}^\alpha$ for ${\alpha\in(1,2]}$,
Assumption~\ref{as:degen_conv}
holds for ${\decopow=2-\alpha}$; thus, Corollary~\ref{cor:main}
can be invoked for this choice of $\decopow$.
Therefore, in this case the convergence rate is simply
${\tO\raisebox{.5ex}{\big(}d^{\frac{2}{\alpha} (1+\frac{1}{\beta}) -1}
\eps^{-\frac{1}{\beta}}\raisebox{.5ex}{\big)}}$.
On the other hand, when the potential has linear tail growth (i.e. $f(x) \sim \norm{x}$),
by setting ${\tf=\left(1+\norm{x}^{1+\tau}\right)^{1/(1+\tau)}}$ where $\tau \in (0,1)$,
one can verify that Assumption~\ref{as:degen_conv}
holds for ${\decopow=2+\tau}$. By tuning this parameter with ${\tau=1/\log \left(6d\right)}$,
we obtain a convergence rate of 
${\tO\raisebox{.5ex}{\big(}d^{2+\frac{3}{\beta}}\eps^{-\frac{1}{\beta}}\raisebox{.5ex}{\big)}}$.
Putting this all together,
one can simply use $\decopow={2-\alpha\ind{\alpha \neq 1}}$,
which yields the advertised convergence rate
${\tO \raisebox{.5ex}{\big(}d^{\frac{1}{\beta} + (1+\frac{1}{\beta})\left(\frac{2}{\alpha}-\ind{\alpha \neq 1}\right)}
\eps^{-\frac{1}{\beta}}\raisebox{.5ex}{\big)}}$.
We emphasize that, in moderate dimensions where $d \ll \eps^{-1}$,
the rate only depends on the order of smoothness $\beta$,
whereas in high dimensions where $d=\mathcal{O}(\eps^{-1})$,
the tail growth rate $\alpha$ enters the rate through the dimension dependence.

The obtained rate is continuous in the domain $\alpha \in (1,2]$ and $\beta\in(0,1]$;
however, there is a discontinuous jump at $\alpha=1$ due to lack of convexity.
One can verify that $\decopow=1$ implies a tail growth of $\|x\|\log(1+\|x\|)$
which is superlinear, but in terms of Assumption~\ref{as:mild_dis}, we still
have $\alpha=1$. In this case, the tail growth cannot be explained with a polynomial in $\|x\|$;
therefore, $\decopow =1 \neq {2-\alpha\ind{\alpha \neq 1}}$ because of the
additional logarithmic factor.
One should also note that
$\alpha=1$ is the exception in the sense that introducing additional $\log$ factors
when $\alpha>1$ does not change $\decopow$, and ultimately the convergence rate stays the same.
See Section~\ref{sec:examples} for a more detailed discussion with examples demonstrating these observations.

\section{Further Implications}\label{sec:implications}
\subsection{Convex potentials}
The next proposition shows that convex potentials have at least linear growth.
\begin{proposition}\label{prop:conv-growth}
  For any differentiable convex potential ${f:\reals^d \to \reals}$,
  there exist constants $a, b >0$ such that
  $${f(x) \geq a \|x\| - b}, \ \ \text{ for all }\ \ x\in\reals^d.$$
\end{proposition}
The proof of the above Proposition is deferred to Section~\ref{sec:convex-proof},
and it follows from showing that any convex potential is
in fact coercive.

If a convex potential has tail growth rate $\alpha$, i.e. $f(x) \geq a \|x\|^\alpha - b$ for some $a,b>0$, then it is straightforward to show that $\alpha$-dissipativity holds using
the Taylor's theorem. Hence, one can argue that
for convex potentials,
the limiting factor for the applicability of our results is
the order of smoothness of the potential function $f$.
For example, consider ${f_1(x)=\sqrt{1+\norm{x}^2}}$ and ${f_2(x)=\norm{x}}$.
They are both convex with linear growth and satisfy Assumptions~\ref{as:degen_conv} and 
\ref{as:mild_dis} with the same $\decopow$ and $\alpha$.
While $f_1$ is smooth, $f_2$ does not satisfy Assumption~\ref{as:holder} for any $\beta$.
Therefore, our results, in particular Corollary~\ref{cor:main},
are applicable to $f_1$, but not to $f_2$.

\subsection{Non-convex potentials}
Assumptions~\ref{as:degen_conv}, \ref{as:mild_dis} and \ref{as:holder}
are robust to bounded perturbations. In other words, if these
assumptions are satisfied for a potential, then they also hold for
its finite perturbations.
The following lemma formalizes this statement.

\begin{lemma}\label{lemma:perturb}
  Let $f$ be a potential satisfying Assumptions~\ref{as:degen_conv},
  \ref{as:mild_dis} and \ref{as:holder} for ${\alpha>1}$.
  Then, for any bounded function $\phi$
  with $\beta$-\Holder continuous and bounded gradient, 
  ${f+\phi}$ can be normalized to a potential
  also satisfying Assumptions~\ref{as:degen_conv}, \ref{as:mild_dis}
  and \ref{as:holder}.
  
  Further, if we additionally have ${\sup_{x\in\reals^d}\norm{\grad \phi(x)} < a}$
  for the constant $a$ as in Assumption~\ref{as:mild_dis},
  the above result also holds for ${\alpha=1}$.
\end{lemma}

\begin{proof-of-lemma}[\ref{lemma:perturb}]
Let the bounds on $\phi$ and ${\grad \phi}$ be $\kappa_1$ and $\kappa_2$, respectively.
Since $\phi$ is bounded, ${\int e^{-f-\phi}}$ is finite,
therefore it can be normalized to be a probability distribution.
We ignore the normalizing constant since it does not change the gradient and the Hessian.

Assumption~\ref{as:degen_conv} holds for $f$, meaning that
there exists a $\tf$ such that ${\norm{f - \tf}_\infty < \pert}$,
and $\tf$ satisfies the conditions in Assumption~\ref{as:degen_conv}.
Since ${\vert \phi \vert \leq \kappa_1}$, we have
\begin{equation*}
  \norm{f + \phi - \tf}_\infty < \pert + \kappa_1,
\end{equation*}
which proves that Assumption~\ref{as:degen_conv} also holds for ${f+\phi}$.
For Assumption~\ref{as:mild_dis}, we write
\begin{equation*}
  \dotprod{\grad f(x)+ \grad \phi(x)}{x}
  \geq
  a \norm{x}^\alpha -b -\dotprod{\phi(x)}{x}
  \geq
  a \norm{x}^\alpha -b -\kappa_2 \norm{x}
  \geq
  a'\norm{x}^\alpha - b',
\end{equation*}
for some $a',b' >0$, where in the last step we used ${\alpha > 1}$.
When $\alpha=1$ this step is still correct because ${\kappa_2 < a}$.
Growth part remains true since the perturbation has bounded gradient
\begin{equation*}
  \norm{\grad f + \grad \phi} \leq \norm{\grad f} + \norm{\grad \phi}
  \leq (\kappa_2 + \growth)\left(1+\norm{x}^\zeta\right),
\end{equation*}
which implies that $g$ satisfies Assumption~\ref{as:mild_dis}.
Finally, for Assumption~\ref{as:holder}, since both ${\grad \phi}$ and ${\grad f}$ are $\beta$-\Holder continuous,
so is their summation for the same order of smoothness $\beta$.
\end{proof-of-lemma}

The previous lemma shows that Corollary~\ref{cor:main} is robust to
finite perturbations.
Moreover, investigating the proof reveals that the growth rate $\alpha$ and the order of smoothness $\beta$ do not 
change (along with $\pert$ and $\decopow$), which means that the convergence rate
of \lmc for the perturbed potential is the same as that for the original potential.

Adding a non-convex perturbation to a convex potential
might result in a non-convex potential; consider for example the potential $f(x)=\norm{x}^\alpha+5\cos{(\norm{x})}$.
The previous lemma implies that the convergence rate remains the same which means that
the main factor in the performance in high dimensions is the tail growth rate
rather than global curvature properties like (strong) convexity.
It is worth noting that the
order of smoothness (cf. Assumption~\ref{as:holder}) is a global assumption unlike the tail growth condition.
For example, $\norm{x}$ has a discontinuous
gradient at the origin; therefore, Assumption~\ref{as:holder} fails for any $\beta$.

\subsection{Other probability metrics}
In this section, we use Csisz\'ar-Kullback-Pinsker (CKP) inequalities \cite{bolley2005weighted}
as well as Talagrand's inequality to translate our result from 
KL-divergence to other measures of distance such as total variation (TV)
and $L_\alpha$-Wasserstein metrics.
The proofs are straightforward, and postponed to Section~\ref{sec:post-proof}.
In order to reach the same level of accuracy in different probability metrics, one
needs to adapt the step size accordingly.
This requires a different upper bound on the accuracy $\eps$ in each metric,
which we refer to an explicit statement in their respective results.
\begin{corollary}[Total variation distance]\label{cor:tv_conv}
  Instantiate the assumptions and notation of Theorem~\ref{thm:main}.
  For a sufficiently small $\epsilon$ satisfying $\eps \leq \sqrt{\psi/2}$ where
  $\psi$ is defined in \eqref{eq:tol-up},
  if \lmc~\eqref{eq:ULA} is initialized with ${x_0\sim \Gsn(x,\id)}$ for any $x \in \reals^d$,
  then, $${N = \tO \raisebox{.5ex}{\big(}d^{\frac{\alpha+\decopow+\beta\decopow}{\alpha\beta}}
    \eps^{-\frac{2}{\beta}}\raisebox{.5ex}{\big)}}$$
  steps are sufficient to obtain ${\tv{\rho_N}{\target} \leq \eps}$.
\end{corollary}
As before, one can simply use $\decopow={2-\alpha\ind{\alpha \neq 1}}$.
In the case of strongly convex and smooth potentials, i.e. $\alpha=2$ and $\beta=1$,
the above corollary recovers the convergence rate ${\tO\left(d\eps^{-2}\right)}$ in TV distance, which was established in \cite{durmus2017nonasymptotic}.

The next corollary establishes the convergence rate in terms of $L_\alpha$-Wasserstein distance.
\begin{corollary}[$L_\alpha$-Wasserstein distance] \label{cor:wasser_conv}
  Instantiate the assumptions and notation of Theorem~\ref{thm:main}.
  For a sufficiently small $\epsilon$ satisfying \eqref{eq:tol-up-w},
  if \lmc~\eqref{eq:ULA} is initialized with ${x_0\sim \Gsn(x,\id)}$ for any $x \in \reals^d$,
  then,
  $${N = \tO \raisebox{.5ex}{\big(}d^{\frac{3\alpha+\decopow+\beta\decopow}{\alpha\beta}}
    \eps^{-\frac{2\alpha}{\beta}} \raisebox{.5ex}{\big)}}$$
  steps are sufficient to obtain
  ${\wasserpp{\rho_N}{\target}{\alpha} \leq \eps}$.
\end{corollary}
For functions that have quadratic growth, the 
above corollary is not optimal, because our result relies on the CKP inequality (cf. Lemma~\ref{lemma:WCKP}) which does not
recover Talagrand's inequality when ${\alpha=2}$ \cite{bolley2005weighted}. Therefore,
the case $\alpha=2$ is handled separately, where Talagrand's inequality is available,
which ultimately yields the following improved rate.
\begin{corollary}[$L_2$-Wasserstein distance]\label{cor:wasser_2}
  Suppose $f$ is a smooth potential with quadratic growth i.e. a potential 
  satisfying Assumption~\ref{as:degen_conv} with ${\decopow=0}$,
  Assumption~\ref{as:mild_dis} with $\alpha=2$ and
  Assumption~\ref{as:holder} with $\beta=1$.
  For a sufficiently small $\epsilon$ satisfying~\eqref{eq:w2-tol-up},
  and for a step size satisfying~\eqref{eq:main_step},
  if the \lmc algorithm~\eqref{eq:ULA} is initialized with an error upper bounded by $\Delta_0$,
  i.e. ${\KL{\rho_0}{\target} \leq \Delta_0}$,
  then,
  $${N = \tO \raisebox{.2ex}{\big(}d\eps^{-2} \raisebox{.2ex}{\big)}}$$
  steps are sufficient to obtain
  ${\wasserpp{\rho_N}{\target}{2} \leq \eps}$.
\end{corollary}
We emphasize that ${\decopow=0}$ corresponds to potentials with tail growth rate ${\alpha=2}$.
Since in this case $\gamma=0$, there is no need to tune $s$ to a specific moment.
The above
corollary only covers smooth potentials, because Assumption~\ref{as:holder} implies
that the gradient of the potential has a tail growth rate upper bounded by $\beta$,
which in turn upper bounds the tail growth of $f$ with $\beta+1$.
Thus, the only feasible value for $\beta$ is $1$.

\section{Applications}\label{sec:examples}
In this section, we apply the results of Sections~\ref{sec:main-res} and \ref{sec:implications}
to various illustrative potential functions.
We begin with a few basic examples in order to demonstrate the effect
of tail growth and order of smoothness on the convergence of \lmc.
\subsection{Pedagogical examples}
\textbf{Example 1 ( Weakly smooth potential with subquadratic tails)}
Consider the potential function ${f(x) = \norm{x}^\alpha}$ for ${\alpha \in (1,2)}$.
This potential is not smooth with an unbounded Hessian near the origin, and
its tails are subquadratic which means the tails of the target ${\target \propto e^{-f}}$ are 
heavier than those of the Gaussian distribution.
In the following, we show that
$f(x)$ satisfies Assumptions~\ref{as:degen_conv}, \ref{as:mild_dis} and \ref{as:holder};
hence, the results of Section~\ref{sec:main-res} can provide a convergence rate for
sampling from this particular potential using the \lmc algorithm. For Assumption~\ref{as:degen_conv}, consider the following function ${\tf (x) = \left(1+\norm{x}^2\right)^{\alpha/2}}$.
Since
${\Hess \tf (x) \succeq \frac{\alpha(\alpha-1)}{2}(1+\norm{x}^2/4)^{\alpha/2 - 1} \id}$
and ${\|f-\tf\|_\infty \leq 1}$, the function $f$ satisfies Assumption~\ref{as:degen_conv} with ${\decopow=2-\alpha}$.
Simple calculation shows that \eqref{eq:mild-dis} in Assumption~\ref{as:mild_dis} is also satisfied.
For Assumption~\ref{as:holder}, we have
${\grad f(x) = \alpha x \norm{x}^{\alpha-2}}$, which
is ${\beta}$-\Holder with ${\beta=\alpha-1}$ (cf. Lemma~\ref{lemma:holder_potential}),
which also
implies that the growth condition in Assumption~\ref{as:mild_dis} holds with ${\zeta=\alpha-1}$.
Therefore, Corollary~\ref{cor:main} implies that we can reach $\eps$ accuracy in KL-divergence
after taking
${\tO \raisebox{.5ex}{\big(} d^{\frac{3-\alpha}{\alpha-1}}\eps^{-\frac{1}{\alpha-1}} \raisebox{.5ex}{\big)}}$ steps.

It is important to highlight the impact of
the order of smoothness on the convergence rate. More specifically, if one has the smooth potential
${f(x)=(1+\norm{x}^2)^{\alpha/2}}$ which has the same tail growth as $\|x\|^\alpha$,
then the convergence rate becomes 
${\tO \raisebox{.5ex}{\big(}d^{\frac{4-\alpha}{\alpha}}\eps^{-1} \raisebox{.5ex}{\big)}}$.

A finite and bounded perturbation should not change
the above convergence rate. Indeed, consider the function ${\phi(x)=\cos{(\norm{x})}}$, 
which is bounded with bounded first derivative.
Its gradient is given by
$\grad \phi(x) = -\frac{x}{\norm{x}}\sin{(\norm{x})}$
which is Lipschitz continuous; hence,
Lemma~\ref{lemma:bounded_diff_holder} implies that it is also $\beta$-\Holder continuous.
By Lemma~\ref{lemma:perturb}, the rate obtained from Corollary~\ref{cor:main}
is applicable to ${g(x) = \norm{x}^\alpha + 10\cos{(\norm{x})} + \pert}$,
and the convergence rate
${\tO \raisebox{.5ex}{\big(} d^{\frac{3-\alpha}{\alpha-1}}\eps^{-\frac{1}{\alpha-1}} \raisebox{.5ex}{\big)}}$ still holds.
Figure~\ref{fig:toy_eg} demonstrates the behavior of this potential and its gradient when ${\alpha=1.5}$,
in 1 dimension.

\begin{figure}[t]
  \begin{subfigure}{.5\textwidth}
    \centering
    \includegraphics[width=.99\linewidth]{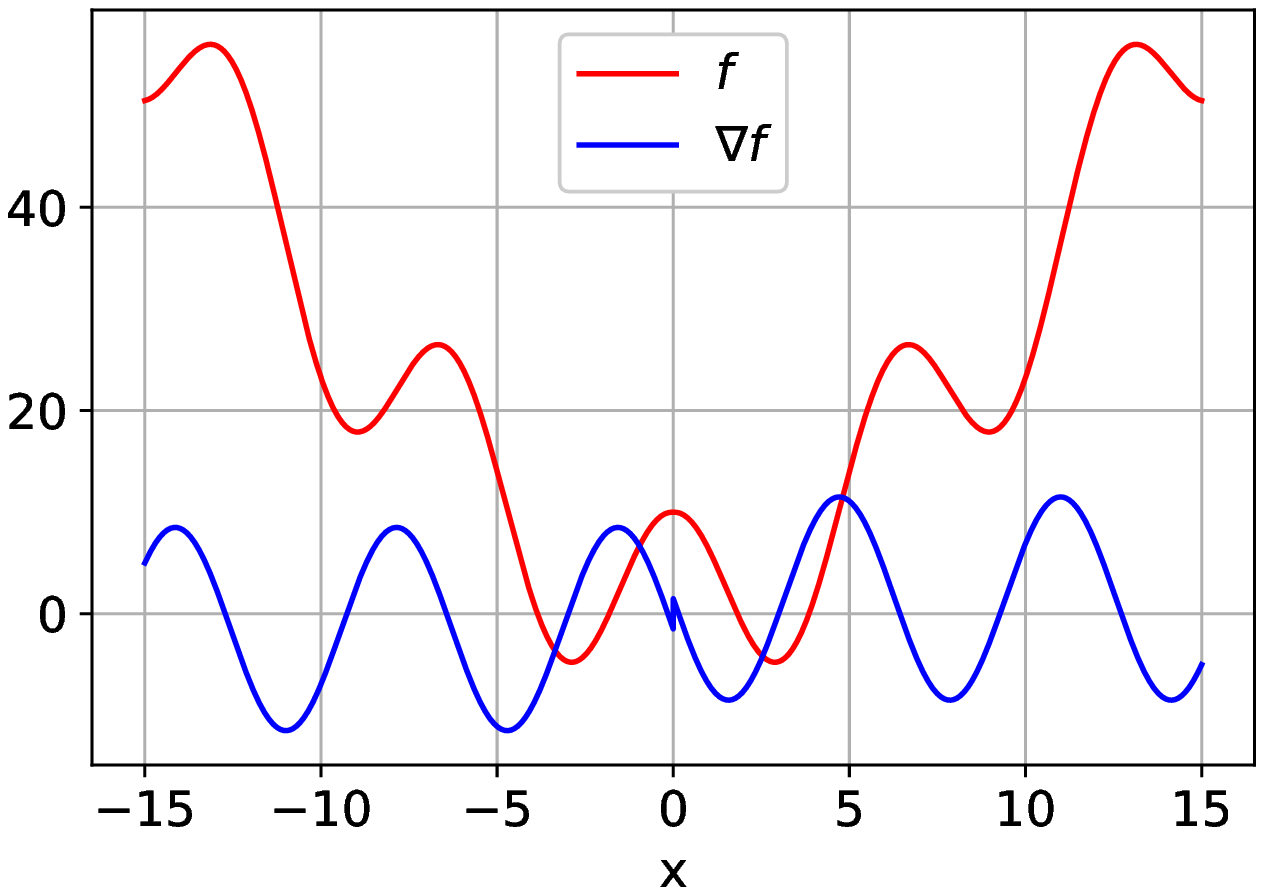}
    \caption{}
    \label{fig:toy_eg}
  \end{subfigure}%
  \begin{subfigure}{.5\textwidth}
    \centering
    \includegraphics[width=.99\linewidth]{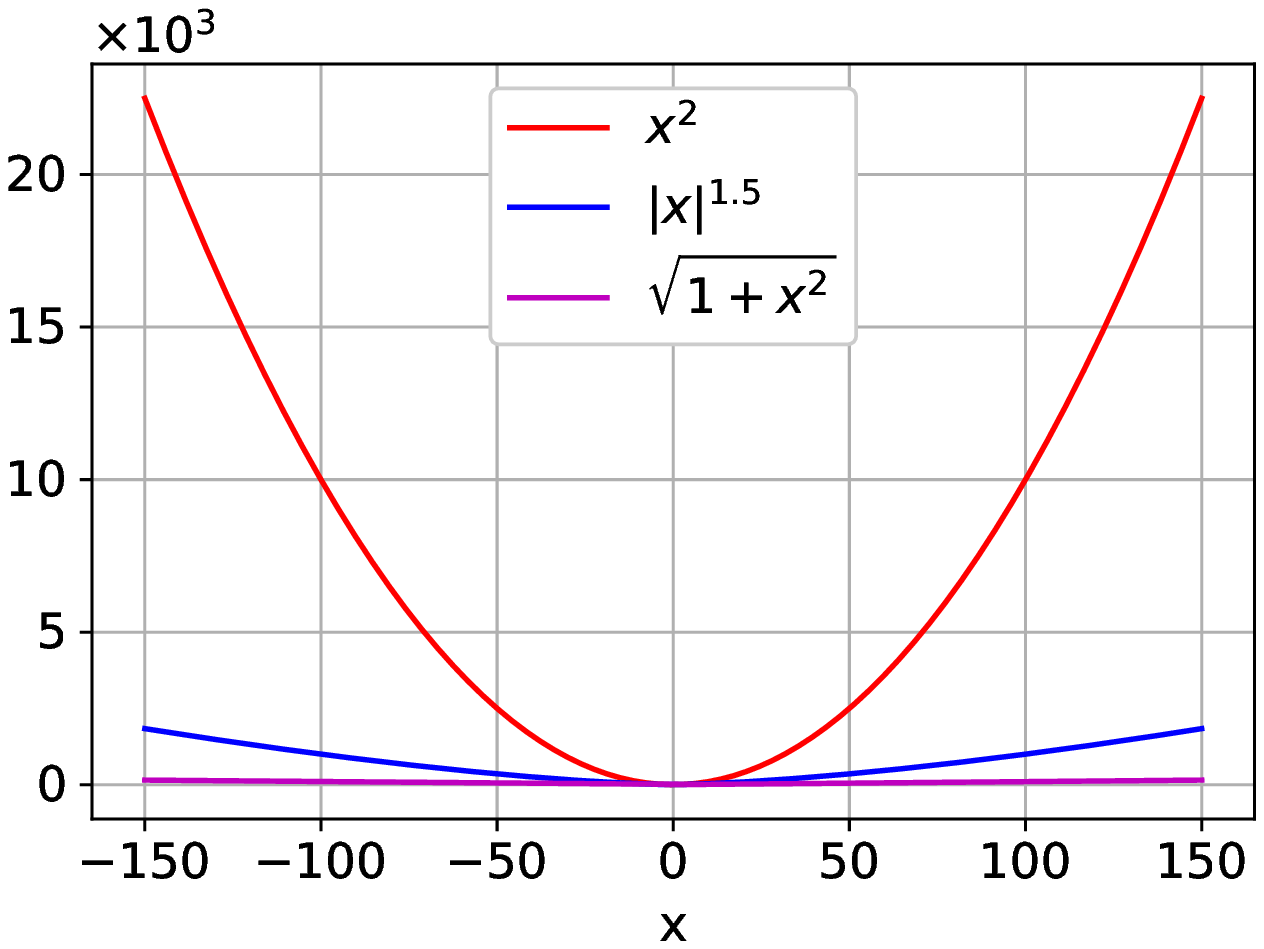}
    \caption{}
    \label{fig:almost_lin}
  \end{subfigure}
  \caption{(a) Potential ${f(x) = \abs{x}^{1.5} + 10\cos{(x)}}$ and its gradient.
  The normalizing constant is ignored.
  (b) Comparison of Growth needed for LSI and this work's 
  setting, both settings can work with perturbations which are ignored
  here for clarity.
  }
  \label{fig:fig}
  \end{figure}

\noindent\textbf{Example 2 (Smooth potential with linear tails):}
Since $\norm{x}$ has discontinuous gradient at the origin,
we consider ${f(x)=\sqrt{1+\norm{x}^2}}$
as an example of a smooth potential with linear growth.
The Hessian of this potential can be computed as
$$\nabla^2f(x)={{\left(1 + \norm{x}^2\right)}^{-3/2}\left[\left(1+\norm{x}^2\right)\id - xx^\top\right]}.$$ 
If we let ${\tf=f}$ in Assumption~\ref{as:degen_conv},
one can easily verify that the assumption holds for ${\decopow=3}$.
Further, computing its gradient reveals that
Assumption~\ref{as:mild_dis} is satisfied with 
${\zeta=0}$ and ${\alpha=1}$.
Finally, since its Hessian is bounded, Assumption~\ref{as:holder} is satisfied with ${\beta=1}$.
Plugging these parameters in Corollary~\ref{cor:main},
we obtain the convergence rate ${\tO \left(d^7\eps^{-1} \right)}$ in KL-divergence.

The dimension dependency in the previous convergence rate can be improved
by changing $\tf$ to a function that is different than $f$.
Observe that the difference between $\sqrt{1+\norm{x}^2}$ and
${\left(1 + \norm{x}^{1+\tau}\right)^{1/(1+\tau)}}$ is bounded
for any ${\tau \in (0,1)}$.
Thus, if we set ${\tf(x) = {\left(1 + \norm{x}^{1+\tau}\right)^{1/(1+\tau)}}}$, Assumption~\ref{as:degen_conv}
is satisfied with ${\decopow = 2+\tau}$ and ${\decof = \*O(\tau)}$. Setting 
${\tau = \mathcal{O}({\log\left(6d \right)}^{-1}})$
and invoking Corollary~\ref{cor:main} implies a convergence rate of ${\tO(d^5\eps^{-1})}$
in KL-divergence,
for a potential like ${f(x)=\sqrt{1+\norm{x}^2}+0.5\cos(\norm{x})}$.
We note that in this case, the norm of the perturbation needs to be strictly smaller than $1$,
otherwise Assumption~\ref{as:mild_dis} is no longer satisfied. 

\noindent\textbf{Example 3 (Smooth potential with linear tails up to log factor):}
We consider the potential function ${f(x) = \norm{x} \log{\left(1+\norm{x}^2\right)}}$
which is smooth and has a superlinear tail growth.
By choosing ${\tf(x)=\log{\left(1+\norm{x}^2 \right)}\sqrt{1+\norm{x}^2}}$,
one can verify Assumption~\ref{as:degen_conv} for ${\decopow=1}$.
Simply calculating the gradient of $f$,
one can also verify $\alpha$-dissipativity
\eqref{eq:mild-dis} for $\alpha=1$.
Moreover, ${\grad f}$ has logarithmic growth, which means that for any ${\zeta>0}$,
gradient growth condition \eqref{eq:grad-growth} holds; thus, Assumption~\ref{as:mild_dis} is satisfied.
Since $\grad f$ is Lipschitz, Assumption~\ref{as:holder} is satisfied with $\beta=1$.
Finally, invoking Corollary~\ref{cor:main},
one obtains a convergence rate of ${\tO \left(d^3\eps^{-1} \right)}$ in KL-divergence.

Comparing the above rate to that in the previous example, we observe that the extra log factor
improves the dimension dependence from 
${\tO \left(d^5 \right)}$ to 
${\tO \left(d^3 \right)}$, 
while the $\eps$-dependence remains the same. Note that in here $\alpha=1$,
same as in the previous example;
yet, the convergence in this case is faster than the advertised rate of
${\tO \raisebox{.5ex}{\big(}d^{\frac{1}{\beta} + (1+\frac{1}{\beta})\left(\frac{2}{\alpha}-\ind{\alpha \neq 1}\right)}
\eps^{-\frac{1}{\beta}}\raisebox{.5ex}{\big)}}$ because the potential is not growing like
$\norm{x}^\alpha$ for any $\alpha$ due to the additional $\log$ factor. We note that
including an extra
log factor will only improve 
the convergence
when $\alpha=1$; when $\alpha>1$, the convergence rate simply stays the same.
One intuitive reason for this phenomenon is that for $\alpha>1$
the tails are already convex but when $\alpha=1$, the tails are linear, in which case the Hessian
is almost equal
to zero, and the extra $\log$ injects the much needed convexity to the tails.

Figure~\ref{fig:almost_lin} compares the growth rate of these examples with that of
the quadratic function,
in which case the standard LSI \eqref{eq:lsi} holds.
Perturbations are ignored to make the figure clear.
\subsection{Bayesian regression}
In this section, the fixed problem parameters such as $\growth$ and $\smooth$ depend on the data, 
and the rates are obtained assuming these constants are $\*O(1)$. Depending on the data scaling,
these parameters may depend
on the dimension as well as the number of samples $n$,
in which case the rates can be still obtained by
using the explicit formulas presented in Theorem~\ref{thm:main} and Corollary~\ref{cor:main}.

\noindent\textbf{Example~4 (Bridge regression):}
Our analysis shows that the LMC algorithm can handle potentials that are weakly smooth,
which comes up frequently in Bayesian statistics. For example,
denoting the matrix of covariates with ${V=\{v_i\}_{i=1}^{n} \in \reals^{n\times d}}$, 
and the response vector with ${Y=\{y_i\}_{i=1}^{n} \in \reals^n}$,
in Bridge linear regression \cite{fu1998penalized, frank1993statistical},
one assumes the following linear model
$${Y=Vx + \varepsilon}\quad \text{ where }\quad\varepsilon \sim\Gsn(0,I_n),$$
and a prior proportional to ${\exp(\sum_{i=1}^{d} \abs{x_i}^q)}$.
Therefore, sampling from the resulting posterior is equivalent to sampling from the following potential
function
\eqn{
  f(x) = \norm{Y-Vx}^2 + \sum_{i=1}^{d} \abs{x_i}^q.
}
Assume that we have ${V^{\top}V \succ 0}$ and $q \in (1,2)$.
Then, the above potential has quadratic growth which, in our framework, translates
to setting ${\decopow=0}$ in Assumption~\ref{as:degen_conv},
and setting ${\alpha=2}$ and ${\zeta=1}$ in Assumption~\ref{as:mild_dis}.
This potential lacks smoothness and in the close neighborhood of the origin,
Assumption~\ref{as:holder} holds with $\beta=q-1$. On the other hand,
$\grad f$ has linear growth and when $\|x-y\|\geq 1$ Assumption~\ref{as:holder} holds with $\beta=1$.
In other words, Assumptions~\ref{as:holder} does not hold for a single $\beta$ for all $x$.
Initially, it might seem that
our results are non-applicable to this potential but by adapting Assumption~\ref{as:holder}
to this setting as
$$
\norm{\grad f(x) - \grad f(y)} \leq \smooth 
\left(\norm{x-y}^{\beta_1} + \norm{x-y}^{\beta_2} \right)
\ \ \text{ for all }\ \ x,y\in\reals^d,
$$
where $\beta_1<\beta_2$, and
some minor changes to Lemma~\ref{lemma:kl_deriv_bound}, our results can be applied to this potential.
In this example, we need to set $\beta_1=q-1$, $\beta_2=1$, and $\zeta \leq \beta_2$ which
yields the convergence rate
${\tO \raisebox{.5ex}{\big(} d^{\frac{1}{q-1}} \eps^{-\frac{1}{q-1}}\raisebox{.5ex}{\big)}}$ in KL-divergence.

This form of Assumption~\ref{as:holder} enables us to deal with composite potentials that are sum of two
components each of which is \Holder continuous but with a different exponent.

\noindent\textbf{Example~5 (Bayesian logistic regression):} In Bayesian logistic regression,
we are given $n$ samples
${V=\{v_i\}_{i=1}^{n} \in \reals^{n\times d}}$, ${Y=\{y_i\}_{i=1}^{n} \in \{0,1\}^n}$
according to the logistic regression model ${\P(y_i=1 \vert v_i)=1/(1+ \exp(-\inner{x, v_i}))}$
for some parameter $x\in\reals^d$, and
we would like to make inference about 
the parameter $x$ under some prior distribution $p(x)$.
In this framework, it is common to use \lmc to generate samples from the posterior
distribution ${p(x\vert V,Y)}$
with the following potential function
\eq{
  f(x) = -\log{p(x)} - \inner{Y, V x} + \sum_{i=1}^{n} \log(1+\exp(\inner{x,v_i})).
}
In practice, the prior distribution can be arbitrary whereas most theoretical results
require the prior to be the Gaussian distribution in order to ensure that
the posterior is smooth and has quadratic growth \cite{dalalyan2017theoretical,li2019stochastic}.
The framework in this paper allows for priors that 
have heavier tails than a Gaussian and/or
have potentials that do not have Lipschitz gradients.
For example, consider a pseudo-Huber prior
${p(x) \propto \exp \big(1-\sqrt{1+\norm{x}^2}\big)}$
\cite{gorham2019measuring, hartley2003multiple},
which results in a similar setting as in
Example~2 in the sense that
simply choosing $\tf$ as the potential
yields an ${\tO \left(d^7 \eps^{-1}\right)}$ convergence rate in KL-divergence.
As before, the dimension dependency can be improved by a more careful choice of $\tf$, i.e.,
\eqn{
  \tf(x) = {\left(1 + \norm{x}^{1+\tau}\right)^{1/(1+\tau)}} - 
  \inner{Y, V x} + 
  \sum_{i=1}^{n} \log(1+\exp(\inner{x,v_i})),
}
which yields a convergence rate of ${\tO \left(d^5 \eps^{-1} \right)}$.
In here the parameters are set as follows $\theta=2+\tau$, $\alpha=1$, $\zeta=0$ and $\beta=1$.

Next, consider the prior 
${p(x) \propto \exp \big(\sum_{i=1}^{d} \abs{x_i}^q \big)}$ for $q\in (1,2)$
which is similar to the Bridge linear regression setting.
The resulting potential is not smooth in this case,
and the potential lacks
quadratic growth. Our analysis shows that \lmc reaches $\eps$-accuracy in KL-divergence
after
${\tO \raisebox{.5ex}{\big(} d^{\frac{3-q}{q-1}} \eps^{-\frac{1}{q-1}} \raisebox{.5ex}{\big)}}$
steps.
Here, the parameters are given as $\theta=2-q$, $\alpha=q$, $\zeta=q-1$, and $\beta=q-1$.

\noindent\textbf{Example~6 (Huberized Regression with log-linear prior):}
Let $V$ and $Y$ denote the matrix of covariates and the response vector respectively, as
in Example~4. We consider the Bayesian analog of Huberized loss \cite{park2008bayesian}
with a log-linear prior, which corresponds to sampling from the posterior
$p(x)\propto e^{-f(x)}$ where
\eqn{
  f(x)= \sum_{i=1}^{n} \sqrt{1+ \left(y_i - \inner{x, v_i}\right)^2 } + \sqrt{1+\norm{x}^2}.
}
By considering 
\eqn{
  \tf(x)= \sum_{i=1}^{n} \sqrt{1+\left(y_i - \inner{x, v_i}\right)^2} + \left(1+\norm{x}^{1+\tau}\right)^\frac{1}{1+\tau},
}
for $\tau \in (0,1)$ and applying an argument similar to the one in Example~2,
one can obtain the convergence
rate ${\tO \left(d^5\eps^{-1}\right)}$ for the \lmc algorithm in KL-divergence.
In here, the parameters are given as $\theta=2+\tau$, $\alpha=1$, $\zeta=0$ and $\beta=1$.

\section{Proof of Modified Log-Sobolev Inequality}\label{sec:proof-mlsi}
We start with a lemma that allows us to construct a finite perturbation of the potential
function 
that has polynomially decaying Hessian which is unbounded at 0.
This will allow us to optimize the final bound.
\begin{lemma}\label{lem:inf_at_origin}
  Suppose Assumption~\ref{as:degen_conv} holds.
  Then for sufficiently small ${\varepsilon>0}$, there exist a function ${\tf_\varepsilon : \reals^d \to \reals}$ such that
  \begin{equation*}
    \big\|f-\tf_\varepsilon\big\|_\infty \leq \pert + \varepsilon/2,
  \end{equation*}
  where $\tf_\varepsilon$ satisfies, 
  \eq{
    \grad^2 \tf_\varepsilon (x) \succeq  \decon (\|x\|) \id,
  }
  where  ${\decon:\reals_+ \to \reals_+}$ is a monotonically decreasing and onto function satisfying
  $${m(r) \geq \frac{\decof - \alpha_\decopow \varepsilon}{(1+r^2/4)^{\decopow/2}}},$$ where 
  ${\alpha_\decopow  <\infty}$ is a constant depending only on $\decopow$.
\end{lemma}
\begin{proof-of-lemma}[\ref{lem:inf_at_origin}]
  Let $\tf_\varepsilon(x) = \tf(x) + \varepsilon\norm{x}^{3/2}e^{-\norm{x}^2}$, and notice that $f$ satisfies
  $${\norm{f - \tf_\varepsilon}_\infty \leq \pert + \varepsilon/2}.$$
  For its Hessian, we write
  \eqn{
    \Hess \tf_\varepsilon(x) = \Hess \tf(x) + \varepsilon e^{-\norm{x}^2}
    \left\{
      \left(\frac{3}{2}\norm{x}^{-1/2} -2\norm{x}^{3/2}\right) \id
      - \left( 6\norm{x}^{-1/2}-4\norm{x}^{3/2}  +\frac{3}{4}\norm{x}^{-5/2} \right)x x^\top
    \right\}
  }
  and choosing ${\alpha_\decopow = 8\sup_{r \geq 0} r^{1.5}(1+r^2/4)^{\decopow/2}e^{-r^2}}$, we observe that
  $${\Hess \tf_\varepsilon(x) \succeq \frac{\decof - \alpha_\decopow \varepsilon}{(1+\norm{x}^2/4)^{\decopow/2}}}\id.$$
  Also,
  ${\Hess \tf_\varepsilon(x) > \varepsilon\norm{x}^{-1/2}/2}$ when ${x\leq 0.1}$. Now by selecting
  $${\decon(r) = (\decof - \alpha_\decopow \varepsilon) (1+r^2/4)^{-\decopow/2} \vee \frac{\varepsilon}{2}r^{-1/2}\ind{r \leq 0.1}},$$
  the lemma follows. Note that ${m:\reals_+ \to \reals}$ is both monotone and onto for ${\varepsilon< {1}/{(\alpha_\decopow + 2)}}$.
\end{proof-of-lemma}

\begin{proof-of-theorem}[\ref{thm:MLSI}]
    We follow a  similar construction developed in \cite{toscani2000trend},
    and define the functions $h$ and $\tilh_\varepsilon$ as
\begin{align}\label{eq:weak-lsi-pert}
  &h(x) = f(x) +  \decon(2r)\big(\norm{x} - r\big)^2 \ind{\norm{x} \geq r} +C_r \ \text{ and }\\
  &\tilh_\varepsilon(x) = \tf_\varepsilon (x) + \decon(2r)\big(\norm{x} - r\big)^2 \ind{\norm{x} \geq r} +C_r,
\end{align}
where $C_r$ is the normalizing constant for the unnormalized potential $h$ satisfying
\begin{equation}\label{eq:norm-constant}
  e^{C_r} = \int_{\norm{x} < r} e^{-f(x)}dx
  +  \int_{\norm{x} \geq r} e^{-f(x)}e^{-\decon(2r)(\norm{x} - r)^2}dx.
\end{equation}
Using Assumption~\ref{as:degen_conv} and Lemma~\ref{lem:inf_at_origin},
we notice that $h$ and $\tilh_\varepsilon$ satisfy ${\big|h(x) - \tilh_\varepsilon(x)\big| \leq \pert + \varepsilon/2}$,
and also
the growth of the function $\tilh_\varepsilon$ can be characterized in the following three regions.
\begin{itemize}
\item For ${\norm{x} < r}$, we have
\begin{align*}
  \grad^2 \tilh_\varepsilon(x)  = \grad^2 \tf_\varepsilon(x) \succeq& \decon (\norm{x}) \id\\
  \succeq&  \decon(2r) \id,
\end{align*}
where in the last step we used the monotonicity of $\decon$.
\item For ${r \leq\norm{x} < 2r}$, we have
\begin{equation*}
\begin{split}
  \grad^2 \tilh_\varepsilon(x)  =& \grad^2 \tf_\varepsilon(x) + \decon(2r) \left\{2\id - \frac{2r}{\norm{x}}\id +2r\frac{xx^\top}{\norm{x}^3}\right\}\\
  \succeq & \decon (\norm{x})\id + {\decon(2r)} \left\{2\id - 2\id + 0\right\}\\
  \succeq &  \decon(2r) \id,
\end{split}
\end{equation*}
where again the last step follows from the monotonicity of $\decon$.
\item For ${2r \leq \norm{x}}$, we have
\begin{equation*}
\begin{split}
  \grad^2 \tilh_\varepsilon(x)  =& \grad^2 \tf_\varepsilon(x) + { \decon(2r)} \left\{2\id - \frac{2r}{\norm{x}}\id +2r\frac{xx^\top}{\norm{x}^3}\right\}\\
  \succeq& 0 + { \decon(2r)} \left\{2\id - \id + 0\right\}\\
  \succeq& { \decon(2r)}\id.
\end{split}
\end{equation*}
\end{itemize}
In all three cases, we obtain that the function $\tilh_\varepsilon$ has a positive definite Hessian
which is lower bounded by ${\decon(2r)}$ which implies, by the Bakry-Emery's LSI result
on strongly convex potentials~\cite{bakry1985LSI} that the distribution $e^{-\tilh_\varepsilon}$
satisfies \eqref{eq:lsi}. Combining this with the Holley-Stroock perturbation
lemma~\cite{holley1987LSI}, we obtain
\begin{equation}\label{eq:lsi-for-perturbed}
 \forall \rho,\ \  \KL{\rho}{e^{-h}} \leq \frac{e^{2\pert+\varepsilon}}{2 \decon(2r)} \I{\rho}{e^{-h}}.
\end{equation}

We will convert the above inequality on the perturbed potential $h$ to an inequality
on the potential function $f$.
Using the definition in \eqref{eq:weak-lsi-pert},
we can obtain an upper bound on the relative entropy for all ${r>0}$,
\begin{align}\label{eq:rel-ent-bound}
  \KL{\rho}{e^{-f}}=& \KL{\rho}{e^{-h}} +\int \rho(x) (f(x)-h(x))dx\\
  =& \KL{\rho}{e^{-h}} -{\decon(2r)}\int_{\norm{x} \geq r} \rho(x) (\|x\|-r)^2dx - C_r.
\end{align}

For the normalizing constant $C_r$
explicitly given in \eqref{eq:norm-constant}, one can obtain a lower bound using the Jensen's inequality,
\begin{align}\label{eq:norm-constant-bound}
  C_r =& \log \int e^{-f(x)}e^{-{\decon(2r)}(\norm{x} - r)^2\ind{\|x\|\geq r}}dx\\
  \geq &-{\decon(2r)} \int_{\|x\|\geq r} e^{-f(x)} (\norm{x} - r)^2dx.
\end{align}
Combining this with \eqref{eq:rel-ent-bound}, we obtain
\eq{
  \KL{\rho}{e^{-f}} \leq& \KL{\rho}{e^{-h}}
  +{\decon(2r)}\int_{\norm{x} \geq r}\big(e^{-f(x)} -\rho(x)\big) (\|x\|-r)^2dx\\
  \leq& \KL{\rho}{e^{-h}}
  +{\decon(2r)}\int_{\norm{x} \geq r}e^{-f(x)} \|x\|^2dx\\
  \leq& \KL{\rho}{e^{-h}}
  +{\decon(2r)}\frac{\M_s(\target)}{(1+r^2)^{s/2-1}}\label{eq:perturbed-entro-bound}
}
where the second step follows since ${\rho \geq 0}$, and ${\|x\|^2 \geq (\|x\| - r)^2}$
in the domain of integration, and the last step follows from Lemma~\ref{lem:tail-moment-bound} below.

\begin{lemma}\label{lem:tail-moment-bound}
  For a given distribution $\rho$ and for a constant ${r >0}$, we have
  \eq{
    \int_{\|x\| \geq r} \rho(x) \|x\|^2dx \leq \frac{\M_s(\rho)}{(1+r^2)^{s/2-1}}.
  }
\end{lemma}
\begin{proof-of-lemma}[\ref{lem:tail-moment-bound}]
  For positive constants ${p,q,s>0}$ satisfying ${1/p+1/q = 1}$, we apply the H\"older's inequality and get
  \eq{
    \int_{\|x\| \geq r} \rho(x) \|x\|^2dx = &\int \rho(x) \|x\|^2 \ind{\|x\| \geq r}dx\\
    \leq &\left( \int \rho(x) \|x\|^{2p}dx\right)^{{1}/{p}}\P\left(\left(1+\|x\|^2\right)^{s/2} \geq {(1+r^2)^{s/2}}\right)^{1/q}\\
    \leq& \frac{\M_{2p}(\rho)^{{1}/{p}}\M_{s}(\rho)^{1/q}}{ (1+r^2)^{s/2q}},
  }
  where the last step follows from Markov's inequality. Final result follows by choosing ${p=s/2}$.
\end{proof-of-lemma}

Similarly for the Fisher information, we write
\begin{align}\label{eq:info-unpack}
  \I{\rho}{e^{-h}} \leq 2 \I{\rho}{e^{-f}} + 2\int \rho(x) \norm{\grad h(x) - \grad f(x) }^2dx.
\end{align}
For the second term on the right hand side, we write
\begin{align}\label{eq:markov-bound}
  \int \rho(x) \norm{\grad h(x) - \grad f(x)}^2dx =& 4\decon(2r)^2\int\rho(x) (\norm{x} - r)^2\ind{\norm{x} \geq r}dx\\
  \leq&  4\decon(2r)^2\int\rho(x) \norm{x}^2\ind{\norm{x} \geq r}dx\\
    \leq&  \frac{4 \decon(2r)^2}{(1+r^2)^{s/2-1}}\M_s(\rho),
\end{align}
where in the last step we applied Lemma~\ref{lem:tail-moment-bound}. Plugging this back in \eqref{eq:info-unpack},
we get
\eq{\label{eq:perturb-info-bound}
  \I{\rho}{e^{-h}} \leq 2 \I{\rho}{e^{-f}} +
  \frac{8 \decon(2r)^2}{(1+r^2)^{s/2-1}}\M_s(\rho).
}
Combining the inequalities \eqref{eq:lsi-for-perturbed}, \eqref{eq:perturbed-entro-bound}, and \eqref{eq:perturb-info-bound},
we obtain
\eq{
  \forall \rho,\ \  \KL{\rho}{e^{-f}} \leq&
  \KL{\rho}{e^{-h}}
  +\frac{\decon(2r)\M_s(\target)}{(1+ r^2)^{s/2-1}}\\
  \leq& \frac{e^{2\pert + \varepsilon}}{ 2\decon(2r)} \I{\rho}{e^{-h}}
  +\frac{\decon(2r)\M_s(\target)}{(1+ r^2)^{s/2-1}}\\
  \leq & \frac{e^{2\pert+ \varepsilon}}{ \decon(2r)}  \I{\rho}{e^{-f}} +
  \frac{m(2r)}{(1+r^2)^{s/2-1}} \left(4e^{2\pert + \varepsilon}\M_s(\rho) + \M_s(\target) \right)\\
  \leq & \frac{e^{2\pert+ \varepsilon}}{ \decon(2r)}  \I{\rho}{e^{-f}} +
  4e^{2\pert + \varepsilon}\frac{m(2r)}{(1+r^2)^{s/2-1}}  \M_s(\rho+\target).
}
Finally, using the Lemma~\ref{lemma:power_rec_max} and optimizing over ${\decon(2r)}$, we get
\eq{
  \forall \rho,\ \  \KL{\rho}{e^{-f}} \leq& 
  \lambda_\varepsilon \I{\rho}{e^{-f}}^{\frac{s-2+\decopow}{s-2+2\decopow}} \M_s(\rho + \target)^{\frac{\decopow}{s-2+2\decopow}},
}
where
\eq{
  \lambda_\varepsilon = \frac{4 e^{2\pert+\varepsilon}}{(\decof-\alpha_\decopow \varepsilon)^{\frac{s-2}{s-2+2\decopow}}},
}
for all sufficiently small ${\varepsilon>0}$.
Taking the limit of ${\varepsilon \downarrow 0}$ concludes the proof.
\end{proof-of-theorem}

\section{Moment Bounds on the LMC Iterates}\label{sec:proof-moment}
\begin{proof-of-proposition}[\ref{prop:disc_mom_bound}]
Similar to the continuous-time case, it suffices to prove
\eqn{
  \M_s(\rho_k) \leq\M_s(\rho_0) + C_s k \eta.
}

\noindent\textbf{Part 1.} We prove a linear bound on the second moment of $\xtild_{k, t}$ conditioned on  $x_k$. Consider the distribution ${\rho(\xtild_{k,t} \rvert x_k)}$ which is the distribution of $\xtild_{k,t}$ given $x_k$.
\begin{equation*}
\begin{split}
  \EE{\norm{\xtild_{k,t}}^2 \rvert x_k} &= 
  \norm{x_k}^2 - 2t\dotprod{\grad f(x_k)}{x_k} + 
  t^2 \norm{\grad{f(x_k)}}^2 + 2dt\\
  &\stackrel{1}{\leq}  \norm{x_k}^2 - 2t(a\norm{x_k}^\alpha - b) + 
  2 t^2 \growth^2(1+\norm{x_k}^{2\zeta}) + 2dt\\
  & = \norm{x_k}^2 + 2t\left(
      -a\left(1+\norm{x_k}^\alpha\right) + \eta \growth^2 \norm{x_k}^{2\zeta}
      + a + b + d+ \eta \growth^2
  \right) \\
  &\stackrel{2}{\leq} \norm{x_k}^2 + 
  2 \left(
      a+ b + d + \eta \growth^2
  \right)t\\
  &\leq
  \norm{x_k}^2 + C_2 t,
\end{split}
\end{equation*}
for any $C_2$ satisfying
\begin{equation}\label{eq:C_2}
    C_2 \geq 3a + 2b+ 2d.
\end{equation}
Step 1 is obtained using Assumptions~\ref{as:mild_dis},
and step 2 is because ${4\eta \growth^2 \leq a}$.
Adding one to both sides, we get the following equation
\begin{equation}
    \M_2(\rhotild_{k,t}\rvert x_k) \leq g_2(x_k) + C_2 t,
\end{equation}
where ${g_s(x) = \left(1 + x^2 \right)^{s/2}}$ and ${\M_s(\rhotild_{k,t}\rvert x_k)}$ denotes the $s$-moment of $\xtild_{k,t}$ conditioned on $x_k$.

\noindent\textbf{Part 2.} We upper bound a term which will become useful in the proof of the induction step.
(In below, $Z$ denotes a standard Gaussian vector that is independent of $x_k$.)
\begin{equation*}
\begin{split}
  \EE{-\dotprod{\grad f(x_k)}{Z} g_2(\xtild_{k, t})\rvert x_k}
  & \stackrel{1}{=} 2\sqrt{2t} \EE{-\dotprod{\grad f(x_k)}{Z}\dotprod{x_k}{Z} +
  t\dotprod{\grad f(x_k)}{Z} \dotprod{\grad f(x_k)}{Z} \rvert x_k}\\
  & = 2\sqrt{2t}\left(
      -\dotprod{\grad f(x_k)}{x_k}+t\norm{\grad{f(x_k)}}^2
  \right)\\
  & \leq
  2\sqrt{2\eta}\left(
      -a\norm{x_k}^\alpha + b + 2 \eta \growth^2(1+\norm{x_k}^{2\zeta})
  \right)\\
  & \leq 2\sqrt{2\eta}\left(
      -a\left(\norm{x_k}^\alpha + 1\right) + 2 \eta \growth^2 \norm{x_k}^{2\zeta}
      + a + b + 2 \eta \growth^2
  \right)\\
  & \stackrel{2}{\leq} 2\sqrt{2\eta}\left(
      a + b + 2 \eta \growth^2
  \right)
  \leq N_2,
\end{split}
\end{equation*}
where 
\begin{equation}\label{eq:N-2}
    N_2 \defeq 2\sqrt{2\eta}\left(1.5a + b\right).
\end{equation}
Step 1 follows from odd Gaussian moments being zero, and step 2 uses ${4\eta\growth^2 < a}$.
Note that $Z$ is independent of $x_k$, with zero mean.

\noindent\textbf{Part 3.} Now we use induction to prove the linear bound for even moments of the conditional distribution. The base case (${s=2}$) is already proved.
Hence we can assume ${s\geq 4}$ which implies ${(s-4)}$ is an even non-negative integer. For the proof to work, we need to strengthen the induction hypothesis for which part 2 in the proof will be useful. For all even $s$, we have
\begin{enumerate}
\item
  ${\M_s(\rhotild_{k,t} \rvert x_k) \leq g_s(x_k) + C_s t}.$
\item
  ${\EE{-\dotprod{\grad f(x_k)}{Z}g_s(\xtild_{k, t}) \rvert x_k} 
  \leq N_s}.$
\end{enumerate}

For the first inequality above, we will bound the time derivative of ${\M_s(\rhotild_{k,t} \rvert x_k)}$ as follows.

\begin{equation*}
\begin{split}
  \frac{\partial}{\partial t} \M_s(&\rhotild_{k,t} \rvert x_k)\\
  & = \EE{-s\dotprod{\grad f(x_k)}{\xtild_{k, t}}g_{s-2}(\xtild_{k, t})+ 
  s(d+s-2)g_{s-2}(\xtild_{k, t}) - s(s-2)g_{s-4}(\xtild_{k, t})\rvert x_k}\\
  & \leq s\EE{\left( -\dotprod{\grad f(x_k)}{x_k - t \grad f(x_k) + \sqrt{2t}Z}+
  (d+s-2)\right)g_{s-2}(\xtild_{k, t})\rvert x_k}\\
  & \leq s \left( 
      - \dotprod{\grad f(x_k)}{x_k} 
      + t \norm{\grad f(x_k)}^2   
      + (d+s-2)
  \right)\EE{g_{s-2}(\xtild_{k, t})\rvert x_k}\\
  &\qquad+ s\sqrt{2t}\EE{\dotprod{-\grad f(x_k)}{Z}g_{s-2}(\xtild_{k, t})\rvert x_k}\\
  & \leq 
  s\left[
      -a \norm{x_k}^\alpha + b + 2 \eta \growth^2(1+\norm{x_k}^{2\zeta}) + (d+s-2)
  \right]_+(g_{s-2}(x_k)+C_{s-2}t)
  +s\sqrt{2\eta} N_{s-2} \\
  & \leq
  s\left[ 
      -a (1+\norm{x_k}^2)^\frac{\alpha}{2}
      + 2 \eta \growth^2 (1 + \norm{x_k}^2)^{\zeta}
      +(2 \eta \growth^2 + a + b + d + s - 2)
  \right]_+ (g_{s-2}(x_k)+C_{s-2}t)\\
  &\qquad +s\sqrt{2\eta} N_{s-2}\\
  & \stackrel{1}{\leq} \max_{u \geq 1} s\left( 
      -a u^{\alpha}
      + 2 \eta \growth^2 u^{2\zeta}
      +(2 \eta \growth^2 + a + b + d + s - 2)
  \right) (u^{s-2}+C_{s-2}t)
  +s\sqrt{2\eta} N_{s-2} \\
  & \leq s\max_{u \geq 1} \left(
      -\frac{a}{2} u^{\alpha+s-2}
      + 2 \eta \growth^2 u^{2\zeta+s-2}
  \right)\\
  & \qquad + s\max_{u \geq 1} \left(
      -\frac{a}{2} u^{\alpha+s-2}
      + (2 \eta \growth^2 + a + b + d + s - 2) u^{s-2}
  \right) \\
  & \qquad + s C_{s-2} \eta \max_{u \geq 1} \left(
      -a u^{\alpha}
      + 2 \eta \growth^2 u^{2\zeta}
      +(2 \eta \growth^2 + a + b + d + s - 2)
  \right)
  +s\sqrt{2\eta} N_{s-2}\\
  &\stackrel{2}{\leq }
  s \bigg[
  (2 \eta \growth^2 + a + b + d + s - 2) \left( \frac{2(2 \eta \growth^2 + a + b + d + s - 2) (s-2)}{a (\alpha+s-2)} \right)^\frac{s-2}{\alpha}\\
  & \qquad +C_{s-2}\eta(2 \eta \growth^2 + a + b + d + s - 2)
  + \sqrt{2\eta} N_{s-2}
  \bigg],
\end{split}
\end{equation*}

\noindent in which substitution  ${u = \sqrt{1+\norm{x_k}^2}}$ is used in step 1 and Lemma~\ref{lemma:poly_max} is used in step 2.
The above inequality shows ${\M_s(\rhotild_{k,t} \rvert x_k) \leq \M_s(x_k) + C_s t}$ for any $C_s$ satisfying

\begin{equation}\label{eq:C_s_rec}
\begin{split}
  \frac{C_s}{s} &\geq
  (2 \eta \growth^2 + a + b + d + s - 2) \left(\frac{2(2 \eta \growth^2 + a + b + d + s - 2)}{a} \right)^\frac{s-2}{\alpha}\\
  & \qquad + C_{s-2}\eta(2 \eta \growth^2 + a + b + d + s - 2) + \sqrt{2\eta} N_{s-2}.
\end{split}
\end{equation}
For proving the second part of the induction step, we use Stein's lemma \cite{stein1981estimation} (the version we use is stated in Lemma~\ref{lemma:stein}) in the first equality below.

\begin{equation*}
\begin{split}
  \mathbb{E}[-\dotprod{\grad f(x_k)}{Z} & g_s(\xtild_{k, t}) \rvert x_k]\\
  & = \EE{-s\sqrt{2t}\dotprod{\grad f(x_k)}{
      g_{s-2}(\xtild_{k, t})\left(
          x_k - t \grad f(x_k) + \sqrt{2t}Z
        \right)
  } \rvert x_k}\\
  & \leq s\sqrt{2t}\left(  -\dotprod{\grad f(x_k)}{x_k}
  + t \norm{\grad f(x_k)} ^2 \right)
  \M_{s-2}(\rhotild_{k,t} \rvert x_k)
  + 2 s t N_{s-2}\\
  & \leq s \sqrt{2\eta} \left[
      - a\norm{x_k}^\alpha +b
      + 2 \eta \growth^2(1+\norm{x_k}^{2\zeta})
  \right]_+(g_{s-2}(x_k) + C_{s-2}\eta)
  + 2 s \eta N_{s-2}\\
  & \leq s \sqrt{2\eta}
  \left[
    - a(1+\norm{x_k}^2)^\frac{\alpha}{2}
    + 2 \eta \growth^2(1+\norm{x_k}^2)^{\zeta}
    + (b + a + 2 \eta \growth^2)
    \right]_+ \left(g_{s-2}(x_k) + C_{s-2}\eta\right)\\
   &\qquad + 2 s \eta N_{s-2}\\
  & \leq s \sqrt{2\eta} \max_{u \geq 1}\left(
      - a u^\alpha
      + 2 \eta \growth^2 u^{2\zeta}
      + (b + a + 2\eta \growth^2)
  \right) (u^{s-2} + C_{s-2}\eta)
  + 2 s \eta N_{s-2}\\
  & \leq s \sqrt{2\eta} \max_{u \geq 1} \left(
      - \frac{a}{2} u^{\alpha+s-2}
      + 2 \eta \growth^2 u^{2\zeta+s-2}
  \right)\\
  & \qquad + s \sqrt{2\eta} \max_{u \geq 1} \left(
      - \frac{a}{2} u^{\alpha+s-2}
      + (b+a+ 2 \eta \growth^2)u^{s-2}
  \right)\\
  & \qquad + s \sqrt{2\eta} C_{s-2}\eta \max_{u} \left(
      - a u^\alpha
      + 2 \eta \growth^2 u^{2\zeta}
      + (b + a + 2 \eta \growth^2)
  \right)
  + 2 s \eta N_{s-2}\\
  &\leq
  s\bigg[
 (b + a + 2\eta \growth^2)\sqrt{2\eta} \left(\frac{2(b + a + 2 \eta \growth^2) (s-2)}{a (\alpha+s-2)} \right)^\frac{s-2}{\alpha} \\
  & \qquad + C_{s-2}\eta\sqrt{2\eta} (b + a + 2 \eta \growth^2) + 2\eta N_{s-2} \bigg]\leq N_s,
\end{split}
\end{equation*}
where
\begin{equation}\label{eq:N_s_rec}
\begin{split}
  \frac{N_s}{s} = (b+a+ 2\eta \growth^2)\sqrt{2\eta} \left(\frac{2(b+a+ 2\eta \growth^2)}{a} \right)^\frac{s-2}{\alpha} 
  + C_{s-2}\eta\sqrt{2\eta} (b+a+2\eta \growth^2) + 2\eta N_{s-2}.
\end{split}
\end{equation}
Again, the substitution ${u=\sqrt{1+\norm{x}^2}}$ is used here. This completes the induction. 

The previous induction showed ${\M_s(\rhotild_{k,t} \rvert x_k) \leq g_s(x_k) + C_s t}$
when $s$ is a positive even integer.
We take expectation with respect to $x_k$ in order to get
\begin{equation}
  \M_s(\rhotild_{k,t}) \leq \M_s(\rho_k) + C_s t,
\end{equation}
setting ${t=\eta}$ yields
\begin{equation*}
  \M_s(\rho_{k+1}) \leq \M_s(\rho_k) + C_s \eta,
\end{equation*}
and finally, induction on $k$ gives
\begin{equation*}
  \M_s(\rho_k) \leq \M_s(\rho_0) + C_s k \eta.
\end{equation*}

\noindent\textbf{Part 4.} In this part, we establish a non-recursive formula for $C_s$.
Note that the theorem holds for a larger $C_s$,
this helps us to derive a closed-form formula for $C_s$.
We combine \eqref{eq:C_2} and \eqref{eq:N-2} to get ${N_2 \leq C_2 \sqrt{2\eta}}$, then we 
use \eqref{eq:C_s_rec} and \eqref{eq:N_s_rec} inductively, to establish ${N_s \leq C_s \sqrt{2 \eta}}$.
By combining the previous inequality with ${4\eta \growth^2 \leq a}$, we can strengthen the bound \eqref{eq:C_s_rec} to
\begin{equation*}
\begin{split}
  C_s \geq
  \left(\frac{3a + 2b + 2d + 2s}{1 \wedge a} \right)^{\frac{s-2}{\alpha}+1}s
  +\frac{3a + 2b + 2d + 2s}{1 \wedge a} \times C_{s-2} s \eta.
\end{split}
\end{equation*}
$C_s$, as defined in \eqref{eq:C_s_closed}, satisfies the previous inequality and \eqref{eq:C_2}, which in turn implies that it also satisfies \eqref{eq:C_s_rec} and \eqref{eq:C_2}.
\end{proof-of-proposition}

The next proposition is the analog moment bound for the continuous-time process, and is adapted from \cite{toscani2000trend} 
for the sake of comparison with the bound for the discrete time process.
\begin{lemma}\label{lemma:moment_bound}
  Let $f$ satisfy Assumption~\ref{as:mild_dis} and $p_t$ be the distribution of $Z_t$, then
  \begin{equation*}
    \M_s(p_t + \target) \leq \M_s(p_0 + \target) + K_s t,
  \end{equation*}
  where ${K_s = (b+d+a+s-2)\left( \frac{b+d+a+s-2}{a}\right)^\frac{s-2}{\alpha}s}$.
\end{lemma}

\begin{proof}
Because of linearity of integral, it is sufficient to prove
\eqn{
  \M_s(p_t) \leq \M_s(p_0) + K_s t.
}
Note that if ${s<s'}$ then
${\M_s(p_t) = \int p_t(x) (1+\norm{x}^2)^{\tfrac{s}{2}} \leq \int p_t(x) (1+\norm{x}^2)^{\tfrac{s'}{2}} = \M_{s'}(p_t)}$.
We differentiate ${\M_s(p_t)}$ with respect to time.
\begin{equation*}
\begin{split}
    \frac{d}{dt} \M_s(p_t) &=
    \int p_t(x)
    \left[ 
        \Delta (1+\norm{x}^2)^{\tfrac{s}{2}} - \dotprod{\grad f(x)}{\grad (1+\norm{x}^2)^{\tfrac{s}{2}}}
    \right]\\
    &= \left(ds+s(s-2)\right)\M_{s-2}(p_t) - s(s-2) \M_{s-4}(p_t)
    -s\int p_t(x) \dotprod{\grad f(x)}{x}(1+\norm{x}^2)^{\tfrac{s-2}{2}}\\
    &\leq (b+d+s-2)s \M_{s-2}(p_t) -  s\int p_t(x) a\norm{x}^{\alpha}(1+\norm{x}^2)^{\tfrac{s-2}{2}}\\
    &\leq (b+d+a+s-2)s \M_{s-2}(p_t) - \frac{as}{2} \M_{s+\alpha-2}(p_t)\\
    &\leq  (b+d+a+s-2)s \M_{s+\alpha-2}(p_t)^{\tfrac{s-2}{s+\alpha-2}} - \frac{as}{2} \M_{s+\alpha-2}(p_t)\\
    &\stackrel{1}{\leq} (b+d+a+s-2)s\left( \frac{2(b+d+a+s-2)(s-2)}{a(s+\alpha-2)}\right)^\frac{s-2}{\alpha}\\
    &\leq (b+d+a+s-2)\left( \frac{b+d+a+s-2}{a/2}\right)^\frac{s-2}{\alpha}s,
\end{split}
\end{equation*}
where step 1 follows from Lemma~\ref{lemma:poly_max}.
\end{proof}

We utilize a method similar to the previous proof in order to bound the moments of target

\begin{proof-of-lemma}[\ref{lemma:target-moment-bound}]
  From the proof of Lemma~\ref{lemma:moment_bound}, we have the following inequality for ${s\geq 2}$.
  \eqn{
    \frac{d}{dt} \M_s(p_t) \leq
    (b+d+a+s-2)s \M_{s-2}(p_t) - \frac{as}{2} \M_{s+\alpha-2}(p_t).
  }
  If we let ${p_0=\target}$, then ${p_t=\target}$ which means that 
  the left hand side of the above inequality is zero.
  The derivative is well defined because Lemma~\ref{lemma:f_bound} shows that ${\M_s(\target)}$ is finite.
  By rearranging the previous inequality, we get
  \eqn{
    \M_{s+\alpha-2}(p_t) \leq \frac{2(b+d+a+s-2)}{a}\M_{s-2}(p_t).
  }
  Using the above inequality inductively from ${s=2}$, we get
  \eqn{
    \M_{k\alpha}(p_t) \leq \left( \frac{2}{a} \right)^{k}(a+b+d+(k-1)\alpha)^k.
  }
  For every $s$ there is an integer $k$ such that ${k\alpha \leq s < (k+1)\alpha}$. 
  We have the following bound
  \eqn{
    \M_{s}(p_t) \leq {\M_{(k+1)\alpha}(p_t)}^{\frac{s}{(k+1)\alpha}} 
    \leq \left( \frac{2}{a} \right)^{s/\alpha}(a+b+d+k\alpha)^{s/\alpha}
    \leq \left( \frac{a+b+3}{a} \right)^{s/\alpha} s^{s/\alpha} d^{s/\alpha}.
  }
\end{proof-of-lemma}

\section{Proof of The Main Theorem}\label{sec:proof-main}
The proof will be done in three parts. In the first part,
we bound the $\alpha$-th moment of a given distribution with 
its KL-divergence from the $\target$. In the second part,
the bound derived in the first part will be used to
construct a differential inequality on the interpolation diffusion.
Next, using a comparison theorem on the differential inequality, we
will derive a single step bound on the \lmc iterates.
Finally in the last part, we will iterate the single step bound to obtain
a non-asymptotic convergence rate.

\subsection{Bounding \lmc moments with KL-divergence}
The behavior of the discrete-time process is different from that of the continuous-time diffusion in that,
a step size dependent bias term appears in the
differential inequality that governs its evolution.
The results in this section will help us handle the bias term.
First, using Assumption~\ref{as:mild_dis}, we
prove that the potential grows at least like $\norm{x}^\alpha$ in Lemma~\ref{lemma:f_bound}.
Using this growth, we bound the $\alpha$-th exponential moment
of the target $\target$ in Lemma~\ref{lemma:exp_mom}.
Finally, using the exponential moment bound, in Lemma~\ref{lemma:moment_kl_bound},
we upper bound the $\alpha$-th moment of a given
distribution with its KL-divergence from the $\target$.
Although this step can be handled easily by Talagrand's inequality in the case of $\alpha=2$,
it is more challenging for $\alpha\in[1,2)$.

\begin{lemma}\label{lemma:f_bound} 
  If $f$ satisfies Assumption~\ref{as:mild_dis}, then
  \begin{equation*}
      f(x) \geq \frac{a}{2\alpha} \norm{x}^\alpha + f(0) 
      - \growth \left(\frac{2a+2b}{a}\right)^2 - b.
  \end{equation*}
  \end{lemma}
  
  \begin{proof}
  For notational ease, let ${R = \left( \frac{2b}{a} \right)^\frac{1}{\alpha}}$.
  First, using the gradient growth condition in Assumption~\ref{as:mild_dis}, we upper bound $\norm{\grad f(x)}$ when ${x \leq R}$.
  \begin{align}
    \norm{\grad f(x)}
    \leq
    \max_{\norm{x} \leq R} \growth(1+\norm{x}^\zeta)
    \leq 
    \growth\left(1+\left( \frac{2b}{a} \right)^{{\zeta}/{\alpha}}\right)
    \leq
    \frac{\growth(2a+2b)}{a}.
  \end{align}
  Now using Assumption \ref{as:mild_dis} we lower bound $f$.
  \begin{equation*}
  \begin{split}
      f(x)
      &= f(0) + \int_{0}^{\frac{R}{\norm{x}}} \dotprod{\grad f(tx)}{x} dt
      + \int_{\frac{R}{\norm{x}}}^{1} \dotprod{\grad f(tx)}{x} dt\\
      &\geq f(0) - \int_{0}^{\frac{R}{\norm{x}}} \norm{\grad f(tx)}\norm{x} dt
      + \int_{\frac{R}{\norm{x}}}^{1} \frac{1}{t} \dotprod{\grad f(tx)}{tx} dt\\
      &\geq f(0) - \left( \frac{\growth(2a+2b)}{a}\right) R + \int_{\frac{R}{\norm{x}}}^{1} \frac{1}{t} \left( a\norm{tx}^\alpha -b \right) dt\\
      &\stackrel{1}{\geq} f(0) - \growth \left(\frac{2a+2b}{a}\right)^2 + \frac{a}{2}\norm{x}^\alpha \int_{\frac{R}{\norm{x}}}^{1} t^{\alpha-1} dt\\
      & \geq f(0) - \growth \left(\frac{2a+2b}{a}\right)^2 + \frac{a}{2\alpha} \norm{x}^\alpha \left(1 - \frac{R^\alpha}{\norm{x}^\alpha} \right)\\
      & \geq \frac{a}{2\alpha} \norm{x}^\alpha + f(0) - \growth \left(\frac{2a+2b}{a}\right)^2 - b.
  \end{split}
  \end{equation*}
  where step 1 uses the fact that if ${t \geq \frac{R}{\norm{x}}}$ then ${a\norm{tx}^\alpha - b \geq \frac{a}{2} \norm{tx}^\alpha}$.
  \end{proof}

We use Lemma~\ref{lemma:f_bound} to prove that the $\alpha$-th exponential moment of the target $\target$ is bounded.

\begin{lemma}\label{lemma:exp_mom}
  If $f$ satisfies Assumption~\ref{as:mild_dis}, then $$0 < \log{\left( \int e^{\frac{a}{4\alpha} \norm{x}^\alpha -f(x)} \right)} \leq  \td\tmu,$$
  where,
  \eq{\label{eq:tdtmu}
  \left\{
  \begin{array}{ll}
    \tmu &=  \log\left(\frac{16\pi}{a}\right) +\growth \left(\frac{2a+2b}{a}\right)^2 + b +\abs{f(0)},\\
    \td &= d\left(1+ (1-\alpha/2)\log(d) \right).
  \end{array}
  \right.
}
\end{lemma}

\begin{proof}
Using Lemma~\ref{lemma:f_bound} we get
\begin{equation*}
\begin{split}
    \int e^{\frac{a}{4\alpha} \norm{x}^\alpha -f(x)} dx
    &\leq
    e^{-f(0) + \growth \left(\frac{2a+2b}{a}\right)^2 + b}\int e^{-\frac{a}{4\alpha} \norm{x}^\alpha} dx\\
    &= \frac{2\pi^{d/2}}{\alpha} \left(\frac{4\alpha}{a}\right)^{d/\alpha}
    e^{-f(0) + \growth \left(\frac{2a+2b}{a}\right)^2 + b} 
     \frac{\Gamma(d/\alpha)}{\Gamma(d/2)}.
\end{split}
\end{equation*}
Next, using an inequality for the ratio of Gamma functions \cite{JovanD1971}, 
we obtain
\eq{
  \frac{\Gamma(d/\alpha)}{\Gamma(d/2)}
  \leq \frac{(d/\alpha)^{\frac{d}{\alpha} - \frac{1}{2}}}{(d/2)^{\frac{d}{2}-\frac{1}{2}}}
  e^{\frac{d}{2} - \frac{d}{\alpha}}.
}
Plugging this back into the previous bound and taking logs, we obtain
\eq{
  \log \left(\int e^{\frac{a}{4\alpha} \norm{x}^\alpha -f(x)} dx\right)
  \leq&
  \frac{d}{2}  \log\left(\pi \right)
  +\frac{d}{\alpha}  \log\left( \frac{4\alpha}{a}\right)
  +\left(\frac{d}{\alpha} -\frac{d}{2}\right)  \log\left(\frac{d}{2e} \right)\\
  & \qquad +\left(\frac{d}{\alpha} +\frac{1}{2}\right)  \log\left(\frac{2}{\alpha} \right)
  +\growth \left(\frac{2a+2b}{a}\right)^2 + b +\abs{f(0)}\\
  \leq& \frac{d}{\alpha}\left( \log\left(\frac{16 \pi}{a}\right)
    + \left(1-\frac{\alpha}{2} \right)\log\left( \frac{d}{2e}\right)
  \right)+\growth \left(\frac{2a+2b}{a}\right)^2 + b +\abs{f(0)}\\
  \leq& \td \tmu.
}
\end{proof}

Finally, using the previous lemma, we will bound the $\alpha$-th moment of any distribution $\rho$ using its KL-divergence from the target $\target$.

\begin{lemma}\label{lemma:moment_kl_bound}
If the potential $f$ satisfies Assumption~\ref{as:mild_dis}, then for ${\target = e^{-f}}$ and any distribution $\rho$, we have
\begin{equation}
    \frac{4\alpha}{a}\left[ 
        \KL{\rho}{\target} + \td \tmu
    \right]
    \geq
    \Esub{\norm{x}^\alpha}{\rho}.
\end{equation}
\end{lemma}
\begin{proof}
Let  ${q(x) = e^{\tfrac{a}{4\alpha} \norm{x}^\alpha - f(x)}}$. Let $z$ be number 
such that $q(x)/z$ be a probability distribution.
Lemma~\ref{lemma:exp_mom} implies $\log{z} \leq \td\tmu$. Using this bound on $z$ we get
\begin{equation*}
\begin{split}
    \KL{\rho}{\target}
    = \int \rho \log{\frac{\rho}{q/z}} + \int \rho \log{\frac{q/z}{\target}}
    = \KL{\rho}{q/z} + \Esub{\log{\frac{q/z}{e^{-f}}}}{\rho}
    \geq \frac{a}{4\alpha} \Esub{\norm{x}^\alpha}{\rho} - \td \tmu.
\end{split}
\end{equation*}
Rearranging this yields the desired inequality.
\end{proof}

\subsection{Single step bound}
The proof strategy is to consider the continuous-time interpolation of a single LMC iteration
\begin{equation}\label{eq:ULA_inter}
  d\xtild_{k, t} = -\grad f(x_k) d t + \sqrt{2} d B_t \ \ \text{ with }\ \ \xtild_{k, 0} =x_k,
\end{equation}
where $x_k$ is the $k$-th iterate of the \lmc algorithm \eqref{eq:ULA}. Denoting the distributions of $x_k$
and $\xtild_{k, t}$ with $\rho_k$ and $\rhotild_{k, t}$, respectively,
we notice that ${\rhotild_{k, 0} = \rho_k}$ and $\xtild_{k, \eta}\sim \rho_{k+1}$.
In the following, we construct a differential inequality for the KL-divergence between $\rhotild_{k, t}$ and the target.
This inequality will be used together with
the modified log-Sobolev inequality Theorem~\ref{thm:MLSI} and the linear moment bounds Proposition~\ref{prop:disc_mom_bound} to obtain a single step bound.

The time derivative of the KL-divergence between $\rhotild_{k, t}$ and the target $\target$
has an additional bias term compared to the diffusion process~\eqref{eq:overdamped}.
The next lemma characterizes this bias and is adapted from \cite{vempala2019rapid}.

\begin{lemma}[\textbf{\cite{vempala2019rapid}}]
  \label{lemma:dis_fok_plan}
  Suppose $\xtild_{k, t}$ is the interpolation of the discretized process \eqref{eq:ULA_inter}. Let $\rhotild_{k, t}$
  denote its distribution. Then
\begin{equation}\label{eq:dis_fok_plan}
\begin{gathered}
  \frac{d}{d t} \KL{\rhotild_{k, t}}{\target} = 
  -\I{\rhotild_{k, t}}{\target} + 
  \EE{\left< \grad f(\xtild_{k, t}) - \grad f(x_k) , \grad \log\left(\frac{\rhotild_{k, t}(\xtild_{k,t})}{\target(\xtild_{k,t})}\right) \right>} \\
  \leq -\frac{3}{4}\I{\rhotild_{k, t}}{\target}
  + \EE{\norm{\grad f(\xtild_{k,t}) - \grad f(x_k)}^2}.
\end{gathered}
\end{equation}
\end{lemma}

\begin{proof}
  The following proof is included for reader's convenience.
  For further notational convenience, we denote
  with ${\rhotild_{kt}(x_k, \xtild_{k, t})}$, the joint distribution of random variables $x_k$ and $\xtild_{k, t}$,
  and similarly, we denote with ${\rhotild_{k \vert t}(x_k)}$ and $\rhotild_{t \vert k}(\xtild_{k, t})$, the conditional distributions of $x_k$ conditioned on $\xtild_{k,t}$, and
  $\xtild_{k,t}$ conditioned on $x_k$, respectively.

The distribution of $\xtild_{k, t}$ conditioned on $x_k$ can be described by the following Fokker-Planck equation.
\begin{equation*}
  \frac{\partial \rhotild_{t \vert k}(\xtild_{k, t})}{\partial t  }
  = \grad \cdot
  \left(\rhotild_{t \vert k}(\xtild_{k, t}) \grad f(x_k) \right)
    + \Delta \rhotild_{t \vert k}(\xtild_{k, t}).
\end{equation*}
Taking expectation with respect to $x_k$ yields
\begin{equation*}
\begin{split}
  \frac{\partial \rhotild_{k, t}(x)}{\partial t  }
  &= \int
    \frac{\partial \rhotild_{t \vert k}(x)}{\partial t  } \rho(x_k) d x_k\\
  &= \int \left(
    \grad \cdot 
  \left(\rhotild_{kt}(x_k, x) \grad f(x_k) \right)
    + \Delta \rhotild_{kt}(x_k, x)
  \right) d x_k \\
  &= \grad \cdot \left(
    \rho(x_k) \int 
    \rhotild_{k\vert t}(x_k) \grad f(x_k) d x_k
  \right) + \Delta \rhotild_{k, t}(x)\\
  &= \grad \cdot \left(
    \rhotild_{k,t}(x)
    \Esub{\grad f(x_k)\vert \xtild_{k, t}=x}{\rhotild_{k\vert t}}
  \right) + \Delta \rhotild_{k, t}(x).
\end{split}
\end{equation*}

This equality is combined with the time derivative of KL-divergence to prove the claim.
\begin{equation*}
\begin{split}
  \frac{d}{d t} \KL{\rhotild_{k, t}}{\target} 
  &= \frac{d}{d t} \int \rhotild_{k,t}(x) \log{\left( \frac{\rhotild_{k,t}(x)}{\target(x)} \right)} dx\\
  &= \int \frac{\partial \rhotild_{k, t}}{\partial t}(x) \log{\left( \frac{\rhotild_{k,t}(x)}{\target(x)} \right)} dx\\
  &= \int \left( 
      \grad \cdot \left(
      \rhotild_{k,t}(x)
      \Esub{\grad f(x_k) \vert \xtild_{k, t}=x}{\rhotild_{k\vert t}}
  \right) + \Delta \rhotild_{k, t}(x)
  \right)\log{\left( \frac{\rhotild_{k,t}(x)}{\target(x)} \right)} dx\\
  &\stackrel{1}{=} \int \left( 
      \grad \cdot \left(
      \rhotild_{k,t}(x)
      \Esub{\grad f(x_k) \vert \xtild_{k, t}=x}{\rhotild_{k\vert t}}
      + \grad \log \left( \frac{\rhotild_{k,t}(x)}{\target(x)} \right) - \grad f(x)
  \right)
  \right)\log{\left( \frac{\rhotild_{k,t}(x)}{\target(x)} \right)} dx\\
  &\stackrel{2}{=} -\int \rhotild_{k, t}(x)
  \dotprod{
      \Esub{\grad f(x_k) \vert \xtild_{k, t}=x}{\rhotild_{k\vert t}}
      + \grad \log \left( \frac{\rhotild_{k,t}(x)}{\target(x)} \right) - \grad f(x)
  }{
      \grad \log{\left( \frac{\rhotild_{k,t}(x)}{\target(x)} \right)}
  }dx\\
  &= -\I{\rhotild_{k, t}}{\target}
  + \int \rhotild_{k, t}(x) \dotprod{
      \grad f(x)-\Esub{\grad f(x_k) \vert \xtild_{k, t}=x}{\rhotild_{k\vert t}}
  }{
      \grad \log{\left( \frac{\rhotild_{k,t}(x)}{\target(x)} \right)}
  }dx\\
  &= -\I{\rhotild_{k, t}}{\target}
  + \Esub{\dotprod{\grad f(\xtild_{k, t}) - \grad f(x_k)}
  {\grad \log{\left( \frac{\rhotild_{k,t}(\xtild_{k, t})}{\target(\xtild_{k, t})} \right)}}}{\rhotild_{kt}}\\
  &\stackrel{3}{\leq} - \I{\rhotild_{k, t}}{\target}
  + \Esub{\norm{\grad f(\xtild_{k,t})- f(x_k)}^2}{\rhotild_{kt}}
  + \frac{1}{4}\Esub{\norm{\grad \log{\left(\frac{\rhotild_{k,t}(\xtild_{k, t})}{\target(\xtild_{k, t})}\right)}}^2}{\rhotild_{kt}}\\
  &= -\frac{3}{4} \I{\rhotild_{k, t}}{\target}
  + \EE{\norm{\grad f(\xtild_{k,t})- f(x_k)   }^2},
\end{split}
\end{equation*}
in which equality 1 is follows from ${\Delta \rhotild_{k,t} = \grad \cdot (\grad \rhotild_{k,t})}$,
equality 2 follows from the divergence theorem, inequality 3 follows from ${\dotprod{u}{v} \leq \norm{u}^2 + \frac{1}{4} \norm{v}^2}$, and in the last step, the subscript of the expectation is removed and indicates that the expectation is taken with respect to both $x_k$ and $\xtild_{k, t}$.
\end{proof}

Next, using Lemma~\ref{lemma:dis_fok_plan}, we bound the time derivative of the KL-divergence ${\frac{d}{d t} \KL{\rhotild_{k, t}}{\target}}$, and obtain a useful differential inequality.
\begin{lemma} \label{lemma:kl_deriv_bound}
If the potential $f$ satisfies Assumptions~\ref{as:degen_conv}, \ref{as:mild_dis} and \ref{as:holder}, then 
\begin{equation}\label{eq:kl_deriv_bound}
\begin{split}
  \frac{d}{d t} \KL{\rhotild_{k, t}}{\target}
  & \leq 
  -\frac{3}{4}\lambda^{-\frac{1}{1-\delta}}
  \left(\M_s(\rho_0 + \target) + C_s(k+1)\eta\right)^{-\frac{\delta}{1-\delta}}
  \KL{\rhotild_{k, t}}{\target}^{\frac{1}{1-\delta}}\\
  & \qquad + \frac{16 \alpha \smooth^2 \growth^{2\beta}}{a}\KL{\rho_k}{\target}\eta^{2\beta}
  + 4 \smooth^2 \left(
  1 + \growth^{2\beta}\left(1 +
  \frac{2\alpha\tmu}{a} \right)
  \right) \td \eta^{\beta},
\end{split}
\end{equation}
when ${t\leq\eta\leq \frac{1}{2}\left(1 \wedge \frac{a}{2\growth^2}\right)}$. The constants $\td$ and $\tmu$ are defined in \eqref{eq:tdtmu}.
\end{lemma}

\begin{proof}
We bound ${\EE{\norm{\grad f(\xtild_{k,t}) - \grad f(x_k)}^2}}$ using Assumption~\ref{as:holder}
\begin{equation*}
\begin{split}
  \EE{\norm{\grad f(\xtild_{k,t}) - \grad f(x_k)}^2} & \leq \smooth^2\EE{\norm{\xtild_{k, t} - x_k}^{2\beta}}
  = \smooth^2\EE{\norm{-t \grad f(x_k) + \sqrt{2t}Z}^{2\beta}}\\
  & \stackrel{1}{\leq} 
  2\smooth^2t^{2\beta}\EE{\norm{\grad f(x_k)}^{2\beta}} 
  + 4\smooth^2 t^{\beta}\EE{\norm{Z}^{2\beta}} \\
  & \stackrel{2}{\leq} 2\smooth^2 t^{2\beta} \EE{\left(2\growth^2(1+\norm{x_k}^{2\zeta}) \right)^{\beta}} 
  + 4\smooth^2 t^{\beta}\EE{\norm{Z}^2}^{\beta}\\
  & \leq 4t^{2\beta}\smooth^2\growth^{2\beta}\EE{1+\norm{x_k}^{2\beta\zeta}}
  + 4\smooth^2 d^{\beta} t^{\beta}\\
  & \stackrel{3}{\leq} 4t^{2\beta}\smooth^2\growth^{2\beta}\EE{2+\norm{x_k}^{\alpha}}
  + 4 \smooth^2 d^{\beta} t^{\beta} \\
  & \stackrel{4}{\leq} \frac{16 \alpha\smooth^2\growth^{2\beta}}{a}\KL{\rho_k}{\target} \eta^{2\beta}
  + 4 \eta^{\beta} \smooth^2\left(
  d^{\beta} + 2 \left(\eta\growth^2\right)^\beta\left(1 + 
  \frac{2\alpha\tmu \td}{a} \right)
  \right)\\
  & \leq
  \frac{16 \alpha\smooth^2\growth^{2\beta}}{a}\KL{\rho_k}{\target} \eta^{2\beta}
  + 4 \td \smooth^2\left(
  1 + 2 a^\beta\left(1 + 
  \frac{2\alpha\tmu}{a} \right)
  \right)\eta^{\beta},
\end{split}
\end{equation*}
where step 1 follows from Lemma~\ref{lemma:power_triang},
step 2 from Assumption~\ref{as:mild_dis},
step 3 from the fact that ${2\zeta\beta \leq \alpha}$, and step 4 from Lemma~\ref{lemma:moment_kl_bound} and ${\eta < 1}$.
Plugging the above inequality back in \eqref{eq:dis_fok_plan} and using Theorem~\ref{thm:MLSI} and Proposition~\ref{prop:disc_mom_bound} results in \eqref{eq:kl_deriv_bound}.
\end{proof}
Finally, using a differential comparison argument, a single step bound is obtained on the KL-divergence of steps of LMC~\eqref{eq:ULA} from the target.

\begin{lemma}\label{lemma:single_step_bound}
Suppose $f$ satisfies Assumptions~\ref{as:degen_conv}, \ref{as:mild_dis} and \ref{as:holder}, then
\begin{equation}\label{eq:single_step_bound}
\begin{split}
  \KL{\rho_{k+1}}{\target}
  \leq&
  \KL{\rho_k}{\target}
  \left(
      1 - \frac{3\eta}{8 \lambda^\frac{1}{1-\delta}} \left(\frac{\KL{\rho_k}{\target}}{\M_s(\rho_0 + \target) + C_s(k+1)}\right)^\frac{\delta}{1-\delta}
      + \frac{16 \alpha \smooth^2 \growth^{2\beta}\eta^{2\beta+1}}{a}
  \right)\\
  &  + \sigma \td \eta^{\beta+1},
\end{split}
\end{equation}
where
${\sigma = 4 \smooth^2 \left(
  1 + 2a^\beta\left(1 + 
  \frac{2 \alpha \tmu}{a} \right)
  \right)}$.
The step size needs to be sufficiently small, satisfying
$${\eta \leq \frac{1}{2}\left( 1 \wedge \frac{a}{2\growth^2}\right) \wedge \left(\frac{4 \lambda^\frac{1}{1-\delta} }{3} \left(\frac{\M_s(\rho_0 + \target) + C_s(k+1)\eta}{\KL{\rho_k}{\target}}\right)^\frac{\delta}{1-\delta} \right)}.$$
\end{lemma}

\begin{proof}
Let
\begin{equation*}
\begin{split}
  \kappa_1 &= 
  \frac{3}{4}\lambda^{-\frac{1}{1-\delta}}
  \left(\M_s(\rho_0+\target) + C_s(k+1)\eta\right)^{-\frac{\delta}{1-\delta}},\\
  \kappa_2 &= 
  \frac{16 \alpha \smooth^2 \growth^{2\beta}}{a}\KL{\rho_k}{\target}\eta^{2\beta}
  + \sigma \td \eta^{\beta},\\
  \psi(t, x) &= - \kappa_1 x^{\frac{1}{1-\delta}} + \kappa_2,
\end{split}
\end{equation*}
where $\kappa_1$ and $\kappa_2$ are constants independent of $t$.
We can rewrite \eqref{eq:kl_deriv_bound} as
\begin{equation*}
  \frac{d}{d t} \KL{\rhotild_{k, t}}{\target}
  \leq
  \psi \left(t, \KL{\rhotild_{k, t}}{\target} \right).
\end{equation*}
For positive and sufficiently small $\teps$ (less than $\KL{\rho_k}{\target}^{-\frac{\delta}{1-\delta}}$), consider the function
\begin{equation*}
  h_{\teps} (t) = \left(
    \KL{\rho_k}{\target}^{-\frac{\delta}{1-\delta}}
    +
    \kappa_1 \frac{\delta}{1-\delta}t - \teps
  \right)^{-\frac{1-\delta}{\delta}} + \kappa_2 t.
\end{equation*}
We will use the following basic comparison lemma for differential inequalities; see, for example \cite{mcnabb1986comparison} for a simple proof.
\begin{lemma}\label{lemma:comparison}
  Suppose $u(t)$ and $v(t)$ are continuous on interval ${[a, b]}$ and differentiable on $(a,b]$, $f:\reals \times \reals\to\reals$ is a continuous mapping and
  \begin{equation*}
    u(a) < v(a)
    \qquad \textrm{and} \qquad
    \frac{d u}{d t} - f(t,u) < \frac{d v}{d t} - f(t,v),
    \qquad \textrm{on} \quad (a,b \, ].
  \end{equation*}
  Then ${u<v}$ on ${[a,b]}$.
  \end{lemma}
For positive $t$, we have
\begin{equation*}
  \frac{d}{d t} h_{\teps}(t) - \psi(t, h_{\teps}(t))
  >
  0
  \geq
  \frac{d}{d t}\KL{\rhotild_{k, t}}{\target}
  -\psi \left(t, \KL{\rhotild_{k, t}}{\target} \right).
\end{equation*}
Since ${h_{\teps}(0) > \KL{\rhotild_{k, 0}}{\target}}$, the previous comparison lemma implies
\begin{equation*}
  h_{\teps}(\eta) > \KL{\rhotild_{k, \eta}}{\target} =  \KL{\rho_{k+1}}{\target}.
\end{equation*}
Taking the limit of ${\teps \downarrow 0}$ gives
\begin{equation*}
  \KL{\rho_{k+1}}{\target}
  \leq
  \left(
  \KL{\rho_k}{\target}^{-\frac{\delta}{1-\delta}}
      +
      \kappa_1 \frac{\delta}{1-\delta}\eta
  \right)^{-\frac{1-\delta}{\delta}} + \kappa_2 \eta.
\end{equation*}
Plugging the values for $\kappa_1$ and $\kappa_2$ back in the previous inequality reads
\begin{equation*}
\begin{split}
  \KL{\rho_{k+1}}{\target}
  & \leq
  \left(
  \KL{\rho_k}{\target}^{-\frac{\delta}{1-\delta}}
      +
      \frac{3 \lambda^{-\frac{1}{1-\delta}} \delta}{4(1-\delta)}
      \left(\M_s(\rho_0+\target) + C_s(k+1)\eta\right)^{-\frac{\delta}{1-\delta}} 
      \eta
  \right)^{-\frac{1-\delta}{\delta}}\\
  & \qquad + \frac{16 \alpha \smooth^2\growth^{2\beta}}{a}\KL{\rho_k}{\target}\eta^{2\beta+1}
  + \sigma \td \eta^{\beta+1}.
\end{split}
\end{equation*}
We rewrite the previous inequality to get
\begin{equation*}
\begin{split}
  \KL{\rho_{k+1}}{\target}
  & \leq
  \frac{\KL{\rho_k}{\target}}{
    \left(
        1
        +
        \frac{3 \lambda^{-\frac{1}{1-\delta}} \delta}{4(1-\delta)}
        \left(\frac{\KL{\rho_k}{\target}}{\M_s(\rho_0 +\target) + C_s(k+1)\eta}\right)^\frac{\delta}{1-\delta}
        \eta
    \right)^\frac{1-\delta}{\delta}
    }
    + \frac{16 \alpha \smooth^2 \growth^{2\beta}}{a}\KL{\rho_k}{\target}\eta^{2\beta+1}\\
    &\qquad + \sigma \td \eta^{\beta+1}.
\end{split}
\end{equation*}
Using the fact that ${(1+x)^\frac{1-\delta}{\delta} \geq 1+ \frac{1-\delta}{\delta} x}$, in the denominator, yields
\begin{equation*}
\begin{split}
  \KL{\rho_{k+1}}{\target}
  & \leq
  \frac{\KL{\rho_k}{\target}}{
      1
      +
      \frac{3}{4 \lambda^\frac{1}{1-\delta}}
      \left(\frac{\KL{\rho_k}{\target}}{\M_s(\rho_0+\target) + C_s(k+1)\eta}\right)^\frac{\delta}{1-\delta}
      \eta
  }
  + \frac{16\alpha\smooth^2\growth^{2\beta}}{a}\KL{\rho_k}{\target}\eta^{2\beta+1}
  + \sigma \td \eta^{\beta+1}.
\end{split}
\end{equation*}
Since ${\frac{1}{1+x} < 1-\frac{x}{2}}$, when ${x \leq 1}$, and 
${\frac{3}{4 \lambda^\frac{1}{1-\delta}}
\left(\frac{\KL{\rho_k}{\target}}{\M_s(\rho_0+\target) + C_s(k+1)\eta}\right)^\frac{\delta}{1-\delta}
\eta < 1}$,
we have
\begin{equation*}
\begin{split}
  \KL{\rho_{k+1}}{\target}
  & \leq
  \KL{\rho_k}{\target}\left(
      1
      -
      \frac{3}{8 \lambda^\frac{1}{1-\delta}}
      \left(\frac{\KL{\rho_k}{\target}}{\M_s(\rho_0+\target) + C_s(k+1)\eta}\right)^\frac{\delta}{1-\delta}
      \eta
  \right)\\ 
  & \qquad + \frac{16 \alpha \smooth^2 \growth^{2\beta}}{a}\KL{\rho_k}{\target}\eta^{2\beta+1}
  + \sigma \td \eta^{\beta+1}.
\end{split}
\end{equation*}
Rearranging the above inequality yields the desired result.
\end{proof}

\subsection{Proof of the main theorem}
In this section, we prove the convergence of the \lmc algorithm by iterating
the single step bound obtained in the previous section.
More specifically, we establish that the algorithm reaches the desired accuracy $\eps$ after $N$ steps,
for which our argument relies on two steps. 
In the first step, we prove that if an iterate of LMC reaches the desired accuracy before 
$N$ steps, then it will remain below that accuracy level until the step $N$.
In the second step,
we show that if LMC does not reach $\eps$ accuracy before $N$ steps,
it is guaranteed to reach that accuracy at the step $N$.
Since the single step bound obtained in Lemma~\ref{lemma:single_step_bound} is quite convoluted,
we first simplify it to a manageable recursive formula, and iterate the resulting expression.
Special care is taken
to determine the upper bound on the accuracy for the aforementioned claims to hold.
The bound on $\eps$ is independent of the moment order $s$,
which is crucial for tuning this parameter to obtain the final convergence rate
leading to the main corollary.

\begin{proof-of-theorem}[\ref{thm:main}]
  We simplify the recurrence relation for the single step bound in \eqref{eq:single_step_bound}.
  For notational convenience, let 
  ${A = \frac{\lam^{-1/(1-\delta)}}{16}
    \left(\frac{\sigma \td}{\M_s(\rho_0+\target) \vee C_s}\right)^{\delta/(1-\delta)}}.$
  We remind that $\td$ is defined as ${\td = d\left(1+ (1-\alpha/2)\log(d) \right)}$.
  We will show that under the conditions and notations of Lemma~\ref{lemma:single_step_bound},
  if ${k<N}$ and ${\KL{\rho_k}{\target} \geq \eps/2}$, then
  \begin{equation}\label{eq:single_step_bound_clean}
    \KL{\rho_{k+1}}{\target} \leq \left(1-\frac{A\eta^{\delta\beta/(1-\delta)+1}}{\log{\left(\frac{2 \Delta_0}{\eps}\right)}^{\delta/(1-\delta)}}\right)
    \KL{\rho_{k}}{\target} +
    \sigma \td \eta^{\beta+1}.
  \end{equation}
  The above expression depends on the choice of step size $\eta$ and number of steps $N$; thus,
  given \eqref{eq:single_step_bound}, we verify the inequality \eqref{eq:single_step_bound_clean} for
  \begin{equation}\label{eq:conv_rate_raw}
  \begin{split}
    \eta^{-1}
    & = 
    (\sigma \td)^{\frac{1}{\beta}}
    (16 \lam)^{\frac{1}{\beta(1-2\delta)}}
    \left(\frac{\M_s(\rho_0 + \target) \vee C_s}{16}\right)^{\frac{\delta}{\beta(1-2\delta)} }
    \log{\left(\frac{2 \Delta_0}{\eps}\right)}^{\frac{\delta}{\beta(1-2\delta)}}
    \left(\frac{2}{\eps}\right)^{\frac{1-\delta}{\beta(1-2\delta)}},\\
    N 
    &= 
    (\sigma \td)^{\frac{1}{\beta}}
    (16\lam)^\frac{1+\beta}{\beta(1-2\delta)}
    \left(\frac{\M_s(\rho_0 + \target) \vee C_s}{16}\right)^{\frac{(1+\beta)\delta}{\beta(1-2\delta)}}
    \log{\left(\frac{2\Delta_0}{\eps}\right)}^{1+ \frac{(\beta+1)\delta}{\beta(1-2\delta)}}
    \left(\frac{2}{\eps}\right)^\frac{1-\delta(1-\beta)}{\beta(1-2\delta)}.
  \end{split}
\end{equation}
For the above choices of $\eta$ and $N$,
using \eqref{eq:single_step_bound} together with the fact that ${k<N}$ and ${\KL{\rho_k}{\target} \geq \eps/2}$,
in order for \eqref{eq:single_step_bound_clean} to hold,
it suffices to prove the following inequality
  \begin{equation*}
    \frac{3\lam^{-1/(1-\delta)}}{8} \left(\frac{\eps/2}{\M_s(\rho_0+\target) + C_s(N+1)\eta}\right)^{\delta/(1-\delta)}\eta
    - \frac{16 \alpha \smooth^2 \growth^{2\beta}}{a}\eta^{2\beta+1}
    \geq
    \frac{A\eta^{\delta\beta/(1-\delta)+1}}{\log{\left(\frac{2 \Delta_0}{\eps}\right)}^{\delta/(1-\delta)}}.
  \end{equation*}
  We will prove this inequality by showing that the following two inequalities hold,
  \begin{equation}\label{eq:two-inequalities}
  \left\{
  \begin{array}{ll}
    \frac{3\lam^{-1/(1-\delta)}}{8} \left(\frac{\eps/2}{\M_s(\rho_0+\target) + C_s(N+1)\eta}\right)^{\delta/(1-\delta)}
    \geq \frac{2A\eta^{\delta\beta/(1-\delta)}}{\log{\left(\frac{2 \Delta_0}{\eps}\right)}^{\delta/(1-\delta)}},\\
    \frac{A\eta^{\delta\beta/(1-\delta)+1}}{\log{\left(\frac{2 \Delta_0}{\eps}\right)}^{\delta/(1-\delta)}} \geq \frac{16 \alpha \smooth^2 \growth^{2\beta}}{a}\eta^{2\beta+1}.
  \end{array}
  \right.
  \end{equation}
  For the second inequality, we simply plug in the values for $\eta$ and $A$.
  Then, by using ${\eps < 2\Delta_0/e}$ and ${\M_s(\rho_0 + \target) \geq 1}$,
  this inequality holds if the following is satisfied,
  \begin{equation*}
  \begin{split}
    \left(\frac{2}{\eps}\right)^\frac{2-3\delta}{1-2\delta}
    &\geq
    \frac{16 \alpha \smooth^2 \growth^{2\beta}}{a}
    \left(
        \frac{1}{16 \lam^\frac{1}{1-\delta}}
    \right)^\frac{1-\delta}{1-2\delta}
    (\sigma \td)^{-2}.
  \end{split}
  \end{equation*}
  This yields an upper bound on the accuracy. In order to simplify this bound
  and make it independent of $s$, we
  define ${\tlam = \frac{4e^{2\pert}}{1 \vee \decof} \leq \lam}$.
  Also using ${4 \smooth^2 < \sigma}$ and ${d \leq \td}$, the bound can be simplified to
  \begin{equation*}
  \begin{split}
    \eps
    \leq
    2\left(\tlam^{0.5} \wedge \tlam^2\right)
    \left(1 \wedge \frac{2a\sigma d^2}{\growth^{2\beta}} \right)^{0.5},
  \end{split}
\end{equation*}
under which the second inequality in \eqref{eq:two-inequalities} holds.

For the first inequality in \eqref{eq:two-inequalities}, we consider two cases.
In the first case ${N \eta \geq 1}$, since we have
$\M_s(\rho_0+\target) + C_s(N+1)\eta \leq 3(\M_s(\rho_0+\target) \vee C_s)N\eta$,
the following condition implies the desired inequality
  $$
  \frac{3\lam^{-1/(1-\delta)}}{8} \left(\frac{\eps/2}{3(\M_s(\rho_0+\target) \vee C_s)N\eta}\right)^{\delta/(1-\delta)}
    \geq \frac{2A\eta^{\delta\beta/(1-\delta)}}{\log{\left(\frac{2 \Delta_0}{\eps}\right)}^{\delta/(1-\delta)}}.
  $$
  This inequality can be verified by plugging in the values of $A$, $\eta$ and $N$. In the other case  
  ${N\eta<1}$, we simply drop $N\eta$ since we have $\M_s(\rho_0+\target) + C_s(N+1)\eta \leq 3(\M_s(\rho_0+\target) \vee C_s)$;
  hence, the following condition suffices
  $$
  \frac{3\lam^{-1/(1-\delta)}}{8} \left(\frac{\eps/2}{3(\M_s(\rho_0+\target) \vee C_s)}\right)^{\delta/(1-\delta)}
    \geq \frac{2A\eta^{\delta\beta/(1-\delta)}}{\log{\left(\frac{2 \Delta_0}{\eps}\right)}^{\delta/(1-\delta)}}.
  $$
  For this to hold, it is sufficient if ${\eps < 2}$ and 
  $$
  \log{\left( \frac{2\Delta_0}{\eps}\right)} \geq \frac{1}{16}\lam^{-\frac{1}{1-\delta}}.
  $$
  which can be further strengthened to
  $$
  \eps \leq 2 \Delta_0 e^{\frac{-1}{16(\tlam \wedge \tlam^2)}}.
  $$
  Hence, the simplified single step bound \eqref{eq:single_step_bound_clean} holds when KL-divergence is not too small,
  i.e. when it is greater than ${\eps/2}$. For handling the case where KL-divergence is small,
  we need to show that once LMC reaches $\eps$-accuracy, it remains below that threshold until the last step. In other words
  \begin{equation}\label{eq:under_eps}
    \KL{\rho_k}{\target} \leq \eps \implies \KL{\rho_{k+1}}{\target} \leq \eps, \ \text{for ${k<N}$}.
  \end{equation}
  We split this into two cases. First, consider the case ${\eps/2 \leq \KL{\rho_k}{\target} \leq \eps}$.
  In this case, using \eqref{eq:single_step_bound_clean} and ${\KL{\rho_k}{\target} \leq \eps}$,
  it suffices to show
  \begin{equation*}
    \sigma \td \eta^{\beta+1}
    \leq\eps \frac{A\eta^{\delta\beta/(1-\delta)+1}}{\log{\left(\frac{2 \Delta_0}{\eps}\right)}^{\delta/(1-\delta)}} ,
  \end{equation*} 
  which can be verified by plugging in the values for $A$, $\eta$ and $\td$.
  The second case is when ${\KL{\rho_k}{\target} \leq \eps/2}$. Using Lemma~\ref{lemma:single_step_bound}, we need to show 
  \begin{equation*}
  \begin{split}
    \frac{16 \alpha \smooth^2 \growth^{2\beta}}{a}\eta^{2\beta+1}\frac{\eps}{2}
    + \sigma \td \eta^{\beta+1}
    \leq 
    \frac{\eps}{2}.
  \end{split}
  \end{equation*}
  We bound each of the terms on the left hand side with ${\eps/4}$.
  By simplifying the expressions and further 
  using ${\eps < 2\Delta_0/e}$ and ${\M_s \geq 1}$, we obtain the following two conditions
  on the accuracy $\eps$ to be combined together later,
  \begin{equation*}
  \begin{split}
    \eps &\leq
    8 (\tlam \wedge \tlam^2)
    \left(1 \wedge \frac{a}{\alpha \smooth^2 \growth^{2\beta}}\right)^\frac{1}{3}
    (1 \wedge \sigma d)
    \leq
    2^{5-\frac{5\beta(1-2\delta)}{(1+2\beta)(1-\delta )}}
    \left(\frac{a}{\alpha \smooth^2 \growth^{2\beta}}\right)^{\frac{\beta(1-2\delta)}{(1+2\beta)(1-\delta)}}
    \tlam^\frac{1}{1-\delta} (\sigma \td)^\frac{1-2\delta}{1-\delta},\\
    \eps
    &\leq
    32(\tlam \wedge \tlam^2)  (1 \wedge \sigma d)
    \leq
    2^{5+\frac{3\beta(1-2\delta)}{1-\delta+\delta\beta}}
    \lam^{\frac{1+\beta}{1-\delta+\delta\beta}}
    (\sigma \td)^{\frac{1-2\delta}{1-\delta+\delta\beta}}.
  \end{split}
  \end{equation*}
  
  Next, our analysis continues with considering the following two cases.
  \begin{enumerate}
  \item \lmc reaches $\eps$ accuracy at a step $k < N$.
  \item \lmc does not reach $\eps$ accuracy at a step $k < N$.
  \end{enumerate}
  For the first case above,
  if at any step ${k<N}$, we have ${\KL{\rho_{k}}{\target} \leq \eps}$,
  then \eqref{eq:under_eps} shows ${\KL{\rho_{N}}{\target} \leq \eps}$ and there is nothing to prove.
  For the second case, we have ${\KL{\rho_{k}}{\target}>\eps}$ for all ${k<N}$;
  therefore, \eqref{eq:single_step_bound_clean} combined with Lemma~\ref{lemma:rec_bound} and the fact that ${\KL{\rho_0}{\target} \leq \Delta_0}$ imply
  \begin{equation*}
    \KL{\rho_{N}}{\target}
    \leq
    \exp{\left(\frac{-A\eta^{\delta \beta /(1-\delta) + 1}}{\log{\left(\frac{2 \Delta_0}{\eps}\right)}^{\delta/(1-\delta)}}N\right)} \Delta_0
    + \frac{\sigma \td \log{\left(\frac{2 \Delta_0}{\eps}\right)}^{\delta/(1-\delta)} \eta^{\beta(1-2\delta)/(1-\delta)}}{A}.
  \end{equation*}
  Notice that to reach $\eps$ accuracy at step $N$,
  it is sufficient that each of the above terms on the right hand side is upper bounded by ${\eps/2}$.
  Simplifying these bounds, we obtain
  \begin{equation*}
  \begin{split}
    \log{\left(\frac{2\Delta_0}{\eps}\right)}
    &\leq
    \frac{A\eta^{\delta \beta /(1-\delta) + 1}}{\log{\left(\frac{2 \Delta_0}{\eps}\right)}^{\delta/(1-\delta)}}N,\\
    \eta
    &\leq
    A^\frac{1-\delta}{\beta(1-2\delta)} (\sigma \td)^{-\frac{1-\delta}{\beta(1-2\delta)}} 
    \left(\frac{\eps}{2}\right)^\frac{1-\delta}{\beta(1-2\delta)}
    \log{\left(\frac{2 \Delta_0}{\eps}\right)}^{-\frac{\delta}{\beta(1-2\delta)}}.
  \end{split}
  \end{equation*}
  The second inequality holds with the selection of $\eta$. Plugging the value for $\eta$ in the first inequality yields
  \begin{equation*}
    (\sigma \td)^\frac{1-\delta+\delta\beta}{\beta(1-2\delta)}
    A^{-\frac{(1+\beta)(1-\delta)}{\beta(1-2\delta)}}
    \log{\left(\frac{2\Delta_0}{\eps}\right)}^{\frac{\beta(1-\delta)+\delta}{\beta(1-2\delta)}}
    \left(\frac{2}{\eps}\right)^\frac{1-\delta+\delta\beta}{\beta(1-2\delta)}
    \leq
    N,
  \end{equation*}
  which is true because of the value of $N$.

  Finally, we translate the bound on the step size in Lemma~\ref{lemma:single_step_bound},
  to a condition on the accuracy $\eps$. That is, we have
  $${\eta \leq \frac{1}{2}\left(1 \wedge \frac{a}{2\growth^2} \right) \wedge \frac{4 \lam^\frac{1}{1-\delta} }{3} \left(\frac{\M_s(\rho_0 + \target) + C_s(k+1)\eta}{\KL{\rho_k}{\target}}\right)^\frac{\delta}{1-\delta}.}$$
  By plugging the value of $\eta$, in ${\eta \leq \frac{1}{2}\left(1 \wedge \frac{a}{2\growth^2} \right)}$, we get
  \begin{equation*}
  \begin{split}
    \left(\frac{\eps}{2}\right)
    \log{\left(\frac{2 \Delta_0}{\eps}\right)}^{-\frac{\delta}{1-\delta}}
    \leq
    32 \left(\frac{1}{2}\left(1 \wedge \frac{a}{2\growth^2} \right)\right)^{\beta(\frac{1-2\delta}{1-\delta})}
    \lam^\frac{1}{1-\delta}
    \left(\M_s(\rho_0 + \target) \vee C_s\right)^{\frac{\delta}{1-\delta}}
    (\sigma \td)^\frac{1-2\delta}{1-\delta},
  \end{split}
  \end{equation*}
  but since ${\eps < \frac{2\Delta_0}{e}}$ and ${\M_s\geq 1}$
  and ${\beta(\frac{1-2\delta}{1-\delta}) \leq 1}$ , it suffices to have
  \begin{equation*}
    \eps
    \leq
    16 (\tlam \wedge \tlam^2) (1 \wedge \sigma d) \left(1 \wedge \frac{a}{2\growth^2} \right)
    \leq
    16 \left(1 \wedge \frac{a}{2\growth^2} \right) \lam^\frac{1}{1-\delta}
    (\sigma \td)^{\frac{1-2\delta}{1-\delta}}.
  \end{equation*}
  For the other constraint on $\eta$,
  if we show
  ${\eta \leq \frac{4 \lam^\frac{1}{1-\delta} }{3} \left(\frac{\M_s(\rho_0 + \target)}{\Delta_0}\right)^\frac{\delta}{1-\delta}}$,
  Lemma~\ref{lemma:single_step_bound} shows that the first step is decreasing and ${\KL{\rho_1}{\target} \leq \Delta_0}$.
  Continuing inductively from there, we get either $\KL{\rho_k}{\target}$ is decreasing or it is less than $\eps$,
  in both of the cases, we have ${\KL{\rho_k}{\target} \leq \Delta_0}$.
  This in turn shows that the constraint on $\eta$ is getting looser,
  so all we need to consider is  
  $$\eta \leq \frac{4 \lam^\frac{1}{1-\delta} }{3} \left(\frac{\M_s(\rho_0 + \target)}{\Delta_0}\right)^\frac{\delta}{1-\delta},$$
  which holds whenever
  \begin{equation*}
  \begin{split}
    \eps
    \log{\left(\frac{2 \Delta_0}{\eps}\right)}^{-\frac{\delta}{1-\delta}} 
    \leq
    32 
    \lam^\frac{1-\delta+\beta(1-2\delta)}{(1-\delta)^2}
    \left(\M_s(\rho_0 + \target) \vee C_s \right)^{\frac{\delta}{1-\delta}}
    \left(\frac{\M_s(\rho_0 + \target)}{\Delta_0}\right)^\frac{\delta(1-2\delta)\beta}{(1-\delta)^2}
    (\sigma \td)^{\frac{1-2\delta}{1-\delta}}.
  \end{split}
  \end{equation*}
  Once again, since ${\eps < \frac{2\Delta_0}{e}}$ and ${\M_s \geq 1}$, all we need is
  \begin{equation*}
    \eps
    \leq
    32
    (1 \wedge \sigma d)
    (\tlam \wedge \tlam^3 )
    (1 \wedge \Delta_0^{-1})^\frac{\beta}{4}
    \leq
    32
    \lam^\frac{1-\delta+\beta(1-2\delta)}{(1-\delta)^2}
    \Delta_0^\frac{-\delta(1-2\delta)\beta}{(1-\delta)^2}
    (\sigma \td)^{\frac{1-2\delta}{1-\delta}}.
  \end{equation*}
  Collecting all the upper bounds on the accuracy we get

\begin{equation}\label{eq:tol-up}
  \begin{split}
    \psi = \min \bigg\{ & 2, \frac{2\Delta_0}{e},
    2\Delta_0 e^{\frac{-1}{16(\tlam \wedge \tlam^2)}},
    32
    (1 \wedge \sigma d)
    (\tlam \wedge \tlam^3 )
    (1 \wedge \Delta_0^{-1})^\frac{\beta}{4},
    16 (\tlam \wedge \tlam^2) (1 \wedge \sigma d) \left(1 \wedge \frac{a}{2\growth^2} \right),\\
    &  2\left(\tlam^{0.5} \wedge \tlam^2\right)
        \left(1 \wedge \frac{2a\sigma d^2}{\growth^{2\beta}} \right)^{0.5},
        8 (\tlam^2 \wedge \tlam)
        \left(1 \wedge \frac{a}{\alpha \smooth^2 \growth^{2\beta}}\right)^\frac{1}{3}
        (1 \wedge \sigma d),
        32(\tlam \wedge \tlam^2)  (1 \wedge \sigma d)
    \bigg\},
  \end{split}
\end{equation}
where $\tlam$ is defined as $\tlam = \frac{4e^{2\pert}}{1 \vee \decof}$.
  Note that the upper bound on $\eps$ is of order $\mathcal{O}(1)$,
  and it depends on the fixed parameters except for $\Delta_0$ which depends on the initial distribution. In case of starting with a Gaussian random vector, Lemma~\ref{lemma:init_point} provides a bound on $\Delta_0$. More importantly, the upper bound on the accuracy does not depend on the moment order $s$, which enables us to optimize over this parameter which is done in Corollary~\ref{cor:main}.
  Finally, we plug in the values for $\delta$, $\td$ and $C_s$ back into \eqref{eq:conv_rate_raw} to get
  
  \begin{equation*}
  \begin{split}
    \eta 
    & =
    \sigma^{-\frac{1}{\beta}}
    (16\lam)^{-\frac{s-2+2\decopow}{\beta(s-2)}}
    \left(
        \frac{\M_s(\rho_0+\target)}{16d^{(s-2+\alpha)/\alpha}} 
        \vee 
        \left(\frac{3a+2b+3}{1 \wedge a}\right)^{\frac{s-2+\alpha}{\alpha}}\frac{s^s}{16}
      \right)^{-\frac{\decopow}{\beta(s-2)}}\\
    &
      d^{-\frac{1}{\beta}-\frac{(s-2+\alpha)\decopow}{\alpha\beta(s-2)}}
      \left(1+ (1-\alpha/2)\log(d) \right)^{-\frac{1}{\beta}}
    \log{\left(\frac{2 \Delta_0}{\eps}\right)}^{-\frac{\decopow}{\beta(s-2)}}
    \left(\frac{\eps}{2}\right)^\frac{s-2+\decopow}{\beta(s-2)},\\
    N 
    &= 
    \sigma^{\frac{1}{\beta}}
    (16\lam)^\frac{(1+\beta)(s-2+2\decopow)}{\beta(s-2)}
    \left(
      \frac{\M_s(\rho_0+\target)}{16d^{(s-2+\alpha)/\alpha}} 
        \vee
      \left(\frac{3a+2b+3}{1 \wedge a}\right)^{\frac{s-2+\alpha}{\alpha}} \frac{s^s}{16}
    \right)^{\frac{(1+\beta)\decopow}{\beta(s-2)}}\\
    &d^{\frac{1}{\beta}+\frac{(s-2+\alpha)(1+\beta)\decopow}{\alpha\beta(s-2)}}
    \left(1+ (1-\alpha/2)\log(d) \right)^{\frac{1}{\beta}}
    \log{\left(\frac{2\Delta_0}{\eps}\right)}^{1+\frac{(\beta+1)\decopow}{\beta(s-2)}}
    \left(\frac{2}{\eps}\right)^{\frac{1}{\beta}+ \frac{(1+\beta)\decopow}{\beta(s-2)}}.
  \end{split}
  \end{equation*}
  \end{proof-of-theorem}

\section{Linear Growth of Convex Potentials}\label{sec:convex-proof}
First, we prove a lemma about one dimensional convex potentials, which will be used to 
prove the unboundedness in the general case.
\begin{lemma}\label{lem:one-dim-potential}
    Let ${f:\reals \to \reals}$ be a convex function such that
    ${\int_{\reals} e^{-f(x)}dx < \infty}$, then $f$ is lower bounded, i.e.
    ${\inf\limits_{x \in \reals} f(x) > - \infty}$.
\end{lemma}

\begin{proof}
Shifting $f$ does not affect convexity or finiteness of the integral,
so we can assume, without loss of generality, ${f(0)=0}$.
Let ${B(r) = \min_{x \in [-r,r]} f(x)}$,
which is well defined because $f$ is continuous -- convexity implies continuity in this context.
If $B$ is lower bounded, then so is $f$.
Suppose $B$ is not lower bounded. Continuity of $f$ implies that $B$ is also continuous,
and ${f(0)=0}$ implies that ${B(0)=0}$. Further, $B$ is a non-increasing function in its domain.

Since the range of $B$ contains all non-positive numbers, we can define
${y(M) = \min \{r \vert B(r)=-M \}}$, for ${M \geq 0}$.
Fix some ${M > 1}$.
Then,  the continuity of $B$ and $f$ imply that either ${f(y(M)) =-M}$ 
or ${f(-y(M))=-M}$. Without loss of generality, we assume ${f(y(M))=-M}$ (the other case is similar), and write
\eqn{
    \forall x \in [0,y(M)]:\ 
    f(x) \leq \left(1-\frac{x}{y(M)}\right)\times f(0) + \frac{x}{y(M)}\times f(y(M)) = -\frac{Mx}{y(M)}.
}
Using this fact, we integrate ${e^{-f}}$
\eqn{
    \int_{\reals} e^{-f(x)}dx \geq \int_{0}^{y(M)} e^{-f(x)}dx
    \geq \int_{0}^{y(M)} e^{\frac{Mx}{y(M)}} dx = y(M)\times \frac{e^M-1}{M}.
}
Monotonicity of $B$ implies ${y(M)>y(1)>0}$ since we also have ${B(0)=0}$.
Hence, the previous inequality yields
\eqn{
    \int_{\reals} e^{-f(x)}dx \geq y(1)\times \frac{e^M-1}{M} \ \ \text{ for every ${M>1}$}.
}
This inequality contradicts ${\int_{\reals} e^{-f(x)}dx < \infty}$.
\end{proof}

We use the previous one dimensional result to show that, in the general case,
not only the potential is lower bounded but also it has at least linear growth along every direction.
The method is to first prove the potential is unbounded along every direction and then use that
to prove linear growth.
\begin{lemma}\label{lemma:conv-unbounded}
  Suppose ${f:\reals^d \to \reals}$ is a convex potential and $u\in\reals^d$ is unit vector.
  Then, $f$ is coercive satisfying $${\sup\limits_{t \geq 0} f(tu) = +\infty}.$$
\end{lemma}
\begin{proof}
  Let ${f :\reals^d \to \reals}$ be a convex potential satisfying ${\int e^{-f(x)}dx = 1}$.
  Assume, for the sake of contradiction, that there is a direction
  ${u_1 \in \reals^d}$ such that $$\sup_{t \geq 0}f(tu_1) < M < \infty,$$
  and let ${\{u_1,u_2,...,u_d\}}$ be an orthonormal basis for $\reals^d$.
  Using convexity, we have
  \eq{
    f(tu_1) \geq f(x) + t\inner{\grad f(x), u_1} - \inner{\grad f(x),x}.
  }
  Taking supremum in both sides with respect to $t$ yields
  ${\inner{\grad f(x),u_1} \leq 0}$ for every ${x \in \reals^d}$.
  Let ${x_1 = \inner{x,u_1}}$ and write ${x=x_1 u_1 + x_{-1}}$ where ${\inner{u_1,x_{-1}}=0}$.
  By convexity, we have
  \eq{
    -f(x) \geq -f(0) - \inner{\grad f(x), x}.
  }
  We can write
  \eq{
    1= \int_{\reals^d} e^{-f(x)}dx
    \geq & \int_{\reals^{d-1}}\int_\reals e^{-f(0) - \inner{\grad f(x), x}} dx_1dx_{-1}\\
    =&\int_{\reals^{d-1}}\int_\reals e^{-f(0)
      - x_1 \inner{\grad f(x), u_1}
      -\inner{\grad f(x),x_{-1}}} dx_1dx_{-1}\\
    \geq &\int_{\reals^{d-1}}\int_{x_1 \geq 0} e^{-f(0)
      - x_1 \inner{\grad f(x), u_1}
      -\inner{\grad f(x),x_{-1}}} dx_1dx_{-1}\\
    \geq &\int_{\reals^{d-1}}\int_{x_1 \geq 0} e^{-f(0)
      -\inner{\grad f(x_1 u_1 + x_{-1}),x_{-1}}} dx_1dx_{-1}.
  }
  If we have ${\sup_{x_1 \geq 0}\inner{\grad f(x_1 u_1 + x_{-1}),x_{-1}} < \infty}$, then the inner integral diverges,
  so for almost every ${x_{-1} \in \Span\{u_2,...,u_d\}}$, 
  $$\sup_{x_1 \geq 0}\inner{\grad f(x_1 u_1 + x_{-1}),x_{-1}} = \infty.$$
  Using finiteness of the integral once again, we write
  \eq{
    1
    = \int_{\reals^d} e^{-f(x)}dx
    = \int_{\reals^{d-1}}\int_\reals e^{-f(x_1 u_1 + x_{-1})}dx_1dx_{-1}.
  }
  The inner integral should converge for almost every ${x_{-1} \in \Span\{u_2,...,u_d\}}$. Since
  a convex function restricted to a line is still convex, Lemma~\ref{lem:one-dim-potential}
  implies ${g(x_1)=e^{-f(x_1 u_1 + x_{-1})}}$ is lower bounded for almost every $x_{-1}$.
  Fix some ${x_{-1} \in \Span\{u_2,...,u_d\}}$ such that ${g(x_1)=e^{-f(x_1 u_1 + x_{-1})}}$ is lower bounded
  and
  ${\sup_{x_1 \geq 0}\inner{\grad f(x_1 u_1 + x_{-1}),x_{-1}} = \infty}$,
  which happens for almost every $x_{-1}$.
  By convexity, we have
  \eq{
    f(x_1u+ 2 x_{-1}) \geq&   f(x_1u+  x_{-1}) + \inner{\grad   f(x_1u_1+  x_{-1}), x_{-1}}.
  }
  Taking supremum over $x_1$, we have ${\sup_{x_1 \geq 0}f(x_1u_1+ 2 x_{-1}) = \infty}$.
  But since ${\sup_{x_1 \geq 0}f(2x_1u_1) < M}$, we can write
  \eq{
    \frac{1}{2} f(4x_{-1}) +   \frac{1}{2} f(2x_1 u) \geq f(x_1 u +{2}x_{-1}).
  }
  Taking supremum with respect to $x_1$ results in a contradiction. 
  So the assumption was incorrect and no direction like $u_1$ exists.
\end{proof}

In the light of the previous lemma, convexity implies a growth that is at least linear.
This is established in the following proof.

\begin{proof-of-proposition}[\ref{prop:conv-growth}]
    Let the function $B$ from unit sphere to real numbers be defined as
    $$B(u) = \inf \{ t>0 \vert f(tu) \geq 1 + f(0) \},$$
    which is well defined because of Lemma~\ref{lemma:conv-unbounded}.
    Convexity (and therefore continuity) of $f$ implies $B$ is continuous.
    Since unit sphere is compact, $B$ attains its maximum on it. Let us call
    this maximum $t_0>0$.
    We have ${f(t_0 u) \geq 1+ f(0)}$ for all unit vectors $u\in\reals^d$.
    For any ${t>t_0}$ and any unit vector $u$, because of convexity, we write
    \eq{
      \left(1-\frac{t_0}{t}\right) f(0) + \frac{t_0}{t} f(t u) \geq f(t_0 u)\geq 1 + f(0).
    }
    Therefore, for $t > t_0$, we have
    \eq{
      f(tu) \geq \frac{t}{t_0} +f(0),
    }
    for all unit directions $u$.

    When $t \in [0, t_0]$, the function $t \to f(tu)$ is lower bounded by some constant,
    i.e. $$\inf_{t\in[0,t_0]}f(tu) \coloneqq g(u) > -\infty$$
    by Lemma~\ref{lem:one-dim-potential}. Since $f$ is continuous in both $t$ and $u$,
    $g(u)$ is also continuous. Further, since its domain is compact, by the extreme value theorem,
    $g$ attains its infimum in its domain; thus, it is also lower bounded, say by $-M <0$.
    Therefore, whenever $t \in [0, t_0]$, $f(tu) \geq -M$ for all unit directions $u$.
    Combining this with the previous result, we obtain that for $t \in [0,\infty)$,
    \eq{
      f(tu) \geq \frac{t}{t_0} - |f(0)| \vee (M+1).
    }
    This completes the proof.
\end{proof-of-proposition}

\section{Proofs of Corollaries}\label{sec:post-proof}
\begin{proof-of-corollary}[\ref{cor:main}]
    Initializing with a Gaussian random vector provides us with
  \eq{
    \M_{s}(\rho_0) &= \EE{\left(1+\norm{x}^2\right)^{s/2}}
    \leq
    2^{s/2} \EE{1+ \norm{x}^{s}}
    \leq
    2^{s/2} (1+ d^{s/2} (s-1)!!)
    \leq
    (2 d s)^{s/2}.
  }
  We state a lemma to bound the moments of the target distribution. The proof is in
  to section~\ref{sec:proof-moment}. 
  \begin{lemma}\label{lemma:target-moment-bound}
    Let $f$ satisfy Assumption~\ref{as:mild_dis} then we have the following bound on the moment
    \eqn{
      \M_s(\target) \leq \left(\frac{a+b+3}{a}\right)^{s/\alpha} s^{s/\alpha}d^{s/\alpha} \ \
      \text{ for all }\ \ s\geq2.
    }
  \end{lemma}
  Combining the Gaussian moment bound with the previous lemma yields
  \eqn{
    \M_{s}(\rho_0 + \target) 
    \leq 
    2 \left(\frac{3a+b+3}{a}\right)^{s/\alpha} s^{s/\alpha} d^{s/\alpha}.
  }
  Using ${s=2+2\ceil{\log(\frac{6d}{\eps})}}$ implies $d^\gamma$
  and ${(2/\eps)^\gamma}$ are bounded with $\exp{\big(\frac{(1+\beta)\decopow}{2\beta}\big)}$.
  By plugging this upper bound back in Theorem~\ref{thm:main}
  and using the inequalities,
  \begin{equation*}
  \gamma < \frac{(1+\beta)\decopow}{2\beta},
  \qquad
  \eps<2 \vee 2\Delta_0/e,
  \qquad
  \lambda \leq \frac{4e^{2\pert}}{1 \wedge \decof},
  \end{equation*}
  the advertised rate is obtained.
\end{proof-of-corollary}

\begin{proof-of-corollary}[\ref{cor:tv_conv}]
    First we state Pinsker's inequality, which bounds total variation with KL-divergence.
    \begin{lemma}[Pinsker's inequality]\label{lemma:pinsker}
      For distributions $p$ and $q$ 
    \begin{equation*}
      \tv{p}{q} \leq \sqrt{\frac{1}{2} \KL{p}{q}}.
    \end{equation*}
    \end{lemma}
    If for given $\eps$ we use Corollary~\ref{cor:main} with accuracy ${2\eps^2}$,
    Pinsker's inequality implies $$\tv{\rho_N}{\target} \leq \eps.$$
    Note that the upper bound on $\eps$ is changed and ${2\eps^2}$ needs to be smaller than 
    upper bound in \eqref{eq:tol-up}. In other words $2\eps^2 \leq \psi$, where 
    $\psi$ is defined in \eqref{eq:tol-up}.
  \end{proof-of-corollary}
  \begin{proof-of-corollary}[\ref{cor:wasser_conv}]
    First we state a result, which is adapted from Corollary 3 in \cite{bolley2005weighted}, that bounds 
    $L_\alpha$-Wasserstein distance with KL-divergence.
    \begin{lemma}[\cite{bolley2005weighted}] \label{lemma:WCKP}
      For probability measure $p$ on $\reals^d$, if ${\int e^{\decopow \norm{x}^\alpha} p(x) dx < \infty}$, then
      \begin{equation*}
        \wasserpp{p}{q}{\alpha} \leq B \left[ \KL{p}{q}^{\frac{1}{\alpha}} + \left(\frac{\KL{p}{q}}{2}\right)^{\frac{1}{2\alpha}} \right],
      \end{equation*}
      where
      \begin{equation*}
        B \defeq 2 \inf_{\kappa} \left( \frac{1}{\kappa}\left(1.5 + \log{\int e^{\kappa \norm{x}^\alpha} p(x) dx} \right) \right)^\frac{1}{\alpha}.
      \end{equation*}
      \end{lemma}
    Lemma~\ref{lemma:exp_mom} proves an upper bound on
    $B$, namely ${B<2(4(\td\tmu+1.5)/a)^{1/\alpha}}$. By plugging
    this upper bound back in the previous lemma we get
    \begin{equation*}
      \wasserpp{\rho_N}{\target}{\alpha} \leq  
      2 \left(
        \frac{4\alpha}{a}
        (1.5+\tmu \left(1+ (1-\alpha/2)\log(d)\right)d)
      \right)^\frac{1}{\alpha}
      (\KL{\rho_N}{\target}^\frac{1}{\alpha}
      + \KL{\rho_N}{\target}^\frac{1}{2\alpha}).
    \end{equation*}
    If ${\eps \leq 4\left(4\alpha a^{-1}(1.5+\td\tmu)\right)^{1/\alpha}}$,
    using Corollary~\ref{cor:main} with 
    accuracy ${(\eps/4)^{2\alpha} (4\alpha a^{-1}(1.5+\tmu\td))^{-2}}$,
    implies the convergence rate. In order to obtain the upper bound on the accuracy, first
    let $\psi$ denote the bound in \eqref{eq:tol-up}.
    Since we used ${(\eps/4)^{2\alpha} (4\alpha a^{-1}(1.5+\tmu\td))^{-2}}$ as the accuracy in 
    terms of KL-divergence we need
    $$(\eps/4)^{2\alpha} (4\alpha a^{-1}(1.5+\tmu\td))^{-2} \leq \psi,$$
    by rearranging we get
    $$
    \eps \leq 4 (4\alpha a^{-1}(1.5+\tmu\td))^{-\frac{1}{\alpha}}\psi^{\frac{1}{2\alpha}}.
    $$
    Collecting these upper bound together we get
    \begin{equation}\label{eq:tol-up-w}
      \eps \leq 4 (4\alpha a^{-1}(1.5+\tmu\td))^{-\frac{1}{\alpha}}\psi^{\frac{1}{2\alpha}}
      \wedge 4\left(4\alpha a^{-1}(1.5+\td\tmu)\right)^{\frac{1}{\alpha}},
    \end{equation}
    where $\psi$ is defined in \eqref{eq:tol-up} and $\td$ and $\tmu$ are defined in \eqref{eq:tdtmu}.
  \end{proof-of-corollary}
  \begin{proof-of-corollary}[\ref{cor:wasser_2}]
    When ${\decopow=0}$, Theorem~\ref{thm:MLSI} implies LSI with constant 
    ${\frac{4e^{2\pert}}{\decof}}$. LSI 
    implies Talagrand's inequality with the same constant\cite{otto2000generalization}.
    $$\wasser{\rho_N}{\target} \leq 4e^{\pert} \sqrt{\KL{\rho_N}{\target}/\decof}.$$
    Theorem~\ref{thm:main} with accuracy ${\frac{\eps^2\decof}{16e^{2\pert}}}$
    implies the convergence rate. Note that we do not need to choose any $s$
    since ${\gamma=0}$ and Theorem~\ref{thm:main} is independent of $s$.
    The upper bound on $\eps$ changes and ${\frac{\eps^2\decof}{16e^{2\pert}}}$
    needs to be smaller than $\psi$. In other words
    \begin{equation}\label{eq:w2-tol-up}
      \begin{split}
        \eps \leq \frac{4e^{\pert}}{\sqrt{\decof}} \psi,
      \end{split}
    \end{equation}
    where $\psi$ is defined in \eqref{eq:tol-up}.
    \end{proof-of-corollary}

\section{Conclusion}\label{sec:conclusion}
In this paper, we analyzed the convergence of unadjusted \lmc algorithm for a
class of potentials whose tails behave like $\|x\|^\alpha$ for $\alpha\in[1,2]$,
and have $\beta$-\Holder continuous gradients.
This covers a wide range of non-convex potentials that are weakly smooth, and can be written
as finite perturbations of a function which is convex degenerate at $\infty$.
To establish this,
we proved a moment dependent modified log-Sobolev inequality for any order moment of the \lmc.
Further establishing a diverging moment estimate on the LMC iterates
under $\alpha$-dissipativity, we obtained a differential inequality which
can be iterated to obtain our main convergence result after tuning the moment order.
To demonstrate the applicability of our results,
we showed that any convex potential have at least linear growth,
and further 
we verified our main assumptions on a variety of sampling problems.
The presented results show that the convergence rate of \lmc
can be described
as a function of the tail growth rate and the order of smoothness in high dimensions.

There are several important future directions that one can consider, among which we highlight a few here.
Verifying the tightness of convergence rates established for the \lmc algorithm~\eqref{eq:ULA} is important;
thus, one needs to derive lower bounds in this framework, similar to those in~\cite{ge2019estimating,chatterji2020oracle,cao2020complexity}.
Moreover, in practice, higher order variants of \lmc is commonly used for sampling.
Among these, algorithms that are based on the underdamped Langevin diffusion received a lot of interest.
Therefore, generalizing the results of this paper to higher order and/or general It\^o diffusions is important.

Our results explain the behavior of \lmc in the case where $\beta\in (0,1]$,
and they do not cover the case $\beta=0$. Indeed,
there is no known result on vanilla \lmc for this case;
thus, exploring the behavior of \lmc in this regime may be of interest.
We note that many of the bounds in the
paper can be improved, at the expense of introducing some additional complexity
into the results. 
Moreover, our results hold only for the last iterate of the \lmc algorithm;
therefore, investigating the behavior of the subsequent iterates
is also an interesting direction left for another study.

\section*{Acknowledgements}
This research is partially funded by NSERC Grant [2019-06167], Connaught New Researcher Award, and CIFAR AI Chairs program at the Vector Institute.

\bibliographystyle{amsalpha}
{\small \bibliography{./bib}}
\newpage
\appendix
\section{Useful Lemmas}
\begin{lemma}\label{lemma:as_not_needed}
  For the potential function $f$, assume that
  there exists a function $\tf$ satisfying $$\big\|\grad f - \grad \tf\big\|_\infty \leq \pert.$$
  If $\tf$ satisfies \eqref{eq:perturb-hessian} in Assumption~\ref{as:degen_conv} for ${\decopow<1}$.
  Then $\alpha$-dissipativity in Assumption~\ref{as:mild_dis} is satisfied for ${\alpha=2-\decopow}$
  with the following constants
  \begin{equation*}
    \begin{split}
      a = \frac{\decof}{2(\alpha-1)}\ \ \text{ and }\ \ 
      b = \left(2(\|\grad \tf(0)\| + \decof +\pert)^\alpha/\decof \right)^{1/(\alpha-1)}.
    \end{split}
  \end{equation*}
\end{lemma}
\begin{remark}
  The additional assumption about bounded perturbation of gradient is to prevent cases when the perturbation is bounded but its gradient is not, for example, ${\left(1-2\sin(x)\right)^{\tfrac{1}{3}}}$.
\end{remark}

\begin{proof}
Using the fundamental theorem of calculus we have
\begin{equation*}
\begin{split}
  \inner{\grad \tf (x), x} &= 
  \inner{\int_{0}^{1} \Hess \tf (tx)x dt + \grad\tf(0), x}\\ 
  & =
  \inner{\grad\tf(0), x} + \int_{0}^{1} x^\top \Hess \tf (tx) x dt\\
  & \geq
  - \norm{\grad\tf(0)} \norm{x} 
  + \int_{0}^{1} \decof \left(1 + \norm{tx} \right)^{\alpha - 2} \norm{x}^2 dt \\
  & \geq
  - \norm{\grad\tf(0)} \norm{x}
  + \frac{\decof\norm{x}}{\alpha-1}\left( (1+\norm{x})^{\alpha-1}-1 \right) \\
  & =
  - \left(\norm{\grad\tf(0)} + \decof \right)\norm{x}
  + \frac{\decof}{\alpha - 1} \norm{x}^\alpha.
\end{split}
\end{equation*}
Since ${\big\|\grad f - \grad \tf\big\|_\infty \leq \pert}$, we get
\begin{equation*}
  \begin{split}
    \inner{\grad f (x), x} &\geq 
    - \left(\norm{\grad\tf(0)} + \decof + \pert \right)\norm{x}
    + \frac{\decof}{\alpha - 1} \norm{x}^\alpha\\
    &\geq
    \frac{\decof}{2(\alpha - 1)} \norm{x}^\alpha
    - \left(
        - \frac{\decof}{2(\alpha - 1)} \norm{x}^\alpha
        + \left(\norm{\grad\tf(0)} + \decof + \pert \right)\norm{x}
      \right)\\
    &\stackrel{1}{\geq}
    \frac{\decof}{2(\alpha - 1)} \norm{x}^\alpha
    - \left(
        \frac{2 \left( \norm{\grad\tf(0)} + \decof + \pert \right)^\alpha}{\decof}
        \times
        \frac{\alpha-1}{\alpha}
      \right) ^{1/(\alpha-1)}\\
    & \geq
    \frac{\decof}{2(\alpha - 1)} \norm{x}^\alpha
    - \left(
        \frac{2 \left( \norm{\grad\tf(0)} + \decof + \pert \right)^\alpha}{\decof}
      \right) ^{1/(\alpha-1)},
  \end{split}
  \end{equation*}
  where step $1$ follows from Lemma~\ref{lemma:poly_max}.
\end{proof}
\begin{lemma}\label{lemma:init_point}
  Under Assumption~\ref{as:holder}, the KL-divergence between distribution ${\rho = \*{N}(x, \id)}$ for ${x \in \reals^d}$ and the target distribution ${\target \propto e^{-f}}$ is bounded as follows
  \begin{equation}\label{eq:init_bound}
      \KL{\rho}{\target} \leq f(x) + \frac{\smooth}{\beta+1} d^\frac{\beta+1}{2}
      + \frac{d}{2} \log{(2\pi e)}.
  \end{equation}
  \end{lemma}
  \begin{remark}
  The RHS depends on ${f(x)}$, so if it is possible to find a minimizer (or an almost minimizer) of $f$, it is preferred to generate initial point from a Gaussian distribution centered around the minimizer.
  \end{remark}
  
  \begin{proof}
  First we bound ${\Esub{f(y)-f(x)}{y \sim \rho}}$ as follows.
  \begin{equation*}
  \begin{split}
      \Esub{f(y)-f(x)}{y \sim \rho}
      &= \Esub{\int_{0}^{1} \dotprod{\grad f(ty + (1-t)x)}{y-x} dt}{y \sim \rho}\\
      &= \Esub{\int_{0}^{1} \dotprod{\grad f(ty + (1-t)x) - \grad f(x)}{y-x} dt}{y \sim \rho}
      + \Esub{\int_{0}^{1} \dotprod{\grad f(x)}{y-x} dt}{y \sim p}\\
      &= \int_{0}^{1} \Esub{\dotprod{\grad f(ty + (1-t)x) - \grad f(x)}{y-x}}{y \sim \rho} dt
      + \int_{0}^{1} \dotprod{\grad f(x)}{ \Esub{y-x}{y \sim \rho}} dt\\
      &\leq \int_{0}^{1} \Esub{t^{\beta} \smooth \norm{y-x}^{\beta+1}}{y \sim \rho} dt\\
      &\leq \frac{\smooth}{\beta+1} \Esub{\norm{y-x}^{\beta+1}}{y \sim \rho} \leq \frac{\smooth}{\beta+1} \Esub{\norm{y-x}^2}{y \sim \rho}^{\tfrac{\beta+1}{2}}
      \leq \frac{\smooth}{\beta+1} d^\frac{\beta+1}{2}.
  \end{split}
  \end{equation*}
  Using the previous formula, we bound the KL-divergence
  \begin{equation*}
      \KL{\rho}{\target} = \int \rho(y) \log{(\rho(y))} dy + \int \rho(y) f(y) dy
      = -\ent{\rho} + \Esub{f(y)-f(x)}{y \sim \rho} + f(x).
  \end{equation*}
  Using the previous bound and the formula for the Gaussian entropy concludes the proof.
  \end{proof}
\begin{lemma}\label{lemma:power_rec_max}
For ${a,b>0}$, the function ${x \to a/x+bx^\theta}$ is minimized at ${x_* = (a/(\theta b))^{\tfrac{1}{1+\theta}}}$ and
the minimum value and an upper bound is given as
\eq{
    \tfrac{1+\theta}{\theta^{{\theta}/{(1+\theta)}}}a^{\frac{\theta}{1+\theta}} b^{\frac{1}{1+\theta}} \leq 2a^{\frac{\theta}{1+\theta}} b^{\frac{1}{1+\theta}}.
}
\end{lemma}
\begin{proof}
Taking derivative and setting it equal to zero yields the value for $x_*$.
\end{proof}
\begin{lemma}\label{lemma:power_triang}
If ${0 \leq \gamma \leq 2}$, then following inequality holds 
$${\norm{u+v}^\gamma \leq 2(\norm{u}^\gamma + \norm{v}^\gamma)}.$$
Further, when ${\gamma \leq 1}$ the factor  $2$ on the right hand side can be omitted.
\end{lemma}
\begin{proof}
 The inequality follows from the fact that functions 
 ${h_1(x) = (x^\gamma + 1) - (1+x)^\gamma}$
 and
 ${h_2(x) = 2(x^\gamma + 1) - (1+x)^\gamma}$
 are non-negative when $\gamma \in [0,1]$ and $\gamma \in [0,2]$, respectively.
\end{proof}

\begin{lemma}\label{lemma:poly_max}
Suppose $A,B,\alpha,\beta>0$ and ${\alpha > \beta}$ and
${f(x) = -Ax^\alpha + B x^\beta}$. The following
upper bound on $f$ holds when ${x > 0}$
$${\sup_{x \geq 0} f(x) \leq B \left( \frac{B \beta}{A \alpha} \right)^\frac{\beta}{\alpha -\beta}}.$$
\end{lemma}
\begin{proof}
Setting the derivative equal to zero implies ${x^{\alpha - \beta} = \frac{\beta B}{\alpha A}}$.
Plugging this into $f(x)$ we get 
${f(x)\leq B x^\beta
= B \left( \frac{B \beta}{A \alpha} \right)^\frac{\beta}{\alpha -\beta}}$.
Since ${\alpha > \beta}$ this function has a maximizer not a minimizer.
\end{proof}

\begin{lemma}[Stein's lemma~\cite{stein1981estimation}] \label{lemma:stein}
Suppose ${x \sim \*{N}(\decof, \sigma^2 \id)}$ and ${f:\reals^d \rightarrow \reals}$ is weakly differentiable. For ${a \in \reals^d}$
\begin{equation*}
    \EE{\dotprod{x-\decof}{af(x)}} = \sigma^2 \EE{\Tr(\grad \left[af(x)\right])}
    = \sigma^2 \EE{\dotprod{a}{\grad f(x)}}
\end{equation*}
\end{lemma}

\begin{lemma}\label{lemma:rec_bound}
If ${x_{k} \leq (1-a) x_{k-1} + b}$ for ${0 < a < 1}$ and ${0 \leq b}$, then
\begin{equation}\label{eq:rec_bound}
    x_k \leq e^{-a k} x_0 + \frac{b}{a}.
\end{equation}
\end{lemma}
\begin{proof}
Recursion on ${x_{k} \leq (1-a) x_{k-1} + b}$ yields
\begin{equation*}
    x_k \leq (1-a)^k x_0 + b(1+(1-a)+(1-a)^2+\dots+(1-a)^{k-1})
    \leq
    (1-a)^k x_0 + \frac{b}{a}.
\end{equation*}
Using the fact that ${1-a \leq e^{-a}}$, \eqref{eq:rec_bound}
is achieved.
\end{proof}

\subsection{Some Properties of \Holder Continuity}
\begin{lemma}\label{lemma:extend_holder}
Let $f$ be $\alpha$-\Holder continuous with constant $h_f^\alpha$ and
$\beta$-\Holder continuous with constant $h_f^\beta$ and 
${0<\beta<\alpha \leq 1}$, then $f$ is $\gamma$-\Holder with constant 
${h_f^\alpha \vee h_f^\beta}$ when ${\beta < \gamma < \alpha}$.
\end{lemma}

\begin{proof}
We consider two cases based on $\norm{x-y}$. First, when ${\norm{x-y} \leq 1}$,
\begin{equation*}
    \norm{f(x)-f(y)}
    \leq 
    h_f^\alpha \norm{x-y}^\alpha
    \leq
    h_f^\alpha \norm{x-y}^\gamma \norm{x-y}^{\alpha-\gamma}
    \leq
    h_f^\alpha \norm{x-y}^\gamma.
\end{equation*}
For the second case, when $\norm{x-y} > 1$,
\begin{equation*}
    \norm{f(x)-f(y)}
    \leq 
    h_f^\beta \norm{x-y}^\beta
    \leq
    h_f^\beta \norm{x-y}^\gamma \norm{x-y}^{\beta-\gamma}
    \leq
    h_f^\beta \norm{x-y}^\gamma.
\end{equation*}
Taking the maximum of constants in two cases completes the proof.
\end{proof}

\begin{lemma}\label{lemma:bounded_diff_holder}
Let $f$ be $\alpha$-\Holder continuous with constant $h_f^\alpha$
and $g$ be $\beta$-\Holder continuous with constant $h_g^\beta$ and
${\beta < \alpha \leq 1}$. If the difference of $f$ and $g$ is bounded
i.e. ${\norm{f-g}_\infty < B}$ then $f$ is $\beta$-\Holder
with constant ${h_f^\alpha \vee (2B + h_g^\beta)}$. In a specific case,
every bounded and Lipschitz function is $\tau$-\Holder for ${\tau \in (0,1)}$.
\end{lemma}

\begin{proof}
We consider two cases based on $\norm{x-y}$. First, when ${\norm{x-y} \leq 1}$,
\begin{equation*}
    \norm{f(x)-f(y)} \leq
    h_f^\alpha \norm{x-y}^\alpha
    \leq
    h_f^\alpha \norm{x-y}^{\beta} \norm{x-y}^{\alpha-\beta}
    \leq
    h_f^\alpha \norm{x-y}^{\beta}.
\end{equation*}
For the second case, when ${\norm{x-y}>1}$,
\begin{equation*}
    \norm{f(x)-f(y)} \leq
    \norm{f(x)-g(x)} + \norm{g(x) - g(y)} + \norm{f(y) - g(y)}
    \leq
    B + h_g^\alpha\norm{x-y}^\beta + B
    \leq 
    (2B + h_g^\alpha)\norm{x-y}^\beta.
\end{equation*}
Taking the maximum of constants in the two cases completes the proof.
\end{proof}

\begin{lemma}\label{lemma:holder_potential}
The function ${\norm{x}^{\alpha-2}x}$ is $\alpha-1$-\Holder for ${1<\alpha<2}$.
\end{lemma}
\begin{proof}
Without loss of generality, assume ${\norm{y} \leq \norm{x}}$ which
implies ${\norm{x-y} \leq \norm{x} + \norm{y} \leq 2 \norm{x}}$, which
in turn implies ${\norm{x}^{\alpha-2} \leq 2^{2-\alpha} \norm{x-y}^{\alpha-2}}$.
Therefore,
\begin{equation*}
\begin{split}
    \norm{f(x)-f(y)}
    & \leq \norm{\norm{x}^{\alpha-2}x - \norm{y}^{\alpha-2}y}
    \leq \norm{\norm{x}^{\alpha-2}x - \norm{x}^{\alpha-1}\frac{y}{\norm{y}} + \norm{x}^{\alpha-1}\frac{y}{\norm{y}} - \norm{y}^{\alpha-2}y}\\
    &\leq  \norm{x}^{\alpha-1}\norm{\frac{x}{\norm{x}}-\frac{y}{\norm{y}}} + \vert \norm{x}^{\alpha-1} - \norm{y}^{\alpha-1} \vert\\
    &\stackrel{1}{\leq} \norm{x}^{\alpha-1}\norm{\frac{x}{\norm{x}}-\frac{y}{\norm{x}}+\frac{y}{\norm{x}}-\frac{y}{\norm{y}}} + \norm{x-y}^{\alpha-1}\\
    &\leq \norm{x}^{\alpha-2}\norm{x-y} + \norm{x}^{\alpha-1}\norm{\frac{y}{\norm{y}}(\frac{\norm{y}}{\norm{x}}-1)} + \norm{x-y}^{\alpha-1} \\
    & \leq 2\norm{x}^{\alpha-2}\norm{x-y} + \norm{x-y}^{\alpha-1}
    \leq (1+2^{3-\alpha})\norm{x-y}^\alpha \leq 5 \norm{x-y}^\alpha,
\end{split}
\end{equation*}
where inequality 1 follows from Lemma~\ref{lemma:power_triang}.
\end{proof}

\end{document}